\documentclass{article}

\usepackage{xspace}
\usepackage{bbm}
\usepackage{bm}

\def\tr{\mathop{\rm tr}\nolimits}

\def\diag{\mathop{\rm diag}\nolimits}

\newcommand{\HP}{\text{\rm HP}}
\newcommand{\vecop}{\text{\rm vec}}
\def\one{\mathbf 1}
\newcommand{\RR}{\mathbb{R}}
\newcommand{\II}{\mathbb{I}}

\def\eps{\epsilon}

\newcommand{\bx}{\mathbf{x}}

\newcommand{\bz}{\mathbf{z}}
\newcommand{\bb}{\mathbf{b}}
\newcommand{\bv}{\mathbf{v}}

\newcommand{\bg}{\mathbf{g}}

\newcommand{\bX}{\bm{X}}

\newcommand{\bS}{\mathbf{S}}
\newcommand{\bW}{\mathbf{W}}
\newcommand{\bG}{\mathbf{G}}

\newcommand{\bu}{\mathbf{u}}

\newcommand{\bbeta}{\bm{\beta}}

\newcommand{\bSigma}{\bm{\Sigma}}
\newcommand{\bphi}{\bm{\phi}}

\def\as{{\rm a.s.}}
\def\wp{\text{\rm \ \ with probability\ }}

\newcommand{\new}{\mathrm{new}}

\newcommand{\bl}{\bm{\hat \beta}}

\newcommand{\bli}{\bm{ \hat \beta}_{/i }}

\newcommand{\brb}{\bm{\breve{\beta}}}
\newcommand{\brbli}{\bm{\breve{\beta}}_{/i}}
\newcommand{\tbli}{\bm{ \tilde \beta}_{/i }}

\newcommand{\GI}{\bm{G}_{/i}}

\newcommand{\ld}{\dot{\ell}}
\newcommand{\ldd}{\ddot{\ell}}

\def\textiid{i.i.d.\@\xspace}
\newcommand\iid{\ifmmode\text{ i.i.d. } \else \textiid \fi}

\newcommand{\poly}{\rm{poly}}

 \usepackage[main, final]{neurips_2026}
\usepackage[utf8]{inputenc} % allow utf-8 input
\usepackage[T1]{fontenc}    % use 8-bit T1 fonts
\usepackage{hyperref}       % hyperlinks
\usepackage{url}            % simple URL typesetting
\usepackage{booktabs}       % professional-quality tables
\usepackage{float}
\usepackage{amsthm}
\usepackage{amssymb}
\usepackage{amsmath}
\usepackage{amsfonts}       % blackboard math symbols
\usepackage{nicefrac}       % compact symbols for 1/2, etc.
\usepackage{microtype}      % microtypography
\usepackage{xcolor}         % colors
\usepackage{multirow}
\usepackage{cleveref}
\usepackage{graphicx}
\usepackage{titlesec}
\titlespacing*{\paragraph}
{0pt}
{0.2ex plus 0.1ex minus 0.1ex}
{0.4em}
\theoremstyle{plain}
\usepackage{mathtools}
\newtheorem{theorem}{Theorem}[section]
\newtheorem{proposition}[theorem]{Proposition}
\newtheorem{lemma}[theorem]{Lemma}

\theoremstyle{definition}

\newtheorem{assumption}[theorem]{Assumption}
\theoremstyle{remark}
\newtheorem{remark}[theorem]{Remark}
\theoremstyle{definition}
\newtheorem{example}[theorem]{Example}
\usepackage{etoolbox}

\title{Understanding What TRAK Gets Right and Wrong: \\ A Theory of Data Attribution}

\author{%
  Han Tong \\
  Department of Statistics \\
  Columbia University \\
  New York, NY \\
  \texttt{ht2672@columbia.edu} \\
  \And
  Haolin Zou \\
  Department of Statistics \\
  Columbia University \\
  New York, NY \\
\And
  Shubhangi Ghosh \\
  Department of Statistics \\
  Columbia University \\
  New York, NY \\
  \And
  Arian Maleki \\
  Department of Statistics \\
  Columbia University \\
  New York, NY \\
  \texttt{mm4338@columbia.edu} \\
}

\begin{document}

\maketitle

\begin{abstract}
Data attribution, which traces a model’s prediction back to specific training data, is an important tool for interpreting sophisticated AI models. The widely used TRAK algorithm addresses this challenge by first approximating the underlying model with a kernel machine and then leveraging techniques developed for approximating the leave-one-out (ALO) risk. Despite its strong empirical performance, the theoretical conditions under which TRAK approximations are accurate, as well as the regimes in which they break down, remain largely unexplored. In this paper, we provide a theoretical analysis of the TRAK algorithm, characterizing its performance and quantifying the errors introduced by the approximations on which the method relies. We show that although the approximations can incur significant errors, TRAK preserves the separation between highly influential and weakly influential data points. We corroborate our theoretical results through extensive simulations and empirical studies.
\end{abstract}

\section{Introduction}

\subsection{Related work and contributions}

Modern machine learning systems, including large language models, rely heavily on massive training datasets, where data quality significantly impacts model performance. This dependence has motivated growing interest in \emph{data attribution}, the task of tracing a model's predictions back to individual training samples, to better understand model behavior, identify harmful or redundant data, and improve data efficiency.

A popular class of data attribution methods is based on counterfactual reasoning: asking how a model's prediction would change if a particular training point were removed or reweighted. Popular approaches rely on quantities such as the Shapley value \citep{jia2019towards,ghorbani2019data,wang2024data} and influence functions \citep{hampel1974influence,sagun2017empirical,yeh2019fidelity,han2020explaining,hammoudeh22identifying,ilyas2022datamodels,pmlr-v206-fisher23a,park2023trak,guu2023simfluence,hammoudeh2024training}, whose computation is often prohibitively expensive or even intractable for modern AI models. Recently, the TRAK algorithm \citep{park2023trak} was proposed as a scalable data attribution method tailored to modern neural networks and has since gained popularity \citep{ye2023cognitive,xia2024less,hammoudeh2024training}. Its computational efficiency stems from three key approximations of the influence functions: linearization, random projection, and Approximate Leave-One-Out (ALO) corrections \citep{rad2018scalable,auddy24a}. TRAK has demonstrated strong empirical performance on large models and datasets.

Despite its practical success, little is known about the theoretical conditions under which TRAK provides accurate attributions, or about the error introduced by its multiple approximation steps. This lack of theoretical understanding limits our ability to determine when and for what tasks TRAK can be reliably used, and what information is lost by its approximations. 

This paper aims to close this gap. 
We provide the first systematic theoretical analysis of TRAK by evaluating the errors each approximation step of TRAK introduces. To give a high-level overview of our contributions, let $\mathcal{I}^{\rm True} (\bz_i, \bz_{\rm new})$ denote the true influence of data point $\bz_i$ on a test data point $\bz_{\rm new}$ \citep{koh2017understanding}. This quantity will be defined more formally in \eqref{def:influ_func}. Furthermore, let $\mathcal{I}^{\rm TRAK} (\bz_i, \bz_{\rm new})$ denote the TRAK approximation of $\mathcal{I}^{\rm True} (\bz_i, \bz_{\rm new})$. 

\paragraph{Scope of our contributions.}
Our main theoretical results are proved in the underparameterized regime,
i.e., $d<n$, where $d$ denotes the number of parameters and $n$ denotes
the number of training samples; these quantities are defined formally
below. Appendix~\ref{sec:overparameterized-ridge} provides a preliminary
analysis of a ridge-regularized formulation for $d>n$, rather than a
complete extension of our main theorems to the overparameterized regime.

Our analysis reveals the following facts about TRAK: 
\paragraph{Negative result:} While the error introduced by the ALO approximation is typically small and benign, both the linearization step and the projection step, especially when the number of projections is much smaller than the number of parameters, can introduce substantial errors. As a result, the exact ranking of training data points induced by $\mathcal{I}^{\rm TRAK}(\bz_i, \bz_{\rm new} )$ can differ significantly from that induced by $\mathcal{I}^{\rm True}(\bz_i, \bz_{\rm new})$. Our empirical results are consistent with our theoretical results and confirm that TRAK does \textbf{not} preserve the exact rankings offered by $\mathcal{I}^{\rm True}(\bz_i, \bz_{\rm new})$.

\paragraph{Positive result:}  Despite the large errors TRAK introduces, it can detect data points that have large influences from those that have small influences. Informally speaking, if the influence of data point $\bz_i$, i.e. $\mathcal{I}^{\rm True}(\bz_i, \bz_{\rm new})$ is much larger than the $\mathcal{I}^{\rm True}(\bz_j, \bz_{\rm new})$, then the $\mathcal{I}^{\rm TRAK}(\bz_i, \bz_{\rm new} )$ remains much larger than $\mathcal{I}^{\rm TRAK}(\bz_j, \bz_{\rm new})$ as well. Hence, when the objective is to identify the most influential data points, TRAK provides accurate and reliable information.

\subsection{TRAK influence function}
In this section, we formally present the three approximations employed by the TRAK algorithm. %In the next section, we study the errors introduced by each approximation step to shed light on the performance of TRAK.
Consider the standard supervised learning setting, in which we are given a dataset 
\[
\mathcal{D} = \{(y_1, \mathbf{x}_1), (y_2, \mathbf{x}_2), \ldots, (y_n, \mathbf{x}_n)\}, 
\]
where \( y_i \in \mathbb{R} \) is the response, \( \mathbf{x}_i \in \mathbb{R}^p \) denotes the corresponding feature vector, and $\bz_i := (y_i, \mathbf{x}_i)$. The goal is to learn a predictor that can accurately predict the response associated with a new feature vector \( \mathbf{x}_{\new} \).
Many successful approaches in AI and machine learning, including generalized linear models and neural networks, use predictors of the form \( f(\bx; \bbeta) \), where \( \bbeta \in \mathbb{R}^d \) are parameters (to be learned from data), and $f(\cdot ; \bbeta)$ is a linear or nonlinear function that produces the predictions. In most applications, the model parameters are estimated using empirical risk minimization:
\begin{eqnarray}\label{eq:bl} 
\bl \triangleq \underset{\bm{\beta} \in \RR^d}{\arg\min} \Biggl \{ \sum_{i=1}^n \ell ( y_i,f(\bx_i;\bbeta) ) \Biggr \}, 
\end{eqnarray} 
where $\ell(y,f(\bx; \bbeta))$ denotes the loss function, such as the squared loss or cross-entropy loss, depending on the application and the nature of the response variable.

To interpret model predictions, many data attribution techniques seek to quantify the influence of individual training samples on a model’s output. Specifically, the influence of a training data point $\bz_i = (y_i, \bx_i)$ on the model’s prediction at a new data point $\bz_{\mathrm{new}} = (y_{\mathrm{new}}, \bx_{\mathrm{new}})$ is defined as 
\begin{equation}\label{def:influ_func} 
\mathcal{I}^{\mathrm{True}}(\bz_i, \bz_{\mathrm{new}})
\triangleq
f(\bx_{\mathrm{new}}; \bli)
-
f(\bx_{\mathrm{new}}; \bl).
\end{equation}
where
\begin{equation}\label{eq:optimization_lo} 
\bli \triangleq \arg\min_{\bm{\beta} \in \RR^d}
\left\{
\sum_{j \neq i} \ell\bigl(y_j, f(\bx_j; \bm{\beta})\bigr)
\right\}.
\end{equation}
A fundamental challenge in applying \eqref{def:influ_func} for data valuation is its prohibitive computational cost. Computing the influence of a single training sample requires solving the leave-one-out optimization problem in \eqref{eq:optimization_lo}. In modern machine learning and AI applications, where both the sample size \(n\) and the model dimension \(d\) are large, the exact computation of \(\bli\) becomes computationally prohibitive, especially when it must be performed repeatedly for many training data points. The \textbf{TRAK} algorithm mitigates this issue by approximating $\bli$ using the following sequence of steps:

\paragraph{Step 1: Linearization.}  
The first step in TRAK is to linearly approximate $f(\bx, \bbeta)$ around $\bl$. Define
\begin{align*}
\bm{g}_j = \nabla f(\bx_j, \bl),  \quad
\bm{b}_j = f(\bx_j, \bl) - \nabla f(\bx_j, \bl)^\top \bl.
\end{align*}
Instead of solving the nonlinear optimization problem \eqref{eq:optimization_lo}, TRAK solves the following linearized version:
\begin{eqnarray}  
\label{def:brb_b_i}
\breve{\bm{\beta}} \triangleq \underset{\bm{\beta} \in \RR^d}{\arg\min} \Biggl \{ \sum_{i=1}^n \ell \big(y_i, \bg_i^\top \bm{\beta} + \bb_i \big) \Biggr \},\quad
\breve{\bm{\beta}}_{/i} \triangleq \underset{\bm{\beta} \in \RR^d}{\arg\min} \Biggl \{ \sum_{j \neq i} \ell \big(y_j, \bg_j^\top \bm{\beta} + \bb_j \big) \Biggr \}.  
\end{eqnarray}  
In the Appendix, we further show that under our assumptions, $\brb = \bl$. The intuition is that the linear approximation of $f(\bx,\bbeta)$ around $\bl$ provides a sufficiently accurate surrogate, so the corresponding optimization problem changes only slightly. As a result, when $\bbeta$ is close to $\bl$, we can use the linear function of $\{\bg_i\}_{i\in[n]}$ as an approximation to $f(\bx_i,\bbeta)$.

\paragraph{Step 2: Random projection.}  
To further reduce computational cost, TRAK applies random projection to the gradients $\bg_i$. The idea is to reduce dimensionality while approximately preserving inner products. More specifically, we define the gradient matrix used in the TRAK algorithm as:
\begin{equation}\label{eq:Gdef_main}
\bG(\bbeta) := \big(\nabla f(\bx_1, \bbeta), \dots, \nabla f(\bx_n, \bbeta)\big)^\top,
\end{equation}
and for notational simplicity, we write $\bG:=\bG(\bl)$.
Using a random projection matrix $\bS \in \RR^{d\times k}$ with $k \ll d$, TRAK defines the feature map $\bphi(\bx) = \bS^\top \bg(\bx)$, where $\bg(\bx) = \nabla f(\bx, \bl)$. Accordingly, we define $\Phi = (\bphi_1,\ldots,\bphi_n)^\top \in \RR^{n\times k}$, where $\bphi_i = \bS^\top \bg_i$ and $\bphi_{\new} = \bS^\top \bg_{\new}$.

\paragraph{Step 3: Approximate Leave-One-Out (ALO).}  
Building on the ideas of approximate leave-one-out \citep{rad2018scalable, beirami2017optimal}, TRAK estimates the influence via the ALO formula \footnote{In the TRAK paper, the authors note that replacing $\diag[\bm\ldd(\brb)]$ with the identity matrix and ignoring the denominator in \eqref{trak:alo} does not significantly affect empirical performance. As a result, they use the simplified expression
\[
\ld_i(\brb)\,
\bphi_{\mathrm{new}}^\top 
\bigl(\Phi^\top \Phi\bigr)^{-1}
\bphi_{i}.
\]
However, for theoretical analysis, we will consider the complete form in \eqref{trak:alo}.}, which provides the following approximation for $\mathcal{I}^{\mathrm{True}}(\bz_i, \bz_{\mathrm{new}})$:
\begin{equation}
     \mathcal{I}^{\mathrm{TRAK}}(\bz_i, \bz_{\mathrm{new}};k) \triangleq\frac{
            \ld_i(\brb)\,
            \bphi_{\mathrm{new}}^\top 
            \bigl(\Phi^\top \diag[\bm\ldd(\brb)] \Phi\bigr)^{-1}
            \bphi_{i}
        }{
            1 - 
            \ldd_i(\brb)\,
            \bphi_{i}^\top 
            \bigl(\Phi^\top \diag[\bm\ldd(\brb)] \Phi\bigr)^{-1}
            \bphi_{i}
},\label{trak:alo}
\end{equation}
where $k$ is the projection dimension, $\ld(y,z)$ and $\ldd(y,z)$ denote the first- and second-order derivatives of the loss function with respect to $z$, respectively; $\ld_i(\brb) = \ld(y_i, \bg_i^{\top} \brb+\bb_i)$, $\ldd_i(\brb) = \ldd(y_i, \bg_i^{\top} \brb+\bb_i)$, $\bm{\ld}(\brb)= [\ld_1(\brb), \ld_2(\brb), \ldots, \ld_n(\brb)]^{\top}$ and $\bm{\ldd}(\brb)= [\ldd_1(\brb), \ldd_2(\brb), \ldots, \ldd_n(\brb)]^{\top}$.

%\begin{remark}
% Assumption~\ref{ass:unique_minimizer}, it is not difficult to prove that $\brb = \bl$, as already shown in the proof of Theorem~\ref{theo:step1}.
%\end{remark}

% \begin{equation}\label{eq:TRAK_Influence}
% \mathcal{I}^{\mathrm{TRAK}}(\bz_i, \bz_{\mathrm{new}})  \triangleq
% f(\bx_{\mathrm{new}}; \bl)
% -
% f(\bx_{\mathrm{new}}; \brb_{\backslash i}).
% \end{equation}

We should also note that TRAK includes two additional steps: (i) ensemble averaging, which mitigates variability arising from random initialization and the inherent randomness of the optimization algorithms used to minimize the training loss, and (ii) soft thresholding of the influence estimates, which introduces shrinkage under the assumption that only a small subset of training points meaningfully affects the prediction for a given test example. Since our main analysis focuses on the accuracy of TRAK’s core approximation steps, we discuss these two additional steps in Appendix~\ref{sec:ensemble-thresholding}.

\section{Main theoretical contributions}
\label{sec:thm-result}
\subsection{Our theoretical framework}
\label{sec:main_assumptions}

Throughout the theoretical part of the paper we assume that the elements of  $\mathcal{D}$ are independent draws from the joint distribution
\begin{equation}
\label{def:q}
(y_i, \bx_i) \sim q(y_i \mid f(\bx_i; \bbeta^*))\, p(\bx_i).
\end{equation}
Here, $\bbeta^* \in \mathbb{R}^d$ denotes the true parameter vector. Unlike classical linear models, the parameter dimension $d$ may differ from the feature dimension $p$; for example, in neural networks, $d$ corresponds to the total number of weights.
We define the loss function as the negative log-likelihood
\[
\ell(y, f(\bx; \bbeta)) = - \log q(y \mid f(\bx; \bbeta)).
\]
Motivated by modern machine learning and AI applications, where the models have billions or even trillions of parameters and are trained on gigabytes of data, in our theoretical studies we focus on the regime that the sample size $n$, the number of features $p$, and number of parameters $d$ are all large. In other words, in our analysis we will ignore the errors that are negligible for large values of $n,p$ and $d$.   

In addition to this assumption, we impose several structural conditions to characterize the scaling behavior of the approximation errors introduced by TRAK. We give an informal summary below and defer the formal statements to Appendix~\ref{app:detailed_assumptions}.

\begin{assumption}[Data and model regularity]\label{ass:main_assumptions}
We assume the following throughout the paper.

\noindent
\textbf{(A1) Sub-Gaussian design.}
The rows of $\bm{X}_n\in\mathbb{R}^{n\times p}$ are independent mean-zero sub-Gaussian random vectors with covariance $\bm{\Sigma}_p$. Moreover, for constants $c,C>0$,
\[
{c}\|\bbeta^*\|^{-2}
\leq \lambda_{\min}(\bm{\Sigma}_p)
\leq \lambda_{\max}(\bm{\Sigma}_p)
\leq {C}\|\bbeta^*\|^{-2}.
\]
Appendix~\ref{app:detailed_assumptions} provides further intuition for the scaling. 

\noindent
\textbf{(A2) Model and loss regularity.}
We assume standard smoothness  conditions on the prediction map $f$ and the loss $\ell$. Informally, we assume that $\|\nabla f(\bx,\bbeta^*)\|$ has sub-Gaussian tails with bounded second moment; the empirical Hessians associated with \eqref{eq:bl} and \eqref{def:brb_b_i} are well-conditioned near $\bl$; and $\ell$ is strongly convex, with $\ell$ and $\nabla\ell$ being Lipschitz continuous. Formal statements are given in Appendix~\ref{app:detailed_assumptions}.
\end{assumption}
In general, these are standard regularity conditions for high-dimensional $M$-estimation. Regarding the assumptions on the loss, Appendix~\ref{sec:corder_check} verifies the loss regularity conditions for common objectives, including the squared loss, cross-entropy loss, and the Poisson loss with a softplus link. The regularity conditions on $f$ also hold on a wide range of models. We mention two canonical examples below.

\begin{example}[Linear model]
For generalized linear models with
\[
f(\bx,\bbeta)=\bx^\top\bbeta,
\]
the model assumptions in Appendix~\ref{app:detailed_assumptions} hold; see Appendix  \ref{ssec:proof:example:glm} for the proof. 
\label{Ex:lin}
\end{example}
\begin{example}[Neural network]
\label{Ex:NN}
For a one-hidden-layer neural network
\[
f(\bx,\bW,\bv)=\bv^\top\sigma(\bW\bx),
\]
where $\bW\in\mathbb{R}^{h\times p}$, $\bv\in\mathbb{R}^h$, and $\sigma$ has bounded derivative, the model assumptions in Appendix~\ref{app:detailed_assumptions} hold; see Appendix  \ref{ssec:proof:example:nn} for the proof.
\end{example}

\subsection{Theoretical results}

\subsubsection{Overview of the results}
Our theoretical contributions  are presented in four sections:

\noindent $\bullet$ Section \ref{ssec:size:influence} characterizes the size of  $\mathcal{I}^{\mathrm{True}}(\bz_i, \bz_{\mathrm{new}})$ in two settings: 
(1) when $\bz_i$ is related to $\bz_{\mathrm{new}}$, in which case the influence of $\bz_i$ on $\bz_{\mathrm{new}}$ is expected to be large; and 
(2) when $\bz_i$ and $\bz_{\mathrm{new}}$ are unrelated, in which case the influence of $\bz_i$ on $\bz_{\mathrm{new}}$ is expected to be small. 
These order-of-magnitude characterizations will be used  to interpret the properties of $\mathcal{I}^{\mathrm{True}}(\bz_i, \bz_{\mathrm{new}})$.

\noindent $\bullet$ Section \ref{ssec:linearization} studies the error introduced by the linearization step in TRAK and highlights the key limitations of this approximation, and also its intriguing features.  

\noindent $\bullet$ Section \ref{ssec:ALO} studies the error that is introduced by the ALO approximation in TRAK, and shows that the error that is introduced in this step is benign and substantially smaller than the errors that are introduced in the linearization and projection steps of TRAK. 

\noindent $\bullet$ Section \ref{ssec:projection} studies the error that is introduced by the projection step in TRAK and shows the important limitations of the projection step, especially when the dimension of the projection is much smaller than the number of parameters. 
\begin{table}[ht]
\centering
\caption{Summary of  the theoretical results of the paper regarding the size of the  influence function and magnitudes of errors introduced at each step of TRAK. The displayed orders hold with high probability, with $n^{\eps}$ and $\poly(\log n)$ factors omitted.}
\label{tab:bound-summary}
\resizebox{\linewidth}{!}{%
\begin{tabular}{lccc}
\toprule
\textbf{Step} 
& \textbf{Quantity / Result} 
& \textbf{Correlated case: $\bz_{\new}=\bz_i$} 
& \textbf{Uncorrelated case: $\bz_{\new}\perp \bz_i$} \\
\midrule

True influence 
& $\mathcal{I}^{\rm True}(\bz_i,\bz_{\new})$
& \Cref{theo:I_true_corr}: $\dfrac{\|\bbeta^*\|^2}{n}$ 
& \Cref{theo:I_true}: $\dfrac{\|\bbeta^*\|}{n}$ \\

Linearization bound
& $\mathcal{I}^{\rm Linear}-\mathcal{I}^{\rm True}$
& \Cref{theo:step1_corr}: $\dfrac{\|\bbeta^*\|^2}{n}$ 
& \Cref{theo:step1}: $\dfrac{\|\bbeta^*\|}{n}$ \\

ALO bound
& $\mathcal{I}^{\rm ALO}-\mathcal{I}^{\rm Linear}$
& \Cref{thm:g_base_ALO_ii}: $\dfrac{\|\bbeta^*\|^2}{n^{3/2}}$ 
& \Cref{thm:g_base_ALO}: $\dfrac{\|\bbeta^*\|}{n^{3/2}}$ \\

Projection bound
& $\mathcal{I}^{\rm TRAK}(\bz_i,\bz_{\new};k)$
& \Cref{theo:magnitude_projection}: $\dfrac{\|\bbeta^*\|^2}{n}\dfrac{k}{d}$ 
& \Cref{theo:magnitude_projection}: $\dfrac{\|\bbeta^*\|}{n}\sqrt{\dfrac{k}{d}}$ \\

\bottomrule
\end{tabular}%
}
\end{table}
Table~\ref{tab:bound-summary} summarizes all the theoretical contributions of our paper.
\subsubsection{Size of influence}\label{ssec:size:influence}
In this section, we study the magnitude of $\mathcal{I}^{\mathrm{True}}(\bz_i, \bz_{\new})$. Our first result obtains an upper bound on the influence function and shows its sharpness.

\begin{proposition}\label{theo:I_true_corr}
Under Assumption~\ref{ass:main_assumptions} (formally, Assumptions~\ref{as1}--\ref{ass:unique_minimizer} in the Appendix), 
for any $\eps > 0$ such that $\|\bbeta^*\|^2 \le n^{1-\eps}$, 
there exists an absolute constant $C^{\rm True}$ such that, with probability tending to $1$,
\begin{equation*}
\sup_{i \in [n]} \big| \mathcal{I}^{\mathrm{True}}(\bz_i, \bz_{\rm new}) \big|
\;\le\;
C^{\rm True} \frac{\|\bbeta^*\|^2 \poly(\log n)}{n},
\end{equation*}
where $\poly(\log n)$ denotes a polynomial function of $\log(n)$. 
Furthermore, if Assumption~\ref{as9} holds, then there exists an absolute constant $\tilde C^{\rm True}$ such that, with probability tending to $1$,
\begin{equation}\label{eq:inf:exact:x_i}
\inf_{i \in [n]} \big| \mathcal{I}^{\mathrm{True}}(\bz_i, \bz_{i}) \big|
\;\ge\;
\tilde C^{\rm True} \frac{\|\bbeta^*\|^2}{n}.
\end{equation}
\end{proposition}
The proof of this proposition is presented in Appendix \ref{proof:upper:inf:corr}. This theorem shows how large an influence of a data point on $\bz_{\rm new}$ can be. But to understand this result better, we have to discuss how large $ \big| \mathcal{I}^{\mathrm{True}}(\bz_i, \bz_{\rm new}) \big|$ can be when ${\bz}_{\rm new}$ is independent of the training set. In such cases, we expect the influence to be small. Our next theorem addresses this question.

\begin{proposition}\label{theo:I_true}
Suppose that $\bz_{\rm new}$ is independent of the training set $\mathcal{D}$. Under Assumption~\ref{ass:main_assumptions} (formally, Assumptions~\ref{as1}--\ref{ass:unique_minimizer} in the Appendix), for any $\eps > 0$ such that $\|\bbeta^*\|^2 \le n^{1-\eps}$, there exists an absolute constant $C^{\rm True}$ such that, with probability tending to $1$,
\begin{equation*}
\sup_{i \in [n]} \big| \mathcal{I}^{\mathrm{True}}(\bz_i, \bz_{\new}) \big|
\;\le\;
C^{\rm True} \frac{\|\bbeta^*\|}{n^{1-\eps}}.
\end{equation*}
\end{proposition}

The proof of Proposition~\ref{theo:I_true} is shown in Appendix~\ref{sec:prop-proof}. 

Since $n$, $p$, and $d$ are all assumed to be large, we expect $\|\bbeta^*\|$ to be much larger than $1$. Comparing the results of the two theorems, we observe a substantial gap between the influence of data points $\bz_i$ in the correlated and independent cases. As we clarify in the next sections, this large separation is in fact one of the  reasons behind the empirical success of the TRAK method.

\subsubsection{Linearization step}\label{ssec:linearization}

In this section, we aim to understand the error introduced by the linearization step in TRAK. If projection and ALO steps are removed from TRAK, then the approximate influence function becomes
\begin{equation}
\mathcal{I}^{\mathrm{Linear}}(\bz_i,\bz_{\new}) \triangleq \bg_{\mathrm{new}}^\top (\brbli-\brb),
    \label{def:Inflinear}
\end{equation}
where $\bg_{\mathrm{new}} = \nabla f(\bx_{\rm new}, \bl)$, and $\brb$ and $\brbli$ are defined in \eqref{def:brb_b_i}.

Following our approach from the previous section, we study $|\mathcal{I}^{\mathrm{Linear}}(\bz_i,\bz_{\new}) - \mathcal{I}^{\mathrm{True}}(\bz_i,\bz_{\new})|$ in two cases: (i) the correlated case, where $\bz_i=\bz_{\rm new}$, and (ii) the uncorrelated case, where $\bz_i$ and $\bz_{\rm new}$ are independent. We begin with the correlated case.

\begin{theorem} \label{theo:step1_corr}
Under Assumptions~\ref{ass:main_assumptions} (formally, Assumptions~\ref{as1}--\ref{ass:unique_minimizer} in the Appendix), for any $\eps > 0$ such that $\|\bbeta^*\|^2 \le n^{1-\eps}$, there exists an absolute constant $C^{\rm{Linear}}$ such that, with probability tending to $1$, the linearization error satisfies:
    \begin{equation*}
        \sup_{i \in [n]}\big| \mathcal{I}^{\mathrm{True}} (\bz_i,\bz_i) - \mathcal{I}^{\mathrm{Linear}}(\bz_i,\bz_i) \big| 
        \leq C^{\rm{Linear}} \frac{\|\bbeta^*\|^2}{n}.  \end{equation*}
    % where \(\bg_{\mathrm{new}} = \nabla f (\bx_{\mathrm{new}}, \bl)\), $ \mathcal{I}^{\mathrm{True}}$ is defined in \eqref{def:influ_func} and $ \mathcal{I}^{\mathrm{Linear}}$ is defined in \eqref{def:Inflinear}.
    Furthermore, if Assumption~\ref{as9} holds, then there exists an absolute constant $c^{\rm Linear}$ such that, with probability tending to $1$,
    \begin{equation}\label{eq:lowerbound:inf:correlated}
    \inf_{i \in [n]}|\mathcal{I}^{\mathrm{Linear}}(\bz_i,\bz_i) |\geq c^{\rm Linear} \frac{\|\bbeta^*\|^2}{n}. 
    \end{equation}
\end{theorem}
The proof can be found in Appendix \ref{proof:step1_corr}.
We highlight several points concerning this theorem.

\begin{remark}
% \textcolor{red}{should think about it later. We do not have it yet! I think my simulation can show it?}.
While this theorem obtains an upper bound on the difference $\big| \mathcal{I}^{\mathrm{True}} (\bz_i,\bz_i) - \mathcal{I}^{\mathrm{Linear}}(\bz_i,\bz_i) \big|$, our arguments presented in Appendix~\ref{sec:magnitude_check} show that this upper bound is sharp, and the order of the error is in fact $\Theta_p(\frac{\|\bbeta^*\|^2}{n})$. 
\end{remark}

\begin{remark}\label{remark:negatvie:1}
If we compare the error term $C^{\rm Linear} \frac{\|\bbeta^*\|^2}{n}$ with the order of $\mathcal{I}^{\mathrm{True}} (\bz_i,\bz_i)$ in Proposition \ref{theo:I_true_corr}, we notice that the error is very large and has the same order as the quantity that we wanted to approximate. This is consistent with the ``Negative result" that we reported in the introduction. 
\end{remark}

Now to show some of the intriguing features of the linearization, our next theorem establishes the accuracy of $\mathcal{I}^{\rm Linear} (\bz_i, \bz_{\rm new})$ for the case that $\bz_i$ and $\bz_{\rm new}$ are independent. 
\begin{theorem}\label{theo:step1}
Suppose that $\bz_{\rm new}$ is independent of the training set $\mathcal{D}$. Under Assumptions~\ref{ass:main_assumptions} (formally, Assumptions~\ref{as1}--\ref{ass:unique_minimizer} in the Appendix), for any $\eps > 0$ such that $\|\bbeta^*\|^2 \le n^{1-\eps}$, there exists an absolute constant $\tilde C^{\rm Linear}$ such that, with probability tending to $1$, the linearization error satisfies
\begin{equation*}
    \sup_{i \in [n]}
    \big| \mathcal{I}^{\mathrm{True}}(\bz_i,\bz_{\new}) 
    - \mathcal{I}^{\mathrm{Linear}}(\bz_i,\bz_{\new}) \big| 
    \leq \tilde C^{\rm Linear} \frac{\|\bbeta^*\|}{n^{1-\eps}}.   
\end{equation*}
\end{theorem}
The proof of Theorem~\ref{theo:step1} is shown in Appendix ~\ref{sec:thm1-proof}.

\begin{remark}\label{remark:lin:neg2}
Again, by comparing the above theorem with Proposition~\ref{theo:I_true}, we observe that the error
\[
\big| \mathcal{I}^{\mathrm{True}}(\bz_i,\bz_{\new}) - \mathcal{I}^{\mathrm{Linear}}(\bz_i,\bz_{\new}) \big|
\]
is of the same order as $\mathcal{I}^{\mathrm{True}}(\bz_i,\bz_{\new})$ itself. Consequently, this approximation does not yield quantitatively accurate estimates of $\mathcal{I}^{\mathrm{True}}(\bz_i,\bz_{\new})$. This observation is consistent with the negative result highlighted in the introduction.
\end{remark}

\begin{remark}
Despite the negative results discussed in Remarks~\ref{remark:negatvie:1} and~\ref{remark:lin:neg2}, the TRAK approximation still exhibits an intriguing positive property. Specifically, when $\bz_{\rm new}=\bz_i$, the quantity $\mathcal{I}^{\rm Linear}(\bz_i,\bz_{\rm new})$ remains large, as shown by the second part of Proposition~\ref{theo:I_true_corr} and the second part of Theorem~\ref{theo:step1_corr}. In contrast, when a data point $\bz_j$ is independent of $\bz_{\rm new}$, Proposition~\ref{theo:I_true} and Theorem~\ref{theo:step1} imply that $\mathcal{I}^{\rm Linear}(\bz_j,\bz_{\rm new})$ is of a much smaller order. As a result, even though the TRAK approximation may incur substantial absolute error, it can still reliably distinguish data points with large true influence $\mathcal{I}^{\rm True}(\bz_i,\bz_{\rm new})$ from those with negligible influence. This is consistent with the ``positive result'' mentioned in the introduction.
\end{remark}

\subsubsection{ALO step}\label{ssec:ALO}

In this section, we characterize the error introduced by the ALO approximation. If we ignore the projection step of the TRAK method ($k=d$), the influence function produced by TRAK reduces to
\begin{equation}
    \mathcal{I}^{\mathrm{ALO}}(\bz_i,\bz_{\new}) \triangleq \frac{
        \ld_i(\brb)\, \bg_{\mathrm{new}}^\top 
        \mathbf{H}^{-1} \bg_{i}
    }{
        1 - \ldd_i(\brb)\, \bg_{i}^\top 
        \mathbf{H}^{-1} \bg_{i}
    },
    \label{eq::inf-alo}
\end{equation}
where $\mathbf{H}=\bG^\top \diag[\bm\ldd(\brb)]\bG$, $\ld_i(\bbeta) = \ld(y_i,\bg_i^\top \bbeta +\bb_i)$ and $\ldd_i(\bbeta) = \ldd(y_i,\bg_i^\top \bbeta +\bb_i)$. In this section we would like to characterize the difference between $\mathcal{I}^{\rm ALO} (\bz_i , \bz_{\rm new})-\mathcal{I}^{\rm Linear} (\bz_i , \bz_{\rm new}) $.  The next two theorems study this error for both correlated and independent cases.

\begin{theorem}
   Under Assumptions~\ref{ass:main_assumptions} (formally, Assumptions~\ref{as1}--\ref{ass:unique_minimizer} in the Appendix), for any $\eps > 0$ such that $\|\bbeta^*\|^2 \le n^{1-\eps}$, there exists an absolute constant $\tilde C^{\rm{ALO}}$ such that, with probability tending to $1$, the ALO step error satisfies:
    \begin{equation*}
    \begin{split}
        \sup_{i \in [n]}\biggl| &
            \mathcal{I}^{\mathrm{Linear}}(\bz_i,\bz_i)
            -  \mathcal{I}^{\mathrm{ALO}}  (\bz_i,\bz_i)
        \biggr| \leq \tilde C^{\rm ALO} \frac{\|\bbeta^*\|^2}{n^{1.5-\eps}}.
    \end{split}
    \end{equation*}
    \label{thm:g_base_ALO_ii}
\end{theorem}
The proof of Theorem~\ref{thm:g_base_ALO_ii} is postponed to Appendix~\ref{sec:thm2-proof-ii}. 
\begin{theorem}
   Suppose that $\bz_{\rm new}$ is independent of the training set $\mathcal{D}$. Under Assumptions~\ref{ass:main_assumptions} (formally, Assumptions~\ref{as1}--\ref{ass:unique_minimizer} in the Appendix), for any $\eps > 0$ such that $\|\bbeta^*\|^2 \le n^{1-\eps}$, there exists an absolute constant $C^{\rm{ALO}}$ such that, with probability tending to $1$, the ALO step error satisfies:
    \begin{equation*}
    \begin{split}
        \sup_{i \in [n]}\biggl| &
            \mathcal{I}^{\mathrm{Linear}}(\bz_i,\bz_{\new})
            -  \mathcal{I}^{\mathrm{ALO}}  (\bz_i,\bz_{\new})
        \biggr| \leq C^{\rm ALO} \frac{\|\bbeta^*\|}{n^{1.5-\eps}}.
    \end{split}
    \end{equation*}
    \label{thm:g_base_ALO}
\end{theorem}
The proof of Theorem \ref{thm:g_base_ALO} is postponed to Appendix \ref{sec:thm2-proof}. 
\begin{remark}
Combining the above two theorems, we can conclude that in both cases, we have
\[
\frac{\mathcal{I}^{\mathrm{ALO}}  (\bz_i,\bz_{\new})}{\mathcal{I}^{\mathrm{Linear}}  (\bz_i,\bz_{\new})} = 1+o_p(1).
\]
Hence, the error introduced by the ALO step is negligible compared to the magnitude of the influence.
\end{remark}

\subsubsection{Projection step}\label{ssec:projection}

Finally, we study the impact of the projection step on the accuracy of the TRAK algorithm.

\begin{theorem}
\label{theo:magnitude_projection}
Suppose that $\bz_{\rm new}$ is independent of the training set $\mathcal{D}$. Under Assumptions~\ref{ass:main_assumptions} (formally, Assumptions~\ref{as1}--\ref{ass:unique_minimizer} in the Appendix), for any $\eps > 0$ such that $\|\bbeta^*\|^2 \le n^{1-\eps}$, there exists an absolute constant $C^{\rm TRAK}$ such that, with probability tending to $1$, the projection error satisfies
\[
\sup_{i\in [n]}\bigl| \mathcal{I}^{\mathrm{TRAK}}(\bz_i,\bz_{\new};k) \bigr|
\;\le\;
C^{\rm TRAK} \frac{\|\bbeta^*\| \sqrt{k}}{n^{1-\eps} \sqrt{d}}.
\]
Furthermore, there is an absolute constant $c^{\rm TRAK}$ such that with probability tending to $1$,
\[
\inf_{i\in [n]}\biggl| \frac{\mathcal{I}^{\mathrm{TRAK}}(\bz_i,\bz_i;k)}{\mathcal{I}^{\mathrm{ALO}}(\bz_i,\bz_i)} \biggr|
\;\ge\;
c^{\rm TRAK} \frac{k}{d}.
\]
\end{theorem}
The proof is presented in Appendix \ref{proof:projection}. ~\Cref{fig:magnitude} in Appendix \ref{sec:magnitude_check} confirms the sharpness of these orders.

\begin{remark}
\label{remark:projection}
The projection dimension $k$ controls whether TRAK can preserve the separation between dependent and independent data points. 
By \Cref{theo:step1_corr} and \Cref{thm:g_base_ALO_ii}, in the dependent case $\bz_i=\bz_{\new}$, the ALO influence is of the order
\(
    \gtrsim \frac{\|\bbeta^*\|^2}{n^{1-\eps}} .
\)
After random projection, \Cref{theo:magnitude_projection} shows that this quantity is reduced by a factor of order $k/d$. Hence,
\(
\mathcal{I}^{\mathrm{TRAK}}(\bz_i,\bz_i;k)
    \gtrsim 
    \frac{\|\bbeta^*\|^2 k}{n^{1-\eps}d}.
\)

The relevant comparison is with the independent case, where $\bz_i$ and $\bz_{\new}$ are independent. In that case, the projected influence is of a smaller intrinsic order before projection, but it is reduced only by a factor of order $\sqrt{k/d}$ rather than $k/d$. Therefore, when $k/d$ is very small, the dependent influence can be compressed more severely than the independent influence. As a result, the separation between the two regimes may disappear. In particular, \Cref{theo:magnitude_projection} implies that if
\(
\|\bbeta^*\|\sqrt{k/d} \ll 1,
\)
then the projected dependent influence $\mathcal{I}^{\mathrm{TRAK}}(\bz_i,\bz_i;k)$ can be no larger than the projected influence for an independent test point $\mathcal{I}^{\mathrm{TRAK}}(\bz_i,\bz_{\new};k)$. In this regime, TRAK may fail to distinguish truly influential training points from independent ones. Thus, to preserve the separation needed for reliable ranking, the projection dimension should satisfy $k \gg \frac{d}{\|\bbeta^*\|^2}.$ 
\end{remark}

% \begin{theorem}
%    Under Assumptions~\ref{ass:main_assumptions} (formally, Assumptions~\ref{as1}--\ref{ass:nsg_grad} in the Appendix), for any $\eps > 0$ such that $\|\bbeta^*\|^2 \le n^{1-\eps}$, 
%    % and \(\log n \ll k = c_{\rm TRAK}d\) for an absolute constant \(c_{\rm TRAK}, 0 < c_{\rm TRAK} < 1 \) ; recall \(d\) is the dimension of the gradients \(g_i\), 
%    there exist an absolute constant $C^{\rm{TRAK}}>1$ such that, with probability tending to $1$, the projection step error satisfies:
%     \begin{equation*}
%     \begin{split}
%        c_{k,d}|\mathcal{I}^{\mathrm{ALO}}(\bz_i,\bz_{\new})|\leq |\mathcal{I}^{\mathrm{TRAK}}(\bz_i,\bz_{\new}) - \mathcal{I}^{\mathrm{ALO}}(\bz_i,\bz_{\new})| 
%     \\
%     \le C_{k,d}|\mathcal{I}^{\mathrm{ALO}}(\bz_i,\bz_{\new})| .
%     \end{split}
%     \end{equation*}
%     where $c_{k,d}=\left(1-\kappa'{\poly(\log n)}(1+\eps)\frac{k}{d}\right)^+$ and $C_{ k,d}=\left(1 -\frac{1}{\kappa'}(1-\eps) \frac{k}{d}\right)$ and $\kappa'\geq 1$ is a constant.
%     \label{theo:magnitude_projection}
% \end{theorem}
% Theorem~\ref{theo:magnitude_projection} provides our magnitude-side conclusion. It indicates that random projections incur an error of the same order as \(\mathcal{I}^{\rm True}\) unless $k$ us super close to $d$. Or even if $k\asymp d$, we can not make sure to preserve magnitude after projection. This is intuitive because, in the regime where \(n > d\), all \(d\) dimensions of the linearized coefficients are important. Thus, projection to a \(k\)-dimensional subspace introduces non-negligible error.

\subsection{Discussion of our theoretical results}

Throughout this paper, we evaluate the accuracy of TRAK in two distinct settings: 
(i) Correlated case: the test data point is identical to one of the training samples, and 
(ii) Independent case: the test data point is independent of all training samples. 

These represent two extreme scenarios, which naturally raises the question of how TRAK behaves in intermediate regimes. It turns out that more general cases can be analyzed as interpolations between the two extremes we have studied above. Appendix~\ref{ssec:general_corr} and Appendix~\ref{sec:result_general_corr} provide further details on this point. 

The main conclusion of Appendix~\ref{ssec:general_corr} is that, under mild conditions, if the correlation between $\bx_{\mathrm{new}}$ and $\bx_i$ is $\alpha_n$, then the magnitude of the influence function scales by a factor of $\alpha_n$. Note that this claim is true as long as $\alpha_n I^{\rm True} (\bz_i, \bz_i) = \Omega (I^{\rm True} ( \bz_i, \tilde{\bz}_{\rm new}))$ for independent $\tilde{\bz}_{\rm new}$. Similarly, the errors introduced at each approximation step of TRAK are also scaled by the same factor. In other words, as long as correlations are strong enough, TRAK can distinguish correlated data points from independent ones.

\section{Simulation results}
\label{sec:simulation}

In this section, we focus on nonlinear multi-class classification. We also perform simulations for binary classification and Poisson regression. However, we present the latter two experiments in Appendix~\ref{sec:sim_linear_main}.

Our main goal is to evaluate whether $\mathcal{I}^{\rm TRAK}$ provides reliable rankings compared with $\mathcal{I}^{\rm True}$. We compare the uncorrelated case, where $\bz_i$ and $\bz_{\new}$ are independent, with the correlated case, where $\bz_i=\bz_{\new}$, and examine whether this order-level separation is preserved through the three TRAK approximation steps: linearization, ALO, and random projection. To save space, we present the full experimental setup in Appendix~\ref{sec:K5_corr}. We should emphasize that the multi-class model is inherently nonlinear (See Appendix~\ref{sec:multi-class} for more details).
\begin{figure}[ht]
    \centering
    \includegraphics[width=1\linewidth]{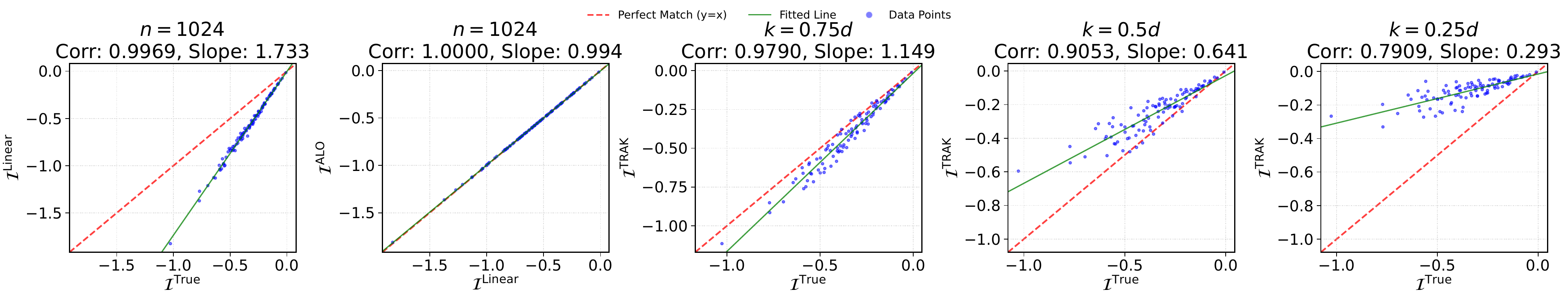}
    \includegraphics[width=1\linewidth]{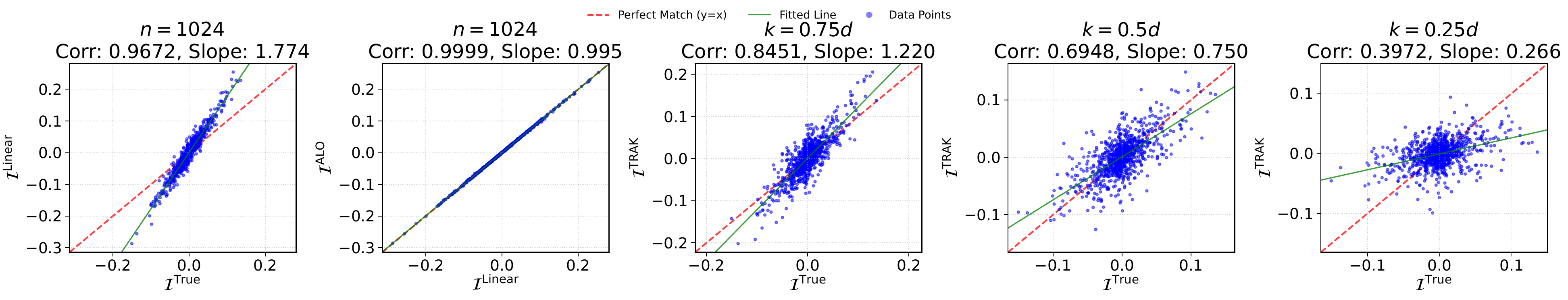}
    \caption{
Experimental results for $3$-class classification with $p=100$ and $d=200$.
\textbf{Top row:} Correlated case, where $\bz_i=\bz_{\new}$.
\textbf{Bottom row:} Uncorrelated case, where $\bz_i$ and $\bz_{\new}$ are independent.
\textbf{First column:} $\mathcal{I}^{\rm Linear}$ versus $\mathcal{I}^{\rm True}$.
\textbf{Second column:} $\mathcal{I}^{\rm ALO}$ versus $\mathcal{I}^{\rm Linear}$.
\textbf{Right three columns:} $\mathcal{I}^{\rm TRAK}(k)$ versus $\mathcal{I}^{\rm True}$ with $k=75\%d,50\%d,25\%d$ ($k=150,100,50$).
}
    \label{fig:multi_alo}
\end{figure}
\paragraph{Linearization step.}
The first column of Figure~\ref{fig:multi_alo} plots $\mathcal{I}^{\rm Linear}$ against $\mathcal{I}^{\rm True}$ in the correlated and uncorrelated cases. The visible deviation from the red $y=x$ line indicates a substantial absolute discrepancy between the two quantities, consistent with Theorems~\ref{theo:step1_corr} and~\ref{theo:step1}. Nevertheless, $\mathcal{I}^{\rm Linear}$ remains strongly aligned with $\mathcal{I}^{\rm True}$: data points with larger true influence also tend to have larger linearized influence. Moreover, the correlated case has a much larger scale than the uncorrelated case, matching the order gap predicted by Propositions~\ref{theo:I_true_corr} and~\ref{theo:I_true}, as well as Theorems~\ref{theo:step1_corr} and~\ref{theo:step1}. Thus, although the linearization step introduces non-negligible magnitude error, it largely preserves the separation structure needed for reliable ranking.
\paragraph{ALO step.}
The second column of Figure~\ref{fig:multi_alo} plots $\mathcal{I}^{\rm ALO}$ against $\mathcal{I}^{\rm Linear}$. The two quantities are nearly indistinguishable in both regimes, showing that the ALO step introduces only negligible error. This agrees with Theorems~\ref{thm:g_base_ALO_ii} and~\ref{thm:g_base_ALO}. Therefore, the ALO step preserves the separation already present after linearization and is not the main source of ranking distortion in this experiment.

\paragraph{Projection step.}
The right three columns of Figure~\ref{fig:multi_alo} show $\mathcal{I}^{\rm TRAK}(k)$ plotted against $\mathcal{I}^{\rm True}$ for $k=0.75d$, $k=0.50d$, and $k=0.25d$. As predicted by Theorem~\ref{theo:magnitude_projection}, the magnitude of $\mathcal{I}^{\rm TRAK}(k)$ decreases as $k/d$ decreases. Furthermore, smaller projection dimensions reduce the separation between the two cases; this effect is further quantified in Table~\ref{tab:in_out_quantile_ratio}. When $k/d=0.25$, the separation becomes much less pronounced, suggesting that small projection dimensions may fail to preserve the separation between influential and independent data points, as discussed in Remark~\ref{remark:projection}. Moreover, as $k$ decreases, the correlation between $\mathcal{I}^{\rm True}$ and $\mathcal{I}^{\rm TRAK}(k)$ drops quickly, further showing that projection can damage  performance, especially when the projection dimension is small.
\paragraph{Quantitative check.}
Table~\ref{tab:in_out_quantile_ratio} further quantifies how the separation between the correlated and uncorrelated cases changes across the TRAK approximation steps. For $\mathcal{I}^{\rm True}$, $\mathcal{I}^{\rm Linear}$, and $\mathcal{I}^{\rm ALO}$, the separation remains stable: the correlated $5\%$ quantile is still about $1.5$ times larger than the uncorrelated $95\%$ quantile. Thus, the true influence, linearized influence, and ALO approximation all maintain a clear magnitude gap between the two regimes.

In contrast, this separation weakens as the projection dimension $k$ decreases. The most stringent ratio, Corr. $5\%$/Uncorr. $95\%$, decreases from $1.42$ at $k=0.75d$ to $0.89$ at $k=0.25d$, and further to $0.42$ at $k=0.1d$. Hence, under aggressive random projection, the lower tail of the correlated magnitudes can overlap with, or even fall below, the upper tail of the uncorrelated magnitudes. This is consistent with Theorem~\ref{theo:magnitude_projection}: in the correlated case, $\mathcal{I}^{\rm TRAK}(k)$ scales as $k/d$, whereas in the uncorrelated case, it scales as $\sqrt{k/d}$. Since $\sqrt{k/d}$ decreases more slowly than $k/d$, small projection dimensions reduce the separation between the two regimes and can therefore harm ranking performance. For example, comparing $\mathcal{I}^{\rm TRAK}(0.5d)$ with $\mathcal{I}^{\rm TRAK}(0.25d)$, the projection dimension is reduced by a factor of two. Correspondingly, the correlated $5\%$ and $95\%$ quantiles both decrease by roughly a factor of two, whereas the uncorrelated quantiles decrease by a smaller factor, close to the predicted $1/\sqrt{2}$ scaling.
\begin{table}[ht]
\centering
\caption{
The $5\%$ and $95\%$ quantiles of absolute influence values in the correlated (Corr.) and uncorrelated (Uncorr.) cases for $3$-class classification with $p=100$ and $d=200$. 
The last three columns report separation ratios. 
The ratio Corr. $5\%$/Uncorr. $95\%$ is the most stringent diagnostic: values above one indicate clear separation between the lower tail of the correlated magnitudes and the upper tail of the uncorrelated magnitudes, while values below one indicate overlap between the two regimes.
}
\label{tab:in_out_quantile_ratio}
\resizebox{\textwidth}{!}{
\begin{tabular}{lccccccc}
\toprule
Quantity 
& Corr. $5\%$ 
& Corr. $95\%$ 
& Uncorr. $5\%$ 
& Uncorr. $95\%$ 
& Corr. $5\%$/Uncorr. $5\%$ 
& Corr. $95\%$/Uncorr. $95\%$ 
& Corr. $5\%$/Uncorr. $95\%$ \\
\midrule
$\mathcal{I}^{\rm True}$ 
& $1.1097\times 10^{-1}$ 
& $6.0677\times 10^{-1}$ 
& $7.2078\times 10^{-4}$ 
& $6.6671\times 10^{-2}$ 
& $153.96$ 
& $9.10$ 
& $1.66$ \\

$\mathcal{I}^{\rm Linear}$ 
& $1.8015\times 10^{-1}$ 
& $1.0876$ 
& $1.9678\times 10^{-3}$ 
& $1.2082\times 10^{-1}$ 
& $91.55$ 
& $9.00$ 
& $1.49$ \\

$\mathcal{I}^{\rm ALO}$ 
& $1.8033\times 10^{-1}$ 
& $1.0860$ 
& $1.7744\times 10^{-3}$ 
& $1.2088\times 10^{-1}$ 
& $101.63$ 
& $8.98$ 
& $1.49$ \\

$\mathcal{I}^{\rm TRAK}(0.75d)$ 
& $1.3540\times 10^{-1}$ 
& $7.5510\times 10^{-1}$ 
& $1.2054\times 10^{-3}$ 
& $9.5666\times 10^{-2}$ 
& $112.34$ 
& $7.89$ 
& $1.42$ \\

$\mathcal{I}^{\rm TRAK}(0.5d)$ 
& $8.1372\times 10^{-2}$ 
& $4.8031\times 10^{-1}$ 
& $1.3459\times 10^{-3}$ 
& $7.3058\times 10^{-2}$ 
& $60.46$ 
& $6.57$ 
& $1.11$ \\

$\mathcal{I}^{\rm TRAK}(0.25d)$ 
& $4.1744\times 10^{-2}$ 
& $2.3990\times 10^{-1}$ 
& $9.4412\times 10^{-4}$ 
& $4.6958\times 10^{-2}$ 
& $44.21$ 
& $5.11$ 
& $0.89$ \\

$\mathcal{I}^{\rm TRAK}(0.1d)$ 
& $1.2220\times 10^{-2}$ 
& $1.1411\times 10^{-1}$ 
& $6.2861\times 10^{-4}$ 
& $2.9211\times 10^{-2}$ 
& $19.44$ 
& $3.91$ 
& $0.42$ \\
\bottomrule
\end{tabular}
}
\end{table}
\paragraph{Additional empirical validation.}
We also conduct experiments on CIFAR-10  \citep{krizhevsky2009learning} to assess whether the conclusions above remain robust on real-world data that may not satisfy our data-generation assumptions. The results, deferred to Appendix~\ref{sec:empirical}, show qualitatively similar behavior: the linearization step introduces non-negligible magnitude distortion but retains substantial rank information, the ALO step is benign, and the projection step can weaken ranking performance when the projection dimension is small.
\section{Conclusion and future directions}
\label{sec:conclusion}
TRAK is a widely used data attribution method that attains scalability through three approximations: linearization, approximate leave-one-out (ALO), and projection. In this work, we theoretically analyze the errors introduced by each step. We show that the ALO approximation incurs negligible error, while the linearization step introduces substantial error but still preserves the ability to distinguish highly influential data points from uninfluential ones. In contrast, the projection step can introduce significant additional error, and using too few projections may prevent TRAK from reliably identifying influential points. Overall, our results clarify both the strengths and limitations of TRAK and provide guidance for its effective use at scale.

Our formal results are developed for the underparameterized regime
$d<n$. Appendix~\ref{sec:overparameterized-ridge} outlines how ridge
regularization modifies the TRAK formulation when $d>n$. Nevertheless,
a complete theoretical treatment of the overparameterized regime remains
an important direction for future work.

\section*{Acknowledgments and Disclosure of Funding}

Arian Maleki would like to declare the following NSF funding:
NSF-DMS 2515716 and NSF-DMS 2210506.

The authors declare no competing interests.

\bibliography{ref-Neurips}

@article{beirami2017optimal,
  title={On optimal generalizability in parametric learning},
  author={Beirami, Ahmad and Razaviyayn, Meisam and Shahrampour, Shahin and Tarokh, Vahid},
  journal={Advances in Neural Information Processing Systems},
  volume={30},
  year={2017}
}

@misc{krizhevsky2009learning,
  title={The CIFAR-10 dataset},
  author={Krizhevsky, Alex and Nair, Vinod and Hinton, Geoffrey and others},
  year={2014}
}

@InProceedings{pmlr-v206-fisher23a,
  title = 	 {Influence Diagnostics under Self-concordance},
  author =       {Fisher, Jillian and Liu, Lang and Pillutla, Krishna and Choi, Yejin and Harchaoui, Zaid},
  booktitle = 	 {Proceedings of The 26th International Conference on Artificial Intelligence and Statistics},
  pages = 	 {10028--10076},
  year = 	 {2023},
  editor = 	 {Ruiz, Francisco and Dy, Jennifer and van de Meent, Jan-Willem},
  volume = 	 {206},
  series = 	 {Proceedings of Machine Learning Research},
  month = 	 {25--27 Apr},
  publisher =    {PMLR},
  url = 	 {https://proceedings.mlr.press/v206/fisher23a.html}
}

@InProceedings{koh2017understanding,
  title     = {Understanding Black-box Predictions via Influence Functions},
  author    = {Koh, Pang Wei and Liang, Percy},
  booktitle = {Proceedings of the 34th International Conference on Machine Learning},
  pages     = {1885--1894},
  year      = {2017},
  editor    = {Precup, Doina and Teh, Yee Whye},
  volume    = {70},
  series    = {Proceedings of Machine Learning Research},
  publisher = {PMLR}
}

@inproceedings{xia2024less,
   title={{LESS}: Selecting Influential Data for Targeted Instruction Tuning},
   author={Xia, Mengzhou and Malladi, Sadhika and Gururangan, Suchin and Arora, Sanjeev and Chen, Danqi},
   booktitle={International Conference on Machine Learning (ICML)},
   year={2024}
}

@article{hammoudeh2024training,
  title={Training data influence analysis and estimation: A survey},
  author={Hammoudeh, Zayd and Lowd, Daniel},
  journal={Machine Learning},
  volume={113},
  number={5},
  pages={2351--2403},
  year={2024},
  publisher={Springer}
}

@article{ye2023cognitive,
  title={Cognitive mirage: A review of hallucinations in large language models},
  author={Ye, Hongbin and Liu, Tong and Zhang, Aijia and Hua, Wei and Jia, Weiqiang},
  journal={arXiv preprint arXiv:2309.06794},
  year={2023}
}

@article{jin2019short,
  title   = {A short note on concentration inequalities for random vectors with subgaussian norm},
  author  = {C. Jin and P. Netrapalli and R. Ge and S. M. Kakade and M. I. Jordan},
  journal = {arXiv preprint arXiv:1902.03736},
  year    = {2019}
}

@inproceedings{park2023trak,
  title        = {{TRAK}: Attributing Model Behavior at Scale},
  author       = {S. M. Park and K. Georgiev and A. Ilyas and G. Leclerc and A. Madry},
  booktitle    = {Proceedings of the 40th International Conference on Machine Learning (ICML)},
  pages        = {27074--27113},
  year         = {2023},
  publisher    = {PMLR}
}

@article{rad2018scalable,
  title   = {A scalable estimate of the extra-sample prediction error via approximate leave-one-out},
  author  = {K. R. Rad and A. Maleki},
  journal = {Journal of the Royal Statistical Society Series B: Statistical Methodology},
  volume  = {82},
  number  = {4},
  pages   = {965--996},
  year    = {2020},
  publisher={Oxford University Press}
}

@inproceedings{sagun2017empirical,
  title     = {Empirical Analysis of the {H}essian of Over-Parametrized Neural Networks},
  author    = {L. Sagun and U. Evci and V. U. G{\"u}ney and Y. N. Dauphin and L. Bottou},
  booktitle = {Proceedings of the 6th International Conference on Learning Representations (ICLR) Workshop Track},
  year      = {2018}
}

@inproceedings{han2020explaining,
    title = "Explaining Black Box Predictions and Unveiling Data Artifacts through Influence Functions",
    author = "Han, Xiaochuang  and
      Wallace, Byron C.  and
      Tsvetkov, Yulia",
    editor = "Jurafsky, Dan  and
      Chai, Joyce  and
      Schluter, Natalie  and
      Tetreault, Joel",
    booktitle = "Proceedings of the 58th Annual Meeting of the Association for Computational Linguistics",
    month = jul,
    year = "2020",
    address = "Online",
    publisher = "Association for Computational Linguistics",
    url = "https://aclanthology.org/2020.acl-main.492/",
    doi = "10.18653/v1/2020.acl-main.492",
    pages = "5553--5563",
}

@article{yeh2019fidelity,
  title={On the (in)fidelity and sensitivity of explanations},
  author={Yeh, Chih-Kuan and Hsieh, Cheng-Yu and Suggala, Arun and Inouye, David I and Ravikumar, Pradeep K},
  journal={Advances in Neural Information Processing Systems},
  volume={32},
  year={2019}
}

@inproceedings{ghorbani2019data,
  title={Data {Shapley}: Equitable valuation of data for machine learning},
  author={Ghorbani, Amirata and Zou, James},
  booktitle={International conference on machine learning},
  pages={2242--2251},
  year={2019},
  organization={PMLR}
}

@inproceedings{jia2019towards,
  title={Towards efficient data valuation based on the {Shapley} value},
  author={Jia, Ruoxi and Dao, David and Wang, Boxin and Hubis, Frances Ann and Hynes, Nick and G{\"u}rel, Nezihe Merve and Li, Bo and Zhang, Ce and Song, Dawn and Spanos, Costas J},
  booktitle={The 22nd International Conference on Artificial Intelligence and Statistics},
  pages={1167--1176},
  year={2019},
  organization={PMLR}
}

@inproceedings{wang2024data,
  title={Data {Shapley} in One Training Run},
  author={Wang, Jiachen T. and Mittal, Prateek and Song, Dawn and Jia, Ruoxi},
  booktitle={Proceedings of the Thirteenth International Conference on Learning Representations},
  year={2025}
}

@article{hampel1974influence,
  title={The influence curve and its role in robust estimation},
  author={Hampel, Frank R},
  journal={Journal of the American Statistical Association},
  volume={69},
  number={346},
  pages={383--393},
  year={1974},
  publisher={Taylor \& Francis}
}

@inproceedings{ilyas2022datamodels,
  title={Datamodels: Predicting Predictions from Training Data},
  author={Ilyas, Andrew and Park, Sung Min and Engstrom, Logan and Leclerc, Guillaume and Madry, Aleksander},
  booktitle={Proceedings of the 39th International Conference on Machine Learning},
  year={2022}
}

@article{guu2023simfluence,
  title={Simfluence: Modeling the influence of individual training examples by simulating training runs},
  author={Guu, Kelvin and Webson, Albert and Pavlick, Ellie and Dixon, Lucas and Tenney, Ian and Bolukbasi, Tolga},
  journal={arXiv preprint arXiv:2303.08114},
  year={2023}
}

@inproceedings{hammoudeh22identifying,
  title={Identifying a training-set attack's target using renormalized influence estimation},
  author={Hammoudeh, Zayd and Lowd, Daniel},
  booktitle={Proceedings of the 2022 ACM SIGSAC Conference on Computer and Communications Security},
  pages={1367--1381},
  year={2022}
}

@article{auddy24a,
  title = 	 {Approximate Leave-one-out Cross Validation for Regression with  {$\ell_1$} Regularizers},
  author =       {Auddy, Arnab and Zou, Haolin and Rahnama Rad, Kamiar and Maleki, Arian},
  journal = {IEEE Transactions on Information Theory},
  pages = 	 {8040--8071},
  year = 	 {2024},
  volume = 	 {70},
  number = {11}
}

@book{anderson1958introduction,
  title     = {An Introduction to Multivariate Statistical Analysis},
  author    = {Anderson, Theodore W.},
  year      = {1958},
  publisher = {Wiley}
}

@book{vershynin2009high,
  title={High-Dimensional Probability: An Introduction with Applications in Data Science},
  author={Vershynin, Roman},
  year={2018},
  publisher={Cambridge University Press}
}
\bibliographystyle{plainnat}
\newpage
\appendix
\section*{Appendix}

\addcontentsline{toc}{section}{Appendix}
\subsection*{Organization of the Appendix}
\begin{itemize}
    \item \textbf{Section~\ref{sec:method_details}} summarizes additional methodological details.
    In particular, in \Cref{sec:multi-class} we discuss how to transform a multi-class classification problem into a nonlinear scalar-output function $f(\cdot,\cdot)$.

    \item \textbf{Section~\ref{sec:all_technical}} summarizes all theoretical and technical details of the paper.
    {Section~\ref{sec:def+ass}} collects all assumptions used throughout the paper.
    We further verify that these assumptions hold for common GLM loss functions in Section~\ref{sec:corder_check},
    and show that they are satisfied by linear models and one-layer neural networks in Section~\ref{sec:output}.
    All proofs are provided in Section~\ref{sec:proofs}, and all technical lemmas are collected in Section~\ref{sec:lemma}.
    A roadmap of our main theorems is summarized in Table~\ref{tab:appendix_roadmap}.

    \item \textbf{Section~\ref{sec:simresults}} presents additional simulation results that further verify our theoretical findings.
    Section~\ref{sec:corr-ind} focuses on correlation-based analysis, while Section~\ref{sec:magnitude_check} focuses on magnitude-based analysis.
    Section~\ref{sec:empirical} presents results on real-world empirical datasets. Finally, Section~\ref{sec:result_general_corr} presents results for more general correlated settings.
\end{itemize}

\begin{table}[ht]
\centering
\caption{Roadmap of theoretical results in the Appendix.}
\label{tab:appendix_roadmap}
\begin{tabular}{c cccc}
\toprule
 & \multicolumn{4}{c}{\textbf{Influence Function Approximation}} \\
\cmidrule(lr){2-5}
 & \textbf{True Influence}
 & \textbf{Linearization Step}
 & \textbf{ALO Step}
 & \textbf{Projection Step} \\
\midrule

\multirow{2}{*}{\textbf{Theorem}}
\;\;$\bz_i=\bz_{\new}$
&
\Cref{theo:I_true_corr}
&
\Cref{theo:step1_corr}
&
\Cref{thm:g_base_ALO_ii}
&
\multirow{2}{*}{\Cref{theo:magnitude_projection}}
\\
\;\;\;\;\;\;\;\;\;\;\;\;\;\;\;\;\;Independent
&
\Cref{theo:I_true}
&
\Cref{theo:step1}
&
\Cref{thm:g_base_ALO}
&
\\

\midrule

\multirow{2}{*}{\textbf{Proof}}
\;\;\;\;\;\;\;$\bz_i=\bz_{\new}$
&
Appendix~\ref{proof:upper:inf:corr}
&
Appendix~\ref{proof:step1_corr}
&
Appendix~\ref{sec:thm2-proof-ii}
&
\multirow{2}{*}{Appendix~\ref{proof:projection}}
\\
\;\;\;\;\;\;\;\;\;\;\;\;\;\;\;\;\;Independent
&
Appendix~\ref{sec:prop-proof}
&
Appendix~\ref{sec:thm1-proof}
&
Appendix~\ref{sec:thm2-proof}
&
\\

\bottomrule
\end{tabular}
\end{table}

\newpage
\section{More details}
\label{sec:method_details}
\subsection{Multi-class classification extension}
\label{sec:multi-class}
The goal of this section is to extend our analysis to the multiclass classification setting. We consider a dataset with $K$ classes and employ the standard cross-entropy loss,
\[
\mathcal{L}_{\mathrm{CE}}(\bz, \bbeta) = -\log p_y(\bz, \bbeta),
\]
where $p_y(\bz, \bbeta)$ denotes the model’s predicted probability for the true class label $y$ associated with the data point $\bz$.
Since we consider a multiclass classification problem, the model must produce class probabilities for all $K$ classes. Even in the simplest setting, the model parameters therefore take the form
\[
\bW = (\bW_1, \ldots, \bW_{K-1}, \bW_K)^\top,
\]
where each $\bW_k \in \mathbb{R}^p$ corresponds to the weight vector for class $k$. For identifiability, we fix $\bW_K = \mathbf{0}_p$. We then define
\[
\bbeta := \mathrm{vec}\!\big(\bW_{[1:(K-1),:]}\big) \in \mathbb{R}^d,
\]
where the parameter dimension is $d = (K-1)p > p$; thus, this model is no longer a simple linear model, unlike the binary classification case.
% \textcolor{red}{Unclear which complexity you are addressing. Can you clarify? }
To address this complexity, we employ a method similar to that used in \cite{park2023trak}. Specifically, we define:
\begin{equation*}
    f(\bz, \bbeta) = \log\left(\frac{p(\bz, \bbeta)}{1-p(\bz, \bbeta)}\right), \ \ p(\bz, \bbeta) = \frac{e^{\bW_y^\top\bx}}{\sum_{j=1}^K e^{\bW_j^\top\bx}},
\end{equation*}
where we set $\bW_K = \mathbf{0}_p$. The loss function can then be expressed as:
\begin{equation*}
    \mathcal{L}_{CE}(\bz, \bbeta) = \log(1+e^{-f(\bz, \bbeta)}) = l(f(\bz, \bbeta)),
\end{equation*}
where $l(s) = \log(1+e^{-s})$.
We can further express the gradient as:
\begin{equation*}
    \bg = \nabla_{\bbeta} f(\bz, \bl) = \frac{(\one_y - \mathbf{p}(\bz, \bl))}{1-p(\bz,\bl)} \otimes \bx,
\end{equation*}
where $\one_y$ is a vector of length $K-1$ with entries ${\one_y}_{[k]} = 1$ if $y=k$ and $0$ otherwise (note that if $y=K$, $\one_y$ is a zero vector). Similarly, $\mathbf{p}(\bz, \bbeta) = \left(\frac{\exp(\bW_1^\top\bx)}{\sum_{j=1}^K \exp(\bW_j^\top\bx)}, \dots, \frac{\exp(\bW_{K-1}^\top\bx)}{\sum_{j=1}^K \exp(\bW_j^\top\bx)}\right)^\top$ is also a vector of length $K-1$. Finally, we obtain $\bg \in \mathbb{R}^{d}$ with $d=(K-1)p$.

\subsection{Overparameterized regime with ridge regularization}
\label{sec:overparameterized-ridge}

Our main analysis considers the underparameterized regime $d<n$. Here,
we outline a ridge-regularized formulation for $d>n$ and indicate which
parts of the preceding analysis may continue to apply.
Consider
$$
\bl_\lambda
=
\arg\min_{\bbeta\in\mathbb R^d}
\left\{
\frac{1}{n}\sum_{j=1}^n
\ell\bigl(y_j,f(\bx_j,\bbeta)\bigr)
+
\frac{\lambda}{2}\|\bbeta\|_2^2
\right\},
\qquad \lambda>0.
$$
Using the unnormalized convention, define
$$
\bm H_\lambda
=
\bG^\top\diag[\bm\ldd(\bl_\lambda)]\bG
+
n\lambda\II_d,
\qquad
\bm H_{\lambda,/i}
=
\bm H_\lambda
-
\ldd_i(\bl_\lambda)\bg_i\bg_i^\top.
$$
Ridge regularization ensures that these Hessians are invertible under the usual convexity conditions.

The regularized ALO and projected-ALO scores are
$$
\mathcal I_\lambda^{\rm ALO}
=
\ld_i(\bl_\lambda)
\bg_{\new}^\top
\bm H_{\lambda,/i}^{-1}
\bg_i
$$
and
$$
\mathcal I_{\lambda,k}^{\rm TRAK}
=
\ld_i(\bl_\lambda)
\bg_{\new}^\top
\bS
\bigl(\bS^\top\bm H_{\lambda,/i}\bS\bigr)^{-1}
\bS^\top\bg_i.
$$
In particular, the projected Hessian contains the projected regularizer:
$$
\bS^\top\bm H_{\lambda,/i}\bS
=
\bm\Phi_{/i}^\top
\diag[\bm\ldd_{/i}(\bl_\lambda)]
\bm\Phi_{/i}
+
n\lambda\bS^\top\bS.
$$
Consequently,
$$
\mathcal I_{\lambda,k}^{\rm TRAK}
-
\mathcal I_\lambda^{\rm ALO}
=
\ld_i(\bl_\lambda)\bg_{\new}^\top
\left[
\bS
\bigl(\bS^\top\bm H_{\lambda,/i}\bS\bigr)^{-1}
\bS^\top
-
\bm H_{\lambda,/i}^{-1}
\right]
\bg_i.
$$

Existing results show that ALO can consistently approximate LOO for a broad class of regularized estimators under suitable conditions
\citep{rad2018scalable}. Moreover, under the same conditioning and concentration assumptions used in our main analysis, the projection argument continues to produce the $k/d$ scaling in the correlated case and the $\sqrt{k/d}$ scaling in the independent case. Therefore, although ridge regularization makes the relevant Hessians invertible when $d>n$, an excessively small projection dimension can still introduce substantial magnitude error and weaken score separation.

\subsection{Ensemble averaging and soft thresholding}
\label{sec:ensemble-thresholding}

For the $m$-th submodel, write the complete projected-ALO score as
$q_{m,i}a_{m,i}$, where
$$
q_{m,i}=\ld_{m,i}(\breve{\bbeta}_m),
\qquad
a_{m,i}
=
\frac{
\bphi_{m,\new}^{\top}\bm H_m^{-1}\bphi_{m,i}
}{
1-
\ldd_{m,i}(\breve{\bbeta}_m)
\bphi_{m,i}^{\top}\bm H_m^{-1}\bphi_{m,i}
}.
$$
Direct averaging reports
$M^{-1}\sum_{m=1}^M q_{m,i}a_{m,i}$, whereas practical TRAK
\citep{park2023trak} averages the two factors separately and reports
$\bar q_i\bar a_i$. The difference between these two quantities is
$$
\bar q_i\bar a_i
-
\frac{1}{M}\sum_{m=1}^M q_{m,i}a_{m,i}
=
-\widehat{\operatorname{Cov}}_M(q_i,a_i),
$$
where
$$
\widehat{\operatorname{Cov}}_M(q_i,a_i)
=
\operatorname{Cov}(q_i,a_i)+O_p(M^{-1/2}).
$$
Thus, the covariance bias does not generally vanish with $M$.

For direct averaging over independent submodels and projections, the
correlated score is of order
$$
\frac{\|\bbeta^\ast\|^2}{n}\frac{k}{d},
$$
whereas the independent score is of order
$$
\frac{\|\bbeta^\ast\|}{n}\sqrt{\frac{k}{dM}}.
$$
Therefore, direct averaging relaxes the sufficient projection-dimension
condition from
$$
k\gg\frac{d}{\|\bbeta^\ast\|^2}
\qquad\text{to}\qquad
Mk\gg\frac{d}{\|\bbeta^\ast\|^2}.
$$

For factorized averaging, although the covariance bias may not vanish, it
does not introduce a new order. Indeed, $a_{m,i}$ contains the same
projected inner-product structure analyzed in our main results. Its
leading order is
$$
\frac{\|\bbeta^\ast\|^2}{n}\frac{k}{d}
$$
in the correlated case and
$$
\frac{\|\bbeta^\ast\|}{n}\sqrt{\frac{k}{d}}
$$
in the independent case, up to the polylogarithmic factors contributed by
$q_{m,i}$. The covariance term therefore inherits the same order
difference. Consequently, factorized averaging may introduce a
non-vanishing magnitude bias, but it does not alter the order separation
between the correlated and independent cases. However, because its
covariance term does not generally vanish with $M$, increasing $M$
reduces the random fluctuations but not this persistent magnitude bias.
This further strengthens our conclusion that order separation can be
maintained, whereas absolute magnitudes cannot generally be preserved.

Finally, define the soft-thresholding operator
$$
\mathcal T_{\tau}(u)
=
\operatorname{sign}(u)(|u|-\tau)_+.
$$
Our independent-case bound suggests choosing
$$
\tau_n
\asymp
\frac{\|\bbeta^\ast\|}{n}.
$$
With a sufficiently large constant, this threshold removes independent
scores satisfying our high-probability bound, while correlated scores
that are asymptotically larger than $\tau_n$ remain nonzero. Thus, soft
thresholding can strengthen order separation, although it does not
preserve exact magnitudes or rankings.
\section{Proofs}
\label{sec:all_technical}
\subsection{Notations and assumptions}
\label{sec:def+ass}
Inspired by recent advances in AI models with billions of parameters trained on massive datasets, we study regimes in which $n$, $p$, and $d$ are all large. Hence, throughout the paper we will use the following standard notations. 
 For deterministic sequences
$\{a_n\}$ and $\{b_n\}$, we write $a_n = O(b_n)$ if there exists a constant $C > 0$ such
that $|a_n| \le C |b_n|$ for all sufficiently large $n$, and $a_n = o(b_n)$ if
$a_n / b_n \to 0$ as $n \to \infty$. We write $a_n = \Theta(b_n)$ if
$a_n = O(b_n)$ and $b_n = O(a_n)$. For random sequences $\{X_n\}$, we write
$X_n = O_p(b_n)$ if for every $\varepsilon > 0$ there exist constants $C > 0$ and $N$
such that $\mathbb{P}(|X_n| > C b_n) < \varepsilon$ for all $n \ge N$, and
$X_n = o_p(b_n)$ if $X_n / b_n $ goes to zero in probability. 

\subsubsection{Notations and definitions}
Let $\sigma_{\min}(A)$ and $\sigma_{\max}(A)$ denote the minimum and maximum eigenvalues of a matrix $A$, respectively. A random vector $\bx \in \RR^d$ is called norm-subGaussian with parameter $\sigma$, denoted as $nSG(\sigma)$, if there exists $\sigma > 0$ such that
\[
\mathbb{P}\big(\|\bx - \mathbb{E} \bx\| \geq t\big) \leq 2e^{-\frac{t^2}{2\sigma^2}} \quad \text{for all } t \geq 0,
\]
as defined in \citet{jin2019short}.

\subsubsection{Detailed modeling assumptions}
\label{app:detailed_assumptions}

In this section, we provide the formal statements of the assumptions summarized in Section~\ref{sec:main_assumptions}. Throughout this section, let
\[
\mathcal{T}
:=
\left\{
(t_1,t_2)\in[0,1]^2:
t_1+t_2\le 1
\right\}.
\]
For each $i\in[n]$ and $(t_1,t_2)\in\mathcal{T}$, define
\[
\begin{aligned}
\bbeta_{1,i}
&=t_1\brb+t_2\brbli
 +(1-t_1-t_2)\bbeta^*,\\
\bbeta_{2,i}
&=t_1\bl+t_2\bli
 +(1-t_1-t_2)\bbeta^*.
\end{aligned}
\]

\begin{assumption}[Sub-Gaussian design]\label{as1}
The rows of $\bm{X}_n\in\mathbb{R}^{n\times p}$ are independent, zero-mean sub-Gaussian random vectors with covariance matrix $\bm{\Sigma}_p$. We assume that $\bm{\Sigma}_p$ is well-conditioned, with
\[
\rho_{\max}\asymp\rho_{\min}
\asymp\|\bbeta^*\|^{-2},
\]
where $\rho_{\max}$ and $\rho_{\min}$ denote the largest and smallest eigenvalues of $\bm{\Sigma}_p$, respectively.
\end{assumption}

\paragraph{Justification of scaling.}
When $d>p$ and $\bbeta^*=\mathrm{vec}(\bW)$ with $\bW\in\mathbb{R}^{h\times p}$, we typically have
\[
f(\bx,\bbeta^*)\asymp\|\bW\bx\|=O_p(1).
\]
Assuming that both $\bW^\top\bW$ and $\bSigma$ are well-conditioned, we obtain
\[
\|\bW\bx\|^2
\asymp
\frac{\|\bW\|_F^2}{p}\|\bx\|^2
\asymp
\frac{\|\bW\|_F^2}{p}\cdot p\|\bSigma\|
=
\|\bW\|_F^2\|\bSigma\|,
\]
which implies
\[
\|\bSigma\|\asymp\|\bbeta^*\|^{-2}.
\]

\begin{assumption}[Regularity conditions on $\nabla f$]
\label{as2}
There exists a constant $C_1=O(1)$ such that, with probability at least $1-q_n$, where $q_n\to0$, the following holds for each $j\in\{1,2\}$:
\[
\sup_{\bx\in\{\bx_{\text{new}},\bx_1,\ldots,\bx_n\}}
\sup_{i\in[n]}
\sup_{(t_1,t_2)\in\mathcal{T}}
\|\nabla f(\bx,\bbeta_{j,i})\|
\le
C_1\,\mathrm{poly}(\log n).
\]
\end{assumption}

\begin{assumption}[Regularity conditions on $\bG$ and empirical Hessian]
\label{as3}
Recall the definition of $\bG(\bbeta)$ from Eq.~\eqref{eq:Gdef_main}. There exist constants $C_2,c_1=\Theta(1)$ such that, with probability at least $1-q_n$, the following conditions hold for each $j\in\{1,2\}$:
\begin{align*}
\frac{1}{n}
\sigma_{\max}\left(
\bG(\bl)^\top\bG(\bl)
\right)
&\le C_2\|\bbeta^*\|^{-2},\\
\inf_{i\in[n]}
\inf_{(t_1,t_2)\in\mathcal{T}}
\frac{1}{n}
\sigma_{\min}\left(
\bG(\bbeta_{j,i})^\top\bG(\bbeta_{j,i})
\right)
&\ge c_1\|\bbeta^*\|^{-2},\\
\sup_{i\in[n]}
\sup_{(t_1,t_2)\in\mathcal{T}}
\frac{1}{n}
\sigma_{\max}\left(
\sum_{r=1}^n
\nabla^2\ell(y_r,f(\bx_r,\bbeta_{j,i}))
\right)
&\le C_2\|\bbeta^*\|^{-2},\\
\inf_{i\in[n]}
\inf_{(t_1,t_2)\in\mathcal{T}}
\frac{1}{n}
\sigma_{\min}\left(
\sum_{r=1}^n
\nabla^2\ell(y_r,f(\bx_r,\bbeta_{j,i}))
\right)
&\ge c_1\|\bbeta^*\|^{-2}.
\end{align*}
\end{assumption}

\begin{assumption}[Norm-sub-Gaussian gradients]
\label{ass:nsg_grad}
Let $\bg_{i,*}=\nabla f(\bx_i,\bbeta^*)$. We assume that $\{\bg_{i,*}\}_{i=1}^n$ are i.i.d. norm-sub-Gaussian random vectors with parameter $\sigma>0$, denoted by
\[
\bg_{i,*}\overset{\mathrm{iid}}{\sim}
\mathrm{nSG}(\sigma).
\]
Furthermore, let
\[
\bSigma_g^*
=
\mathbb{E}\left[
\bg_{i,*}\bg_{i,*}^\top
\right].
\]
We assume that
\[
\lambda_{\max}(\bSigma_g^*)
\asymp
\lambda_{\min}(\bSigma_g^*)
=
O(\|\bbeta^*\|^{-2}).
\]
\end{assumption}

\begin{assumption}[Regularity conditions on $\nabla^2 f$]
\label{as4}
There exists a constant $C_2=O(1)$ such that, with probability
at least $1-q_n$,
\[
\sup_{\bx\in\{\bx_{\text{new}},\bx_1,\ldots,\bx_n\}}
\sup_{i\in[n]}
\sup_{(t_1,t_2)\in\mathcal T}
\left\|
\nabla^2 f(\bx,\bbeta_{2,i})
\right\|_{\mathrm{op}}
\le C_2\|\bbeta^*\|^{-2},
\]
where $\bbeta_{2,i}$ is defined at the beginning of this subsection.
\end{assumption}

\begin{assumption}[Regularity conditions on $\ell$]
\label{as5}
There exist constants $\tilde C_1,\tilde C_2=O(1)$ and $\mu>0$
such that all the following bounds hold simultaneously with
probability at least $1-q_n$:
\begin{align*}
\sup_{i\in[n]}
\sup_{(t_1,t_2)\in\mathcal T}
\left|\ld(y_i,f(\bx_i,\bbeta_{2,i}))\right|
&\le \tilde C_1\,\mathrm{poly}(\log n),\\
\sup_{i\in[n]}
\sup_{(t_1,t_2)\in\mathcal T}
\left|\ldd(y_i,f(\bx_i,\bbeta_{2,i}))\right|
&\le \tilde C_1\,\mathrm{poly}(\log n),\\
\inf_{i\in[n]}
\inf_{(t_1,t_2)\in\mathcal T}
\ldd(y_i,f(\bx_i,\bbeta_{2,i}))
&\ge \mu.
\end{align*}

For each $i\in[n]$, define the nonlinear loss-curvature vector
\[
\bm{\ldd}_{/i}(\bbeta)
:=
\bigl(\ldd(y_j,f(\bx_j,\bbeta))\bigr)_{j\in[n]\setminus\{i\}}
\in\mathbb R^{n-1}.
\]
The following Lipschitz-like conditions hold:
\begin{align*}
&\sup_{i\in[n]}\sup_{t\in[0,1]}
\frac{
\left\|
\bm{\ldd}_{/i}((1-t)\brbli+t\brb)
-\bm{\ldd}_{/i}(\brb)
\right\|_2
}{
\|\brb-\brbli\|_2
}
\le
\tilde C_2\frac{\sqrt n}{\|\bbeta^*\|}
\sqrt{\mathrm{poly}(\log n)},\\
&\sup_{i\in[n]}\sup_{t\in[0,1]}
\frac{
\left\|
\bm{\ldd}_{/i}((1-t)\brbli+t\brb)
-\bm{\ldd}_{/i}(\bbeta^*)
\right\|_2
}{
\|\brb-\bbeta^*\|_2
}
\le
\tilde C_2\frac{\sqrt n}{\|\bbeta^*\|}
\sqrt{\mathrm{poly}(\log n)},\\
&\sup_{i\in[n]}\sup_{t\in[0,1]}
\frac{
\left\|
\bm{\ldd}_{/i}((1-t)\bli+t\bl)
-\bm{\ldd}_{/i}(\bbeta^*)
\right\|_2
}{
\|\bl-\bbeta^*\|_2
}
\le
\tilde C_2\frac{\sqrt n}{\|\bbeta^*\|}
\sqrt{\mathrm{poly}(\log n)},\\
&\sup_{i\in[n]}
\frac{
\left\|
\bm{\ldd}_{/i}(\bl)-\bm{\ldd}_{/i}(\bli)
\right\|_2
}{
\|\bl-\bli\|_2
}
\le
\tilde C_2\frac{\sqrt n}{\|\bbeta^*\|}
\sqrt{\mathrm{poly}(\log n)}.
\end{align*}

For the frozen linearized model, define
\[
\bm d_{/i}^{\mathrm{lin}}(\bbeta)
:=
\bigl(
\ldd(y_j,\bg_j^\top\bbeta+\bb_j)
\bigr)_{j\in[n]\setminus\{i\}}
\in\mathbb R^{n-1}.
\]
The corresponding curvature satisfies
\[
\inf_{i\in[n]}
\inf_{t\in[0,1]}
\min_{j\in[n]\setminus\{i\}}
\ldd\left(
y_j,
\bg_j^\top\bigl((1-t)\brbli+t\brb\bigr)+\bb_j
\right)
\ge\mu,
\]
and, simultaneously for every $i\in[n]$,
\[
\sup_{t\in[0,1]}
\left\|
\bm d_{/i}^{\mathrm{lin}}((1-t)\brbli+t\brb)
-\bm d_{/i}^{\mathrm{lin}}(\brb)
\right\|_2
\le
\tilde C_2\frac{\sqrt n}{\|\bbeta^*\|}
\sqrt{\mathrm{poly}(\log n)}
\|\brbli-\brb\|_2.
\]
\end{assumption}

\begin{assumption}[Unique minimizer]
\label{ass:unique_minimizer}
For $\bg_j=\nabla f(\bx_j,\bl)$, the function
\(
\sum_{j=1}^n
\ell\left(
y_j,\bg_j^\top\bbeta+\bb_j
\right)
\)
has a unique minimizer.
\end{assumption}

The following assumption is only required for establishing the lower bounds in Proposition~\ref{theo:I_true_corr} and Theorem~\ref{theo:step1_corr}.

\begin{assumption}[Uniform lower bounds]
\label{as9}
There exists a constant $D_1>0$ such that
\[
\inf_{i\in[n]}
\|\nabla f(\bx_i;\bbeta^*)\|_2^2
\ge D_1,
\]
and
\[
\inf_{i\in[n]}
\left|
\ld(y_i,f(\bx_i,\bbeta^*))
\right|
\ge D_1.
\]
\end{assumption}

\subsection{Proofs of theorems}
\label{sec:proofs}
\subsubsection{Proof of Proposition \ref{theo:I_true_corr}} \label{proof:upper:inf:corr}

By using the mean value theorem we have
\begin{equation}\label{eq:exactinf:linearized}
\mathcal{I}^{\mathrm{True}}(\bz_i, \bz_{\mathrm{new}})
\triangleq
f(\bx_{\mathrm{new}}; \bli)
-
f(\bx_{\mathrm{new}}; \bl)= \nabla^\top f(\bx_{\mathrm{new}}, \tilde{\bbeta})(\bli- \bl),
\end{equation}
where $\tilde{\bbeta}$ is a point on the line segment that connects $\bli$ and $\bl$. Using the first order optimality conditions for $\bl$ and $\bli$ we have 
\begin{eqnarray}
 0 &=& \frac{1}{n} \sum_{j\neq i} \nabla l (y_j,f(\bx_j, \bli)) \nonumber \\
  &=&\frac{1}{n} \sum_{j\neq i} \nabla l (y_j,f(\bx_j, \bl))+\frac{1}{n}\nabla l(y_i,f(\bx_i, \bl)).
\end{eqnarray}
Hence, by using the integral form of the mean value theorem we have
\begin{eqnarray}\label{eq:mvt:int:bli:bl}
(\bli-\bl) = \Big( \int_{t=0}^1 \frac{1}{n} \sum_{j\neq i} \nabla^2 l (y_j,f(\bx_j, t \bli +(1-t) \bl) ) dt \Big)^{-1} \left( \frac{1}{n} \nabla l (y_i, f(\bx_i, \bl)) \right).
\end{eqnarray}
By combing \eqref{eq:exactinf:linearized} and \eqref{eq:mvt:int:bli:bl} we obtain
\begin{equation}\label{eq:influence1:order}
\mathcal{I}^{\mathrm{True}}(\bz_i, \bz_{\mathrm{new}}) = \nabla^\top f(\bx_{\mathrm{new}}, \tilde{\bbeta}) \Big( \int_{t=0}^1 \frac{1}{n} \sum_{j\neq i} \nabla^2 l (y_j,f(\bx_j, t \bli +(1-t) \bl) ) dt \Big)^{-1} \left( \frac{1}{n} \nabla l (y_i, f(\bx_i, \bl)) \right).
\end{equation}
Therefore we can use the Cauchy-Schwartz inequality and Assumption \ref{as3} to prove
\begin{eqnarray}
|\mathcal{I}^{\mathrm{True}}(\bz_i, \bz_{\mathrm{new}})| \leq\frac{\|\bbeta^*\|^2}{c_1n} \|\nabla f(\bx_{\mathrm{new}}, \tilde{\bbeta})\| \|\nabla l (y_i, f(\bx_i, \bl)) \| = O_{p}\Big(\frac{\mathrm{poly}(\log n) \|\bbeta^*\|^2}{n}\Big),
\end{eqnarray}
where to obtaint the last equality we have used Assumptions \ref{as2} and \ref{as5}.

To prove the second part of the theorem, i.e. \eqref{eq:inf:exact:x_i}, again we use \eqref{eq:influence1:order}. Note that
\begin{eqnarray}\label{eq:influence:simp1}
\lefteqn{\mathcal{I}^{\mathrm{True}}(\bz_i, \bz_i) = \nabla^\top f(\bx_{\mathrm{new}}, \tilde{\bbeta}) \Big( \int_{t=0}^1 \frac{1}{n} \sum_{j\neq i} \nabla^2 l (y_j,f(\bx_j, t \bli +(1-t) \bl) ) dt \Big)^{-1} \left( \frac{1}{n} \nabla l (y_i, f(\bx_i, \bl)) \right)} \nonumber \\ &=& \nabla^\top f(\bx_{\mathrm{new}}, \bbeta^*) \Big( \int_{t=0}^1 \frac{1}{n} \sum_{j\neq i} \nabla^2 l (y_j,f(\bx_j, t \bli +(1-t) \bl) ) dt \Big)^{-1} \frac{1}{n} \nabla l (y_i, f(\bx_i, \bl))  \nonumber \\
&+ & \Big(\nabla^\top f(\bx_{\mathrm{new}}, \bl)- \nabla^\top f(\bx_{\mathrm{new}}, \bbeta^*) \Big)\Big( \int_{t=0}^1 \frac{1}{n} \sum_{j\neq i} \nabla^2 l (y_j,f(\bx_j, t \bli +(1-t) \bl) ) dt \Big)^{-1} \frac{1}{n} \nabla l (y_i, f(\bx_i, \bl))  \nonumber \\
&=&  \nabla^\top f(\bx_{\mathrm{new}}, \bbeta^*) \Big( \int_{t=0}^1 \frac{1}{n} \sum_{j\neq i} \nabla^2 l (y_j,f(\bx_j, t \bli +(1-t) \bl) ) dt \Big)^{-1} \frac{1}{n} \nabla l (y_i, f(\bx_i, \bl))  + O_{p}\left(\frac{\|\bbeta^*\|^2}{n^{1.5- \epsilon}}\right),
\end{eqnarray}
where to obtain the last inequality we have used the Cauchy-Schwartz inequality and Lemma \ref{lem:bound_beta}. With a similar reasoning we can further simplify \eqref{eq:influence:simp1} in the following way:
\begin{eqnarray}\label{eq:influence:simp2}
\mathcal{I}^{\mathrm{True}}(\bz_i, \bz_i) &=&  \nabla^\top f(\bx_{\mathrm{new}}, \bbeta^*) \Big( \int_{t=0}^1 \frac{1}{n} \sum_{j\neq i} \nabla^2 l (y_j,f(\bx_j, t \bli +(1-t) \bl) ) dt \Big)^{-1} \frac{1}{n} \ld(y_i, f(\bx_i, \bbeta^*) \nabla f(\bx_i, \bbeta^*)  \nonumber \\
&+& O_{p}\left(\frac{\|\bbeta^*\|^2}{n^{1.5- \epsilon}}\right).
\end{eqnarray}
Furthermore, using Assumption \ref{as3} we have 
\begin{eqnarray}\label{eq:bounderror:beta:betastar}
|\mathcal{I}^{\mathrm{True}}(\bz_i, \bz_i) | \geq \frac{\|\bbeta^*\|^2}{\tilde{c}_2n} |\nabla^\top f(\bx_i, \bbeta^*)\nabla f(\bx_i, \bbeta^*)||\ld (y_i, f(\bx_i, \bbeta^*))| -  O_{p}\left(\frac{\|\bbeta^*\|^2}{n^{1.5- \epsilon}}\right). 
\end{eqnarray}
Assumption~\ref{as9} then implies \eqref{eq:inf:exact:x_i}.

\subsubsection{Proof of Proposition~\ref{theo:I_true}}
\label{sec:prop-proof}
First note that
\begin{eqnarray*}  
    &&|f(\bx_\new, \bli) - f(\bx_\new, \bl)|
    \leq \sup_{t \in [0,1]} \|\bg_{\new}^\top(\bli - \bl)+(\bli - \bl)^\top\nabla^2 f(\bx_{\text{new}}, (1-t)\bli + t\bl)(\bli - \bl)\| \\
    &\leq& \sup_{t \in [0,1]} \|\nabla^2 f(\bx_{\text{new}}, (1-t)\bli + t\bl)\|\|\bli - \bl\|^2 +\|\bg_{\new,*}^\top(\bli - \bl)\| 
    + \|(\bg_{\new}-\bg_{\new,*})^\top(\bli - \bl)\| \\
    &\leq& C_2 \|\bbeta^*\|^{-2} \|\bli - \bl\|^2 +\|\bg_{\new,*}^\top(\bli - \bl)\| + \|\bg_{\new}-\bg_{\new,*}\|\|\bli - \bl\|.
\end{eqnarray*}
The first inequality follows from the mean value theorem. The second inequality uses the triangle inequality. The third inequality follows from Assumption~\ref{as4} and the Cauchy–Schwarz inequality. Thus, it suffices to bound the three terms in the last expression.

\noindent\textbf{1. Bounding $\|\bli-\bl\|$.}
Using the mean value theorem and the first-order optimality conditions in \eqref{eq:bl} and \eqref{eq:optimization_lo}, we have:
\begin{eqnarray*}
    0 &=& \frac{1}{n} \sum_{j=1}^n \nabla l (y_j,f(\bx_j, \bli)) -\frac{1}{n}\nabla l(y_i,f(\bx_i, \bli)) \\
    &=& \frac{1}{n} \sum_{j=1}^n \nabla l (y_j,f(\bx_j, \bl)) +\frac{1}{n}\sum_{j=1}^n \nabla^2 l (y_j,f(\bx_j, \tbli)) (\bli - \bl) -\frac{1}{n}\nabla l(y_i,f(\bx_i, \bli)) \\
    &=& \frac{1}{n}\sum_{j=1}^n \nabla^2 l (y_j,f(\bx_j, \tbli)) (\bli - \bl) -\frac{1}{n}\nabla l(y_i,f(\bx_i, \bli)),
\end{eqnarray*}
where $\tbli = (1-t_i)\bli + t_i\bl$ for some $t_i \in [0,1]$. This implies:
\begin{eqnarray}
    \sup_{i \in [n]}\|\bli - \bl\| &=& \sup_{i \in [n]}\Big\| \Big(\frac{1}{n}\sum_{j=1}^n \nabla^2 l (y_j,f(\bx_j, \tbli)) \Big)^{-1} \frac{1}{n}\nabla l(y_i,f(\bx_i, \bli))\Big\| \nonumber\\
    &\leq& \sup_{i \in [n]} \sigma_{\max } \Big(\frac{1}{n}\sum_{j=1}^n \nabla^2 l (y_j,f(\bx_j, \tbli)) \Big)^{-1} \Big) \sup_{i \in [n]} \Big\| \frac{1}{n}\nabla l(y_i,f(\bx_i, \bli))\Big\|  \nonumber \\
    &\leq& \frac{\tilde C_1\|\bbeta^*\|^2}{c_1} \frac{\poly(\log n)}{n}, \qquad \wp \geq 1-q_n,
    \label{eq:bli-bl}
\end{eqnarray}
where to obtain the last inequality we use Assumption~\ref{as2} and \ref{as5}.

\noindent\textbf{2. Bounding $\|\bg_{\new} - \bg_{\new, *}\|$.}
From Lemma~\ref{lem:bound_beta}, for any $\eps>0$, with probability at least $1-q_n - n^{-\eps}$, there exists an absolute constant $s_g^\ast$ such that 
\begin{equation}
   \|\bg_{\new} - \bg_{\new, *}\| \leq  s_g^\ast(\log n)^{0.5}n^{\eps-0.5}.
   \label{eq:bg-bgast}
\end{equation}

\noindent\textbf{3. Bounding $\|\bg_{\new,*}^\top(\bli - \bl)\|$.}
By Assumption~\ref{ass:nsg_grad}, $\bg_{\new,*}$ is sub-Gaussian. Using \eqref{eq:bli-bl} and the independence of $\bg_{\new,*}$ from $\bl$, $\bli$, $\brb$, and $\brbli$, there exists an absolute constant $s_g$ such that $\bg_{\new,*}^\top(\bli - \bl)$ is a $s_g\frac{\|\bbeta^*\|\poly(\log n)}{n}$-sub-Gaussian random variable. Hence, for any $\eps>0$, with probability at least $1 - 2n^{-\eps}$, there exists an absolute constant $\hat s_g$ such that
\begin{equation}
  \|\bg_{\new,*}^\top(\bli - \bl)\| \leq \hat s_g \frac{\|\bbeta^*\|}{n^{1-\eps}}.
  \label{eq:bgast_bli-bi}
\end{equation}

\noindent\textbf{Final bound.}
Combining the bounds above, there exists a constant $s_1$ such that, for any $\eps>0$,
\[
|f(\bx_\new, \bli) - f(\bx_\new, \bl)| \leq s_1 \Big( \frac{\|\bbeta^*\|^2}{n^2}{\poly(\log n)} + \frac{\|\bbeta^*\|^2}{n^{1.5-\eps}}+ \frac{\|\bbeta^*\|}{n^{1-\eps}}\Big),
\]
with probability at least $1 - 3(q_n + n^{-\eps})$. The final result follows from $\|\bbeta^*\|^2\ll n^{1-\eps}$.

\subsubsection{Proof of Theorem \ref{theo:step1_corr}} \label{proof:step1_corr}

By the first-order optimality conditions, we have:
\begin{eqnarray*}
    \sum_{j=1}^n \ld(y_j,\bg_j^\top \hat{\bm{\beta}} + \bb_j)\bg_j &=& 0,\\
    \sum_{j=1}^n \ld(y_j,\bg_j^\top \brb + \bb_j)\bg_j &=& 0.
\end{eqnarray*}
Under Assumption~\ref{ass:unique_minimizer}, $\sum_{j=1}^n \ell(y_j,\bg_j^\top {\bm{\beta}} + \bb_j)$ has a unique minimizer. Therefore, we conclude that $\bl = \brb$. Next, we bound the difference:
\begin{eqnarray*}  
   \lefteqn{ |f(\bx_i, \bli) - (\bg_i^\top \brbli+ \bb_i)| \leq |f(\bx_i, \bli) - f(\bx_i, \bl)| + |\bg_i^\top (\bl - \brb)| + |\bg_i^\top (\brb - \brbli)|} \nonumber  \\
    &\leq& \sup_{t \in [0,1]} \|\bg_i^\top(\bli - \bl)+(\bli - \bl)^\top  \Big( \int_{t=0}^1 \nabla^2 f(\bx_i, (1-t)\bli + t\bl)dt \Big) (\bli - \bl)\| \\
    &&+ |\bg_i^\top (\bl - \brb)| + |\bg_i^\top (\brb - \brbli)| \\
    &\leq& \sup_{t \in [0,1]} \|\nabla^2 f(\bx_{i}, (1-t)\bli + t\bl)\|\|\bli - \bl\|^2 \\
    &&+\|\bg_{i,*}^\top(\bli - \bl)\| + \|\bg_{i,*}^\top(\brb - \brbli)\| \\
    &&+ \|(\bg_{i}-\bg_{i,*})^\top(\bli - \bl)\| + \|(\bg_{i}-\bg_{i,*})^\top(\brb - \brbli)\| \\
    &\leq& C_2 \|\bbeta^*\|^{-2} \|\bli - \bl\|^2 \\
    &&+\|\bg_{i,*}^\top(\bli - \bl)\| + \|\bg_{i,*}^\top(\brb - \brbli)\| \\
    &&+ \|\bg_{i}-\bg_{i,*}\|\|\bli - \bl\| + \|\bg_{i}-\bg_{i,*}\|\|\brb - \brbli\|.
\end{eqnarray*}
The first inequality follows from the mean value theorem and the fact that $\brb = \bl$. The second inequality uses the triangle and Cauchy-Schwarz inequalities. The third inequality follows from Assumption~\ref{as4} and Cauchy-Schwarz. Thus, it suffices to bound the three terms in the last three lines.
Using the same method as the one used in deriving \eqref{eq:bli-bl} we have
\begin{eqnarray*}
    \sup_{i \in [n]}\|\brbli - \brb\| &=&O_{p}\Big(   \frac{\|\bbeta^*\|^2\poly(\log n)}{n}\Big).
       \end{eqnarray*}
   Furthermore, using the mean value theorem we have
   \begin{eqnarray}
   &&\bg_{i,*}^\top(\brbli - \brb) = \ld_i(\brb) \bg_{i,*}^\top 
\Bigl(\int_{t=0}^1 \GI^\top \diag[\bm{d}^{\rm lin}_{/i}(\brbli - t\bm{\breve{\Delta}}_{/i})] \GI dt \Bigr)^{-1}  \bg_{i} \nonumber \\
 &=& \ld_i(\brb) \bg_{i,*}^\top 
\Bigl(\int_{t=0}^1 \GI^\top \diag[\bm{d}^{\rm lin}_{/i}(\brbli - t\bm{\breve{\Delta}}_{/i})] \GI dt \Bigr)^{-1}  \bg_{i, *} +O_{p}\left(\frac{\|\bbeta^*\|^2}{n^{1.5- \epsilon}}\right),
   \end{eqnarray}
   where to obtain the last equality we used an argument similar to the one we used for deriving \eqref{eq:bounderror:beta:betastar}. Hence, we have
  \begin{eqnarray}\label{eq:upper_influence}
   |\bg_{i,*}^\top(\brbli - \brb)|  &\leq & \frac{\|\bbeta^*\|^2}{C_2 n} 
 |\ld_i(\brb)| \bg_{i,*}^\top  \bg_{i, *} + O_{p}\left(\frac{\|\bbeta^*\|^2}{n^{1.5- \epsilon}}\right) = O_{p} \left(\frac{\poly(\log n) \|\bbeta^*\|^2}{n}\right). 
   \end{eqnarray}
Similarly, we have
\begin{equation*}
|\bg_{i,*}^\top(\brb - \brbli)|=  O_{p} \left(\frac{\poly(\log n) \|\bbeta^*\|^2}{n}\right).
\end{equation*}
   Now we aim to bound the term $\|\bg_{i}-\bg_{i,*}\|\|\bli - \bl\|$. According to  \eqref{eq:bg-bgast} we have with high probability,
\begin{equation}
   \|\bg_i - \bg_{i, *}\| \leq  s_g^\ast(\log n)^{0.5}n^{\eps-0.5}.
\end{equation}
Therefore, 
\begin{equation}
\|\bg_{i}-\bg_{i,*}\|\|\bli - \bl\| = O_{p}\left(\frac{\|\bbeta^*\|^2}{n^{1.5- \epsilon}}\right).
\end{equation}
Combining all the calculation above we can concluse that 
\[
|f(\bx_i, \bli) - (\bg_i^\top \brbli+ \bb_i)| = O_{p} \left(\frac{\|\bbeta^*\|^2 \poly(\log n)}{n}\right). 
\]
Now we turn our attention to proving \eqref{eq:lowerbound:inf:correlated}. First note that using the same argument as the one presented above we have 
\begin{eqnarray}\label{eq:influencesize}
\lefteqn{\mathcal{I}^{\rm Linear}(\bz_i,\bz_i)
=
\nabla f(\bx_i,\bl)^\top(\brbli-\brb) = \nabla f(\bx_i,\bbeta^*)^\top(\brbli-\brb) + O_{p} \left(\frac{\|\bbeta^*\|^2 \poly(\log n)}{n^{1.5 -\epsilon}}\right)} \nonumber \\
&=& \ld_i(\brb) \bg_{i,*}^\top 
\Bigl(\int_{t=0}^1 \GI^\top \diag[\bm{d}^{\rm lin}_{/i}(\brbli - t\bm{\breve{\Delta}}_{/i})] \GI dt \Bigr)^{-1}  \bg_{i} + O_{p} \left(\frac{\|\bbeta^*\|^2 \poly(\log n)}{n^{1.5 -\epsilon}}\right)\nonumber \\
 &=& \ld_i(\brb) \bg_{i,*}^\top 
\Bigl(\int_{t=0}^1 \GI^\top \diag[\bm{d}^{\rm lin}_{/i}(\brbli - t\bm{\breve{\Delta}}_{/i})] \GI dt \Bigr)^{-1}  \bg_{i, *} +O_{p}\left(\frac{\|\bbeta^*\|^2}{n^{1.5- \epsilon}}\right),
\end{eqnarray}
Using \eqref{eq:upperbound:inv} we can conclude that with high probability we have that  
\begin{eqnarray}
|\ld_i(\brb) \bg_{i,*}^\top 
\Bigl(\int_{t=0}^1 \GI^\top \diag[\bm{d}^{\rm lin}_{/i}(\brbli - t\bm{\breve{\Delta}}_{/i})] \GI dt \Bigr)^{-1}  \bg_{i, *}| \nonumber\\ \geq |\ld_i(\brb)| \bg_{i,*}^\top 
  \bg_{i, *} (\frac{2\tilde C_1 C_2 n\poly(\log n)}{\|\bbeta^*\|^2})^{-1}. 
\end{eqnarray}
Combing this results with Assumption \ref{as9} establishes \eqref{eq:lowerbound:inf:correlated}.

\subsubsection{Proof of Theorem~\ref{theo:step1}}
\label{sec:thm1-proof}
By the first-order optimality conditions for $\bl$ and $\brb$, we have:
\begin{eqnarray*}
    \sum_{j=1}^n \ld(y_j,\bg_j^\top \hat{\bm{\beta}} + \bb_j)\bg_j &=& 0,\\
    \sum_{j=1}^n \ld(y_j,\bg_j^\top \brb + \bb_j)\bg_j &=& 0.
\end{eqnarray*}
Under Assumption~\ref{ass:unique_minimizer}, $\sum_{j=1}^n \ell(y_j,\bg_j^\top {\bm{\beta}} + \bb_j)$ has a unique minimizer; therefore, we conclude that $\bl = \brb$.
Next, we bound the difference:
\begin{eqnarray*}  
  \lefteqn{  |f(\bx_\new, \bli) - (\bg_\new^\top \brbli+ \bb_\new)| \leq |f(\bx_\new, \bli) - f(\bx_\new, \bl)| + |\bg_\new^\top (\bl - \brb)| + |\bg_\new^\top (\brb - \brbli)|} \\
    &\leq& \sup_{t \in [0,1]} \|\bg_{\new}^\top(\bli - \bl)+(\bli - \bl)^\top\nabla^2 f(\bx_{\text{new}}, (1-t)\bli + t\bl)(\bli - \bl)\| \\
    &&+ |\bg_\new^\top (\bl - \brb)| + |\bg_\new^\top (\brb - \brbli)| \\
    &\leq& \sup_{t \in [0,1]} \|\nabla^2 f(\bx_{\text{new}}, (1-t)\bli + t\bl)\|\|\bli - \bl\|^2 \\
    &&+\|\bg_{\new,*}^\top(\bli - \bl)\| + \|\bg_{\new,*}^\top(\brb - \brbli)\| \\
    &&+ \|(\bg_{\new}-\bg_{\new,*})^\top(\bli - \bl)\| + \|(\bg_{\new}-\bg_{\new,*})^\top(\brb - \brbli)\| \\
    &\leq& C_2 \|\bbeta^*\|^{-2} \|\bli - \bl\|^2 \\
    &&+\|\bg_{\new,*}^\top(\bli - \bl)\| + \|\bg_{\new,*}^\top(\brb - \brbli)\| \\
    &&+ \|\bg_{\new}-\bg_{\new,*}\|\|\bli - \bl\| + \|\bg_{\new}-\bg_{\new,*}\|\|\brb - \brbli\|.
\end{eqnarray*}
The first inequality follows from the mean value theorem and the fact that $\brb = \bl$. The second inequality uses the triangle and Cauchy-Schwarz inequalities. The third inequality follows from Assumption~\ref{as4} and Cauchy-Schwarz. Thus, it suffices to bound the three terms in the last three lines.
Using the same method as \eqref{eq:bli-bl} and \eqref{eq:bgast_bli-bi},
we can get the bounds
\begin{eqnarray}
    \sup_{i \in [n]}\|\brbli - \brb\| &=&O(   \frac{\|\bbeta^*\|^2\poly(\log n)}{n}), \wp \geq 1-q_n,
    \label{eq:brbli-brb}\\
    \|\bg_{\new,*}^\top(\brbli - \brb)\| &=& O( \frac{\|\bbeta^*\|}{n^{1-\eps}}), \wp \geq 1-2n^{-\eps}.
  \label{eq:bgast_brbli-brb}
\end{eqnarray}
Combining 
\eqref{eq:bg-bgast}, we conclude that there exists a constant $s_1$ such that, for any $\eps>0$,
\[
|f(\bx_\new, \bli) - (\bg_\new^\top \brbli+ \bb_\new)| \leq s_1 \Big(  \frac{\|\bbeta^*\|^2}{n^2}{\poly(\log n)} + \frac{\|\bbeta^*\|^2}{n^{1.5-\eps}}+\frac{\|\bbeta^*\|}{n^{1-\eps}}\Big),
\]
with probability at least $1 - 3(q_n + n^{-\eps})$. The final result follows from $\|\bbeta^*\|^2\ll n^{1-\eps}$. 

%%%%%%%%%%%%

\subsubsection{Proof of Theorem~\ref{thm:g_base_ALO_ii}}
\label{sec:thm2-proof-ii}

\begin{proof}
Throughout the proof, $P_n$ denotes a deterministic fixed
polynomial in $1+\log n$, which may be enlarged from line to line.

\paragraph{Step 1: Conditioning and stability.}
For each $i\in[n]$, write
\[
\GI:=\bG_{/i}(\bl),
\qquad
\bm D_i
:=
\operatorname{diag}
\bigl(\ldd(y_j,f(\bx_j,\bl))\bigr)_{j\ne i},
\qquad
A_i:=\GI^\top\bm D_i\GI.
\]
Assumptions~\ref{as2} and~\ref{as5} give
\[
\max_{i\in[n]}
\left\{
\|\bg_i\|_2,\|\bg_{i,*}\|_2,|\ld_i(\brb)|
\right\}
\le P_n,
\qquad
\mu I_{n-1}\preceq\bm D_i\preceq B_n I_{n-1},
\quad B_n\le P_n.
\]
Here
$\ld_i(\brb)=\ld(y_i,f(\bx_i,\bl))$
because $\brb=\bl$.

Since deleting row $i$ subtracts $\bg_i\bg_i^\top$ from
$\bG(\bl)^\top\bG(\bl)$, Assumption~\ref{as3} implies
\begin{equation}
\frac{c_1n}{2\|\bbeta^*\|^2}
\le
\lambda_{\min}(\GI^\top\GI)
\le
\lambda_{\max}(\GI^\top\GI)
\le
\frac{C_2n}{\|\bbeta^*\|^2}
\label{eq:good_bound_GI}
\end{equation}
for all sufficiently large $n$.
Indeed, after enlarging $P_n$, the removed row has squared norm
at most $P_n$, whereas
$n/\|\bbeta^*\|^2\ge n^\eps$.
Consequently,
\begin{equation}
\frac{\|\bbeta^*\|^2}{C_2B_n n}I_d
\preceq A_i^{-1}
\preceq
\frac{2\|\bbeta^*\|^2}{\mu c_1n}I_d.
\label{eq:upperbound:inv}
\end{equation}
Moreover, since $A_i\succeq\mu\GI^\top\GI$,
\begin{align*}
\|\GI A_i^{-1}\|_{\mathrm{op}}^2
&=
\|A_i^{-1}\GI^\top\GI A_i^{-1}\|_{\mathrm{op}}\le
\frac1\mu\|A_i^{-1}\|_{\mathrm{op}}
\le C\frac{\|\bbeta^*\|^2}{n}.
\end{align*}
Thus,
\[
\|\GI A_i^{-1}\|_{\mathrm{op}}
\le C\frac{\|\bbeta^*\|}{\sqrt n}.
\]

\paragraph{Step 2: A genuinely leave-one-out reference.}
Choose $\bli$ as a measurable function of
$\mathcal D_{/i}=\{\bz_j:j\ne i\}$ and define
\[
\GI^0:=\bG_{/i}(\bli),
\qquad
\bm D_i^0
:=
\operatorname{diag}
\bigl(\ldd(y_j,f(\bx_j,\bli))\bigr)_{j\ne i},
\qquad
A_i^0:=(\GI^0)^\top\bm D_i^0\GI^0.
\]
These quantities are functions of $\mathcal D_{/i}$ alone.
Therefore,
$\bg_{i,*}$ is independent of $(\GI^0,\bm D_i^0,A_i^0)$.

The uniform training-point Hessian bound in
Assumption~\ref{as4}, together with \eqref{eq:bli-bl}, gives
\begin{align*}
\rho_i
:=
\|\GI-\GI^0\|_{\mathrm{op}}
&\le \|\GI-\GI^0\|_F\le
C_2\sqrt n\,\|\bbeta^*\|^{-2}\|\bl-\bli\|_2\le \frac{P_n}{\sqrt n}.
\end{align*}
The nonlinear endpoint Lipschitz-like condition in
Assumption~\ref{as5} gives
\begin{align*}
\gamma_i
:=
\|\bm D_i-\bm D_i^0\|_F
&\le
P_n\frac{\sqrt n}{\|\bbeta^*\|}\|\bl-\bli\|_2\le
P_n\frac{\|\bbeta^*\|}{\sqrt n}
=o(1).
\end{align*}
Hence, for sufficiently large $n$,
\[
\frac\mu2 I_{n-1}
\preceq\bm D_i^0
\preceq(B_n+1)I_{n-1}.
\]

Singular-value perturbation and \eqref{eq:good_bound_GI} imply
\[
\|\GI^0\|_{\mathrm{op}}
\le C\frac{\sqrt n}{\|\bbeta^*\|},
\qquad
\lambda_{\min}\bigl((\GI^0)^\top\GI^0\bigr)
\ge c\frac n{\|\bbeta^*\|^2}.
\]
It follows that
\[
\max\left\{
\|A_i^{-1}\|_{\mathrm{op}},
\|(A_i^0)^{-1}\|_{\mathrm{op}}
\right\}
\le C\frac{\|\bbeta^*\|^2}{n}.
\]
Therefore, Lemma~\ref{lem:weighted_gram_infty_transfer}
applies uniformly in $i$, with
\[
a=C\frac{\|\bbeta^*\|^2}{n},
\qquad
\mu_0=\frac\mu2,
\qquad
B\le P_n,
\qquad
\rho_i\le\frac{P_n}{\sqrt n},
\qquad
\gamma_i\le P_n\frac{\|\bbeta^*\|}{\sqrt n}.
\]

\paragraph{Step 3: The infinity-norm bound.}
Define
\[
\bu_i^0=\GI^0(A_i^0)^{-1}\bg_{i,*},
\qquad
\bu_{i,*}=\GI A_i^{-1}\bg_{i,*}.
\]
Conditional on $\mathcal D_{/i}$, the matrix
$\GI^0(A_i^0)^{-1}$ is fixed and $\bg_{i,*}$ retains its
unconditional distribution. Then we get,
\[
\|\GI^0(A_i^0)^{-1}\|_{\mathrm{op}}
\le C\frac{\|\bbeta^*\|}{\sqrt n}.
\]
and 
\begin{equation}
\mathbb P\left(
\sup_{i\in[n]}\|\bu_i^0\|_\infty
\le\frac{P_n}{\sqrt n}
\right)
\longrightarrow1.
\label{eq:alo_reference_infty_target}
\end{equation}

On the event in \eqref{eq:alo_reference_infty_target},
Lemma~\ref{lem:weighted_gram_infty_transfer} yields
\begin{align*}
\|\bu_{i,*}\|_\infty
&\le
\left(1+\frac{\gamma_i}{\mu_0}\right)
\|\bu_i^0\|_\infty
+
\left(1+\frac{2B}{\mu_0}\right)
a\rho_i\|\bg_{i,*}\|_2\le
\frac{P_n}{\sqrt n}
+
P_n\frac{\|\bbeta^*\|^2}{n^{3/2}}\le
\frac{P_n}{\sqrt n},
\end{align*}
after enlarging $P_n$.

We next replace $\bg_{i,*}$ by $\bg_i$.
Apply Lemma~\ref{lem:bound_beta} with $\eta=\eps/4$.
Using the uniform training-point Hessian bound in
Assumption~\ref{as4}, we obtain
\[
\sup_{i\in[n]}\|\bg_i-\bg_{i,*}\|_2
\le
C_2\|\bbeta^*\|^{-2}\|\bl-\bbeta^*\|_2
\le
P_n n^{-1/2+\eta}
\]
with probability tending to one.
Consequently,
\begin{align*}
\|\GI A_i^{-1}(\bg_i-\bg_{i,*})\|_\infty
&\le
\|\GI A_i^{-1}\|_{\mathrm{op}}
\|\bg_i-\bg_{i,*}\|_2\le
P_n\|\bbeta^*\|n^{-1+\eta}\le
P_n n^{-1/2-\eps/4}.
\end{align*}
Combining these estimates gives
\begin{equation}
\sup_{i\in[n]}
\|\GI A_i^{-1}\bg_i\|_\infty
\le\frac{P_n}{\sqrt n}.
\label{eq:alo_infty_full_fit}
\end{equation}

\paragraph{Step 4: Exact representation of the ALO error.}
Set
\[
\bm{\breve\Delta}_{/i}=\brbli-\brb,
\qquad
s_i=\ld_i(\brb),
\qquad
\bm{\hat\Delta}_{/i}=s_iA_i^{-1}\bg_i.
\]
Using the frozen-gradient curvature vector from
Assumption~\ref{as5}, define
\[
\bm D_i^{\mathrm{lin}}(t)
:=
\operatorname{diag}\left(
\bm d_{/i}^{\mathrm{lin}}
(\brb+t\bm{\breve\Delta}_{/i})
\right),
\]
\[
\overline{\bm D}_i
:=
\int_0^1\bm D_i^{\mathrm{lin}}(t)\,dt,
\qquad
\overline A_i
:=
\GI^\top\overline{\bm D}_i\GI.
\]
Since $\brb=\bl$ and
$\bg_j^\top\bl+\bb_j=f(\bx_j,\bl)$,
we have $\bm D_i^{\mathrm{lin}}(0)=\bm D_i$.

The first-order conditions for the full and leave-one-out
linearized objectives, together with the integral Taylor formula,
give the exact identity
\[
\overline A_i\bm{\breve\Delta}_{/i}=s_i\bg_i.
\]
By the frozen-gradient curvature lower bound in
Assumption~\ref{as5},
\[
\overline{\bm D}_i\succeq\mu I_{n-1}.
\]
Therefore, \eqref{eq:good_bound_GI} and the same weighted-Gram
argument used in Step~1 give
\[
\|\overline A_i^{-1}\|_{\mathrm{op}}
\le C\frac{\|\bbeta^*\|^2}{n},
\qquad
\|\overline A_i^{-1}\GI^\top\|_{\mathrm{op}}
\le C\frac{\|\bbeta^*\|}{\sqrt n}.
\]
In particular,
\[
\sup_{i\in[n]}\|\bm{\breve\Delta}_{/i}\|_2
\le P_n\frac{\|\bbeta^*\|^2}{n}.
\]

Define
\[
\bm\Gamma_i:=\overline{\bm D}_i-\bm D_i,
\qquad
\Delta_i:=\GI^\top\bm\Gamma_i\GI,
\]
so that $\overline A_i=A_i+\Delta_i$.
The frozen-gradient Lipschitz-like condition in
Assumption~\ref{as5} implies
\begin{align}
\|\bm\Gamma_i\|_F
&\le
\int_0^1
\|\bm D_i^{\mathrm{lin}}(t)-\bm D_i\|_F\,dt
\le
P_n\frac{\sqrt n}{\|\bbeta^*\|}
\|\bm{\breve\Delta}_{/i}\|_2
\le
P_n\frac{\|\bbeta^*\|}{\sqrt n}.
\label{eq:gamma_term}
\end{align}
Here the segment in Assumption~\ref{as5} is the same segment,
with the parameterization reversed.

The resolvent identity now gives
\begin{align*}
\bm{\breve\Delta}_{/i}-\bm{\hat\Delta}_{/i}
&=
s_i(\overline A_i^{-1}-A_i^{-1})\bg_i=
-s_i\overline A_i^{-1}\GI^\top
\bm\Gamma_i\GI A_i^{-1}\bg_i.
\end{align*}
Because $\bm\Gamma_i$ is diagonal,
\[
\|\bm\Gamma_i\bv\|_2
\le\|\bm\Gamma_i\|_F\|\bv\|_\infty
\qquad\text{for every }\bv\in\mathbb R^{n-1}.
\]
Consequently,
\begin{align*}
\|\bm{\breve\Delta}_{/i}-\bm{\hat\Delta}_{/i}\|_2
&\le
|s_i|\,
\|\overline A_i^{-1}\GI^\top\|_{\mathrm{op}}
\|\bm\Gamma_i\|_F
\|\GI A_i^{-1}\bg_i\|_\infty\le
P_n
\frac{\|\bbeta^*\|}{\sqrt n}
\frac{\|\bbeta^*\|}{\sqrt n}
\frac1{\sqrt n}=
P_n\frac{\|\bbeta^*\|^2}{n^{3/2}}.
\end{align*}
Thus, simultaneously over all $i$,
\begin{equation}
\sup_{i\in[n]}
\|\bm{\breve\Delta}_{/i}-\bm{\hat\Delta}_{/i}\|_2
\le
P_n\frac{\|\bbeta^*\|^2}{n^{3/2}}.
\label{eq:brbdelta-hatdelta}
\end{equation}

\paragraph{Step 5: From parameter error to influence error.}
By the definition of linearized influence,
\[
\mathcal I^{\mathrm{Linear}}(\bz_i,\bz_i)
=
\bg_i^\top\bm{\breve\Delta}_{/i}.
\]
Write
\[
H
=
\bG^\top
\operatorname{diag}\bigl(\ldd_j(\brb)\bigr)_{j=1}^n
\bG.
\]
Since
\[
A_i=H-\ldd_i(\brb)\bg_i\bg_i^\top,
\]
the Sherman--Morrison identity yields
\[
\mathcal I^{\mathrm{ALO}}(\bz_i,\bz_i)
=
\ld_i(\brb)\bg_i^\top A_i^{-1}\bg_i
=
\bg_i^\top\bm{\hat\Delta}_{/i}.
\]
Therefore, Assumption~\ref{as2} and
\eqref{eq:brbdelta-hatdelta} imply
\begin{align*}
&\sup_{i\in[n]}
\left|
\mathcal I^{\mathrm{Linear}}(\bz_i,\bz_i)
-
\mathcal I^{\mathrm{ALO}}(\bz_i,\bz_i)
\right|\le
\left(\sup_{i\in[n]}\|\bg_i\|_2\right)
\left(
\sup_{i\in[n]}
\|\bm{\breve\Delta}_{/i}-\bm{\hat\Delta}_{/i}\|_2
\right)\le
P_n\frac{\|\bbeta^*\|^2}{n^{3/2}}.
\end{align*}
Every fixed polynomial in $1+\log n$ is bounded by $n^\eps$
for all sufficiently large $n$.
Hence, the asserted
bound follows with probability tending to one.
\end{proof}

\subsubsection{Proof of Theorem~\ref{thm:g_base_ALO}}
\label{sec:thm2-proof}

First, define $\bm{\breve{\Delta}}_{/i} := \brbli - \brb$ and the Jacobian compactly as
\begin{equation*}
\bm{J}_{/i}(\bm{\theta}) = \GI^\top \diag(\bm d_{/i}^{\mathrm{lin}}(\bm{\theta})) \GI.
\end{equation*}
Also define the estimated and exact leave-$i$-out perturbation vectors as
\begin{equation*}
\begin{aligned}
    \bm{\hat{\Delta}}_{/i} & \triangleq \ld_i(\brb) \bigl[ \bm{J}_{/i} (\brbli - \bm{\breve{\Delta}}_{/i}) \bigr]^{-1} \bg_{i}. 
\end{aligned}
\end{equation*}

Applying the Woodbury matrix identity yields
\begin{align*}
\bg_{\mathrm{new}}^\top \bm{\hat{\Delta}}_{/i} 
&= \ld_i(\brb)\, \bg_{\mathrm{new}}^\top
\Bigl(\GI^\top \diag[\bm{\ldd}_{/i}(\brbli - \bm{\breve{\Delta}}_{/i})] \GI\Bigr)^{-1} \bg_{i} \\
&= \frac{
            \ld_i(\brb)\,
            \bg_{\mathrm{new}}^\top 
            \bigl(\bG^\top \diag[\bm\ldd(\brb)] \bG\bigr)^{-1}
            \bg_{i}
        }{
            1 - 
            \ldd_i(\brb)\,
            \bg_{i}^\top 
            \bigl(\bG^\top \diag[\bm\ldd(\brb)] \bG\bigr)^{-1}
            \bg_{i}
          }.
\end{align*}

From \eqref{eq:brbli-brb}, we know that there exists an absolute constant $s_\Delta^\ast$ such that
\begin{equation}
\left\|\bm{\breve{\Delta}}_{/i}\right\|_2 
\leq  s_\Delta^\ast\frac{\|\bbeta^*\|^2 \poly(\log n)}{n}, 
\quad \wp \geq 1-q_n. 
\label{eq:delta_i_star}
\end{equation}

By the first-order optimality conditions and the integral form
of Taylor's theorem for the gradient of the linearized
leave-one-out objective,
\begin{equation*}
\bm{\breve{\Delta}}_{/i}
=
\ld_i(\brb)
\left[
\int_0^1
\bm{J}_{/i}\bigl(\brbli-(1-t)\bm{\breve{\Delta}}_{/i}\bigr)
\,dt
\right]^{-1}
\bg_i.
\end{equation*}

Introduce
\begin{equation*}
\bm{\gamma}_{\bm{\delta}/i}(\bm{\theta}) 
\triangleq \bm d_{/i}^{\mathrm{lin}}(\bm{\theta} + \bm{\delta}) - \bm d_{/i}^{\mathrm{lin}}(\bm{\theta}),
\end{equation*}
so that
\begin{equation*}
\bm{J}_{/i}(\bm{\theta} + \bm{\delta}) 
= \bm{J}_{/i}(\bm{\theta}) 
+ \GI^\top \diag\!\left[\bm{\gamma}_{\bm{\delta}/i}(\bm{\theta})\right]\GI. 
\end{equation*}
Note that $\bm{J}_{/i}(\bm{\theta} + \bm{\delta})$ remains positive definite for all $t \in [0,1]$, with $\bm{\theta} = \brbli$ and $\bm{\delta} = -(1-t)\bm{\breve{\Delta}}_{/i}$.

Using the notation above, we now bound the error. With high probability, we obtain
\begin{eqnarray}
    &&\left|\bg_{\mathrm{new}}^\top \bm{\breve{\Delta}}_{/i} - \bg_{\mathrm{new}}^\top \bm{\hat{\Delta}}_{/i}\right| \nonumber\\
&\leq&   \left|\bg_{\mathrm{new}*}^\top (\bm{\breve{\Delta}}_{/i} -\bm{\hat{\Delta}}_{/i})\right| +\|\bg_{\new,*}-\bg_{\new}\|\|\bm{\breve{\Delta}}_{/i} -\bm{\hat{\Delta}}_{/i}\| \nonumber\\
&\leq& (\|\bbeta^*\|^{-1} + n^{\eps-0.5})\| \bm{\breve{\Delta}}_{/i} -\bm{\hat{\Delta}}_{/i}\| \poly(\log n),
\label{eq:conclu}
\end{eqnarray}
where we use that $\bg_{\new,*}$ is independent of $\bm{\breve{\Delta}}_{/i} -\bm{\hat{\Delta}}_{/i}$, along with Assumption~\ref{ass:nsg_grad} and Lemma~\ref{lem:bound_beta}.

Next, we consider $\| \bm{\breve{\Delta}}_{/i} -\bm{\hat{\Delta}}_{/i}\|$. As we have get in \eqref{eq:brbdelta-hatdelta},
\[
\| \bm{\breve{\Delta}}_{/i} -\bm{\hat{\Delta}}_{/i}\|
\preceq  \frac{\|\bbeta^*\|^2}{n^{3/2-\eps}}.
\]

Returning to \eqref{eq:conclu}, we conclude that with high probability,
\[
\left|\bg_{\mathrm{new}}^\top \bm{\breve{\Delta}}_{/i} - \bg_{\mathrm{new}}^\top \bm{\hat{\Delta}}_{/i}\right|
= O\!\left(\frac{\|\bbeta^*\|}{n^{3/2-\eps}}\right).
\]
This completes the proof.

\subsubsection{Proof of~\Cref{theo:magnitude_projection}}\label{proof:projection}

By~\Cref{eq::inf-alo}, we have
\[
\mathcal{I}^{\mathrm{ALO}}(\bz_i,\bz_{\new}) = \frac{
        \ld_i(\brb)\, \bg_{\mathrm{new}}^\top
        (\bG^\top \diag[\bm\ldd(\brb)]\bG)^{-1} \bg_{i}
    }{
        1 - \ldd_i(\brb)\, \bg_{i}^\top
        (\bG^\top \diag[\bm\ldd(\brb)]\bG)^{-1} \bg_{i}
    }.
\]
Using the Woodbury identity, we obtain
\[
\mathcal{I}^{\mathrm{ALO}}(\bz_i,\bz_{\new}) =
\ld_i(\brb)\, \bg_{\mathrm{new}}^\top
(\GI^\top \diag[\bm\ldd_{/i}(\brb)]\GI)^{-1} \bg_{i}.
\]
Similarly, from~\Cref{trak:alo}, we have
\[
\mathcal{I}^{\mathrm{TRAK}}(\bz_i,\bz_{\new}) =
\ld_i(\brb)\, \bphi_{\mathrm{new}}^\top
(\Phi_{/i}^\top \diag[\bm\ldd_{/i}(\brb)]\Phi_{/i})^{-1} \bphi_{i}.
\]
We ignore the $k$ in $\mathcal{I}^{\mathrm{TRAK}}(\bz_i,\bz_{\new};k)$ for simplicity.
Recall that for the projection matrix \(S \in \RR^{k \times d}\), TRAK defines the feature map \(\bphi(\bx) = S \bg(\bx)\), where \(\bg(\bx) = \nabla f(\bx, \bl)\). Accordingly, we define \(\Phi = (\bphi_1,\ldots,\bphi_n)^\top \in \RR^{n\times k}\) with \(\bphi_i = S \bg_i\) and \(\bphi_{\new} = S \bg_{\new}\). The entries \(S_{ij}\) for \(i \in [k], j \in [d]\) are i.i.d. \(\mathcal{N}(0,1)\).

By \ref{eq:good_bound_GI} and Assumption~\ref{as5}, with probability at least \(1-q_n-2n^{-\eps}\), %for sufficiently large \(n\) and absolute constants \(C_{\max}\) and \(c_{\min}\),
\begin{align}
\frac{n}{\|\bbeta^*\|^2} \preceq \sigma_{\min}\bigl(\GI^\top \diag[\bm\ldd_{/i}(\brb)]\GI\bigr)
\preceq \sigma_{\max}\bigl(\GI^\top \diag[\bm\ldd_{/i}(\brb)]\GI\bigr)
\preceq \frac{n{\poly(\log n)}}{\|\bbeta^*\|^2}.
\label{eq::condno}
\end{align}
We denote \(\kappa\) as the condition number of $\GI^\top \diag[\bm\ldd_{/i}(\brb)]\GI$ and  \(\kappa = O_p({\poly(\log n)})\) by Assumption~\ref{as5}.
Then we split $\mathcal{I}^{\mathrm{ALO}}(\bz_i,\bz_{\new})$ and $\mathcal{I}^{\mathrm{TRAK}}(\bz_i,\bz_{\new})$ as
\[
\mathcal{I}^{\mathrm{ALO}}(\bz_i,\bz_{\new}) =
\ld_i(\brb)\, \bg_{\mathrm{new},*}^\top
(\GI^\top \diag[\bm\ldd_{/i}(\brb)]\GI)^{-1} \bg_{i}+\ld_i(\brb)\, (\bg_{\mathrm{new}}-\bg_{\mathrm{new},*})^\top
(\GI^\top \diag[\bm\ldd_{/i}(\brb)]\GI)^{-1} \bg_{i}.
\]
\[
\mathcal{I}^{\mathrm{TRAK}}(\bz_i,\bz_{\new}) =
\ld_i(\brb)\, \bphi_{\mathrm{new},*}^\top
(\Phi_{/i}^\top \diag[\bm\ldd_{/i}(\brb)]\Phi_{/i})^{-1} \bphi_{i}+\ld_i(\brb)\, (\bphi_{\mathrm{new}}-\bphi_{\mathrm{new},*})^\top
(\Phi_{/i}^\top \diag[\bm\ldd_{/i}(\brb)]\Phi_{/i})^{-1} \bphi_{i}.
\]
We denote the first terms in $\mathcal{I}^{\mathrm{ALO}}(\bz_i,\bz_{\new})$ and $\mathcal{I}^{\mathrm{TRAK}}(\bz_i,\bz_{\new})$ as $\mathcal{I}_*^{\mathrm{ALO}}(\bz_i,\bz_{\new})$ and $\mathcal{I}_*^{\mathrm{TRAK}}(\bz_i,\bz_{\new})$, and the last terms as $\mathcal{D}_*^{\mathrm{ALO}}(\bz_i,\bz_{\new})$ and $\mathcal{D}_*^{\mathrm{TRAK}}(\bz_i,\bz_{\new})$.

First, we consider the case when $\bz_{i} = \bz_{\new}$.
By our \Cref{thm:g_base_ALO} and \Cref{lem:bound_beta}, we get that
$$\mathcal{I}_*^{\mathrm{ALO}}(\bz_i,\bz_i) \geq O_p\Bigl(\frac{\|\bbeta^*\|^2}{n}\Bigr), \quad \mathcal{D}_*^{\mathrm{ALO}}(\bz_i,\bz_i) \leq O_p\Bigl(\frac{\|\bbeta^*\|^3}{n^{1.5-\eps}}\Bigr).$$
Then, by the assumption $\|\bbeta^*\| \leq n^{0.5-\eps}$, we obtain that with high probability, for sufficiently large $n$ and $p$,
$$\frac{\mathcal{D}_*^{\mathrm{ALO}}(\bz_i,\bz_i)}{\mathcal{I}_*^{\mathrm{ALO}}(\bz_i,\bz_i)} \leq \frac{1}{2}.$$
Thus,
\begin{eqnarray*}
     &&\frac{\mathcal{I}^{\mathrm{TRAK}}(\bz_i,\bz_i)}{\mathcal{I}^{\mathrm{ALO}}(\bz_i,\bz_i)} = \frac{\mathcal{I}_*^{\mathrm{TRAK}}(\bz_i,\bz_i)+\mathcal{D}_*^{\mathrm{TRAK}}(\bz_i,\bz_i)}{\mathcal{I}_*^{\mathrm{ALO}}(\bz_i,\bz_i)+\mathcal{D}_*^{\mathrm{ALO}}(\bz_i,\bz_i)} \\
    &\geq& \frac{\mathcal{I}_*^{\mathrm{TRAK}}(\bz_i,\bz_i)}{\mathcal{I}_*^{\mathrm{ALO}}(\bz_i,\bz_i)+\mathcal{D}_*^{\mathrm{ALO}}(\bz_i,\bz_i)} \geq \frac{\mathcal{I}_*^{\mathrm{TRAK}}(\bz_i,\bz_i)}{\mathcal{I}_*^{\mathrm{ALO}}(\bz_i,\bz_i)} \frac{1}{1+\frac{\mathcal{D}_*^{\mathrm{ALO}}(\bz_i,\bz_i)}{\mathcal{I}_*^{\mathrm{ALO}}(\bz_i,\bz_i)}} .
\end{eqnarray*}

We now focus on the ratio $\frac{\mathcal{I}_*^{\mathrm{TRAK}}(\bz_i,\bz_{i}) }{\mathcal{I}_*^{\mathrm{ALO}}(\bz_i,\bz_{i}) }$.
From~\Cref{eq::norm-pres}, we get that with probability greater than \(1 - C\exp(-c\epsilon^2k)\),
\begin{equation}
  \frac{1}{\kappa} (1-\eps)\frac{k}{d}\leq \frac{\mathcal{I}_*^{\mathrm{TRAK}}(\bz_i,\bz_{i}) }{\mathcal{I}_*^{\mathrm{ALO}}(\bz_i,\bz_{i}) }\leq  \min\Bigl(\kappa (1+\eps)\frac{k}{d},1\Bigr).
  \label{eq:control_ab}
\end{equation}
Hence, there exists a constant $c^{\rm TRAK}>0$ such that with high probability,
$$\frac{\mathcal{I}^{\mathrm{TRAK}}(\bz_i,\bz_i)}{\mathcal{I}^{\mathrm{ALO}}(\bz_i,\bz_i)}  \geq c^{\rm TRAK} \frac{k}{d}.$$

Next, we consider the case when $\bz_i$ and $\bz_{\new}$ are independent. Define
\[
\bu_i \coloneqq (\GI^\top \diag[\bm\ldd_{/i}(\brb)]\GI)^{-1}\bg_i
\quad\text{and}\quad
\bu_{i*} \coloneqq (\GI^\top \diag[\bm\ldd_{/i}(\brb)]\GI)^{-1}\bg_{i,*}.
\]
\[
\bv_i \coloneqq S^\top (S\GI^\top \diag[\bm\ldd_{/i}(\brb)]\GI S^\top)^{-1}S\bg_i
\quad\text{and}\quad
\bv_{i*} \coloneqq  S^\top(S\GI^\top \diag[\bm\ldd_{/i}(\brb)]\GI S^\top)^{-1} S\bg_{i,*}.
\]
Note that our goal is to calculate an upper bound for $\bg_{\rm new}^{\top}\bv_i$. Using Assumption \ref{ass:nsg_grad} and the independence of $\bg_{\new,*}$ and $\bv_{i}$ we have $\bg_{\new,*}^\top \bv_{i}$ is \(\frac{\|\bv_{i*}\|\text{poly}(\log(n))}{\|\bbeta^*\|}\)-subGaussian. Thus,
\begin{equation}
    |\bg_{\new,*}^\top \bv_{i}| = O_{p}\Bigl(\frac{\|\bv_{i*}\|\text{poly}(\log(n))}{\|\bbeta^*\|}\Bigr).
    \label{eq:1_corr}
\end{equation}
Now, by~\Cref{lem:bound_beta}, for any \(\epsilon > 0\),
\(
\bg_i = \bg_{i,*} + (\bg_i - \bg_{i,*}) = \bg_{i,*} + O_{\text{P}}\Bigl(\sqrt{\frac{\log n}{n^{1-2\epsilon}}}\Bigr).
\) Consequently,
\begin{equation}
    |(\bg_{\new}-\bg_{\new,*})^\top \bv_{i}| =O_{p}\Bigl({\frac{\poly(\log n)}{n^{0.5-\eps}}} \|\bv_{i*}\|\Bigr)=O_{p}\Bigl(\frac{\|\bbeta^*\|}{n^{0.5-\eps}} \frac{\|\bv_{i*}\|\text{poly}(\log(n))}{\|\bbeta^*\|}\Bigr).
    \label{eq:2_corr}
\end{equation}
Hence we have
\begin{eqnarray}\label{eq:imp:up:linear}
|\bg_{\rm new}^{\top} \bv_i| = |\bg_{\rm new,*}^{\top} \bv_i|+ O_{p}\Bigl(\frac{\|\bbeta^*\|}{n^{0.5-\eps}} \frac{\|\bv_{i*}\|\text{poly}(\log(n))}{\|\bbeta^*\|}\Bigr)= O_{p}\Bigl(\frac{\|\bv_{i*}\|\text{poly}(\log(n))}{\|\bbeta^*\|}\Bigr),
\end{eqnarray}
where to get the last equality we have used \eqref{eq:1_corr}. To bound $\|\bv_{i*}\|^2$ note that
\begin{eqnarray*}
  &&\sigma_{\min}\bigl((S H_{i} S^\top)^{-0.5}SS^\top(S H_{i} S^\top)^{-0.5}\bigr)\bg_{i,*}^\top S^\top (S H_{i}S^\top)^{-1} S \bg_{i,*} \leq\|\bv_{i*}\|^2 \\
    &\leq& \sigma_{\max}\bigl((S H_{i} S^\top)^{-0.5}SS^\top(S H_{i} S^\top)^{-0.5}\bigr)\bg_{i,*}^\top S^\top (S H_{i}S^\top)^{-1} S \bg_{i,*}
\end{eqnarray*}
where we denote $H_{i}=\GI^\top \diag[\bm\ldd_{/i}(\brb)]\GI$. Observe that $ \sigma_{\max}((S H_{i} S^\top)^{-0.5}SS^\top(S H_{i} S^\top)^{-0.5}) \leq \sigma_{\min}(H_i)^{-1}$ and $\sigma_{\min}((S H_{i} S^\top)^{-0.5}SS^\top(S H_{i} S^\top)^{-0.5}) \geq \sigma_{\max}(H_i)^{-1}$. Hence,
\begin{eqnarray*}
 &&\frac{\bg_{i,*}^\top S^\top(SH_iS^\top)^{-1}S\bg_{i,*}}
       {\sigma_{\max}(H_i)}
\leq \|\bv_{i*}\|^2
\leq
\frac{\bg_{i,*}^\top S^\top(SH_iS^\top)^{-1}S\bg_{i,*}}
     {\sigma_{\min}(H_i)}.
\end{eqnarray*}
Further we can show that
\begin{eqnarray}
  &&\sigma_{\min}(H_{i}^{-1})\bg_{i,*}^\top H_{i}^{-1}\bg_{i,*} \leq\|\bu_{i*}\|^2 \leq  \sigma_{\max}(H_{i}^{-1})\bg_{i,*}^\top H_{i}^{-1}\bg_{i,*}
\end{eqnarray}
 Combining these two equations with \Cref{eq::norm-pres}, we finally get with high probability,
$$\frac{1}{\kappa^2}(1-\epsilon)\frac{k}{d} \leq \frac{\|\bv_{i*}\|^2}{\|\bu_{i*}\|^2} \leq \min\Bigl(\kappa(1+\epsilon)\frac{k}{d},1\Bigr) \kappa. $$
Combining this bound with \eqref{eq:imp:up:linear} leads to
\[|\bg_{\new}^\top \bv_{i}| = O_{p}\Bigl(\frac{\|\bu_{i*}\|\text{poly}(\log(n))}{\|\bbeta^*\|} \sqrt{\frac{k}{d} }\Bigr).\]
Since
$$\|\bu_{i*}\| = O_{p}\Bigl(\frac{\|\bbeta^*\|^2}{n^{1-\eps}}\Bigr),$$
and with \(\kappa = O_{p}({\poly(\log n)})\), we obtain
\[|\bg_{\new}^\top \bv_{i}| = O_{p}\Bigl( {\ \frac{\|\bbeta^*\| \sqrt{k}}{n^{1-\eps}\sqrt{d}}\text{poly}(\log(n))}\Bigr).\]

\subsection{Technical lemmas}
\label{sec:lemma}
\begin{lemma} \label{lem:neumann}
Assume that $\bm{G}^\top (\bm{D + \Gamma}) \bm{G}$ and $\bm{G}^\top \bm{D G}$ are positive definite, and both $\bm{D}$ and $\bm{\Gamma}$ are diagonal matrices. Let $\bg$ be a vector independent of $\bG$, $\bm{\Gamma}$, and $\bm{D}$. If $\|(\bG^\top\bm{D}\bG)^{-1}(\bG^\top\bm{\Gamma}\bG)\| < 1$, then there exists an absolute constant $M_0$ such that
\begin{equation*}
\left\| \left( (\bm{G}^\top (\bm{D}+\bm{\Gamma}) \bm{G})^{-1} - (\bm{G}^\top\bm{D G})^{-1} \right)\bg \right\| 
\leq M_0 \left\| (\bm{G}^\top\bm{D G})^{-1} \bG^\top \bm\Gamma \bG (\bm{G}^\top\bm{D G})^{-1} \bg \right\|.
\end{equation*}
\end{lemma}

\begin{proof}[Proof of Lemma~\ref{lem:neumann}]
Let $A = \bm{G}^\top \bm{D} \bm{G}$ and $\Delta = \bm{G}^\top \bm{\Gamma} \bm{G}$. By the Neumann series expansion, since $\|A^{-1}\Delta\| < 1$,
\[
(A + \Delta)^{-1} = A^{-1} - A^{-1}\Delta A^{-1} + A^{-1}\Delta A^{-1}\Delta A^{-1} - \cdots.
\]
Thus,
\[
(A + \Delta)^{-1} - A^{-1} = -A^{-1}\Delta A^{-1} + A^{-1}\Delta A^{-1}\Delta A^{-1} - \cdots.
\]
Multiplying by $\bg$ and taking norms,
\begin{align*}
\|((A + \Delta)^{-1} - A^{-1})\bg\|
&\leq \|A^{-1}\Delta A^{-1}\bg\| + \|A^{-1}\Delta A^{-1}\Delta A^{-1}\bg\| + \cdots \\
&\leq \|A^{-1}\Delta A^{-1}\bg\| \left(1 + \|A^{-1}\Delta\| + \|A^{-1}\Delta\|^2 + \cdots \right) \\
&= \|A^{-1}\Delta A^{-1}\bg\| \cdot \frac{1}{1 - \|A^{-1}\Delta\|}.
\end{align*}
Since $\|A^{-1}\Delta\| < 1$ by assumption, the geometric series converges. Taking $M_0 = \frac{1}{1 - \|A^{-1}\Delta\|}$ completes the proof.
\end{proof}
\begin{lemma}
\label{lem:bound_beta}
Suppose Assumptions~\ref{ass:nsg_grad} hold. Then for all $\eps>0$, with probability at least $1-q_n-n^{-\eps}$, there exists a constant $C_{\bbeta}, C_\bg > 0$, independent of $n$, $p$ and $d$, and
\[
\|\bl_n-\bm\beta^*\| \leq C_{\bbeta} \|\bbeta^*\|^2\sqrt{\frac{\log (n)}{n^{1-2\eps}}}\ ,
\]
and 
\[
\|\bg_i - \bg_{i,*}\| \leq C_\bg \sqrt{\frac{\log (n)}{ n^{1-2\eps}}}\ ,
\]
\end{lemma}

\begin{proof}[Proof of Lemma~\ref{lem:bound_beta}]
Denote $\ld_{i*} = \ld\big(y_i, f(\bx_i, \bbeta^*)\big)$. By Assumption~\ref{as5}, we know that $\max_{i\in [n] }|\ld_{i*}| \leq \tilde C_1 \log n$. Also, using Lemma 2(1) in \citet{jin2019short}, we know that
$\mathbb E \|\bg_{i,*}\|^2 \leq 2\sigma^2$, then using $\mathbb E [\ld_{i*} \bg_{i,*}] = 0$ and $\{\ld_{i*} \bg_{i,*}\}_{i\in [n]}$ are independent, we get that
$$\mathbb E\|\frac{1}{n} \sum_{i=1}^n \ld_{i*} \bg_{i,*}\|^2 = \frac{1}{n^2}\mathbb E\sum_{i=1}^n \ld_{i*}^2\bg_{i,*}^\top \bg_{i,*}  \leq \frac{2\tilde C_1 \sigma^2\log n}{n}.$$
Then using moment inequality, we get that
$$\mathbb P(\|\frac{1}{n} \sum_{i=1}^n \ld_{i*} \bg_{i,*}\| \geq t) \leq \frac{2 \tilde C_1 \sigma^2 \log n}{n t^2}.$$
Thus we can say that for all $\epsilon>0$, $t = \frac{\sqrt{\log n}}{n^{\frac{1}{2}-\eps}}$, when $n$ big enough, then with probability $\geq 1-n^{-\eps}$, there exists a constant $C_g$ such that
\begin{equation}
\label{eq:lg_star_bound}
    \|\frac{1}{n}\sum_{i=1}^n \ld_{i*} \bg_{i,*}\| \leq C_g \frac{{\log^{\frac{1}{2}} (n)}}{n^{\frac{1}{2}-\eps}}.
\end{equation}
Using Mean-Value Theorem, denote $\ld_{i} = \ld\big(y_i, f(\bx_i, \bl)\big)$, we know that there exists $\bar\bbeta = t\bbeta^*+(1-t)\bl$ where $t\in [0,1]$, such that
$$0=\frac{1}{n}\sum_{i=1}^n \ld_{i} \bg_{i} = \frac{1}{n}\sum_{i=1}^n \ld_{i*} \bg_{i,*} +\frac{1}{n}\sum_{i=1}^n \nabla^2\ell(y_i, f(x_i,\bar\bbeta))(\hat\bbeta - \bbeta^*),$$
then 
$$\|\hat\bbeta - \bbeta^*\|\leq \|(\frac{1}{n}\sum_{i=1}^n \nabla^2\ell(y_i, f(x_i,\bar\bbeta)))^{-1}\|\|\frac{1}{n}\sum_{i=1}^n \ld_{i*} \bg_{i,*}\|,$$
by Assumption~\ref{as3} , we know that$\wp \ge 1-q_n$,
$$\sigma_{\min}(\frac{1}{n}\sum_{i=1}^n \nabla^2\ell(y_i, f(x_i,\bar\bbeta)))\geq  c_1\|\bbeta^*\|^{-2},$$
combining \eqref{eq:lg_star_bound}, we know that
$$\|\hat\bbeta - \bbeta^*\|\leq {\frac{C_g}{c_1}} \frac{{\log^{\frac{1}{2}} (n) }\|\bbeta^*\|^2}{n^{\frac{1}{2}-\eps}}:=C_{\bbeta} \frac{{\log^{\frac{1}{2}} (n) }\|\bbeta^*\|^2}{n^{\frac{1}{2}-\eps}}$$
Finally, using Assumption~\ref{as4}, note that there also $\bar\bbeta' = t'\bbeta^*+(1-t')\bl$ where $t'\in [0,1]$, such that 
$$\|\bg_i-\bg_{i,*}\|\leq \|\nabla^2 f(\bx_i,\bar\bbeta)\|\|\bbeta^*-\bl\| \leq C_2 \|\bbeta^*\|^{-2}\|\bbeta^*-\bl\|\leq {C_2C_{\bbeta}} \frac{{\log^{\frac{1}{2}} (n) }}{n^{\frac{1}{2}-\eps}}:=C_\bg\frac{{\log^{\frac{1}{2}} (n) }}{n^{\frac{1}{2}-\eps}} .$$
\end{proof}

\begin{lemma}[Infinity-norm transfer for weighted Gram matrices]
\label{lem:weighted_gram_infty_transfer}
Let $\bG,\bG_0\in\mathbb R^{m\times d}$ have full column rank,
and let $\bm D,\bm D_0\in\mathbb R^{m\times m}$ be diagonal
matrices satisfying
\[
\mu I_m\preceq\bm D\preceq B I_m,
\qquad
\mu I_m\preceq\bm D_0\preceq B I_m,
\qquad 0<\mu\le B.
\]
Define
\[
A=\bG^\top\bm D\bG,
\qquad
A_0=\bG_0^\top\bm D_0\bG_0,
\]
and suppose that
\[
\max\{\|A^{-1}\|_{\mathrm{op}},
         \|A_0^{-1}\|_{\mathrm{op}}\}\le a.
\]
Set
\[
\rho=\|\bG-\bG_0\|_{\mathrm{op}},
\qquad
\gamma=\|\bm D-\bm D_0\|_F.
\]
Then, for every $\bg\in\mathbb R^d$,
\begin{align}
\|\bG A^{-1}\bg\|_\infty
&\le
\left(1+\frac{\gamma}{\mu}\right)
\|\bG_0 A_0^{-1}\bg\|_\infty
+
\left(1+\frac{2B}{\mu}\right)a\rho\|\bg\|_2.
\label{eq:weighted_gram_infty_transfer}
\end{align}
\end{lemma}

\begin{proof}
Write
\[
E=\bG-\bG_0,
\quad P=\bG A^{-1}\bG^\top,
\quad \bv_0=A_0^{-1}\bg,
\quad \bu_0=\bG_0\bv_0,
\quad \bu=\bG A^{-1}\bg.
\]
Since $A\succeq\mu\bG^\top\bG$, we have
\[
\|\bG A^{-1}\|_{\mathrm{op}}^2
=\|A^{-1}\bG^\top\bG A^{-1}\|_{\mathrm{op}}
\le\frac{\|A^{-1}\|_{\mathrm{op}}}{\mu}
\le\frac a\mu.
\]
The same argument gives
$\|\bG_0 A_0^{-1}\|_{\mathrm{op}}\le\sqrt{a/\mu}$.
Furthermore,
\[
P=\bm D^{-1/2}\Pi\bm D^{-1/2},
\qquad
\Pi=\bm D^{1/2}\bG A^{-1}\bG^\top\bm D^{1/2},
\]
where $\Pi$ is an orthogonal projector. Thus
$\|P\|_{\mathrm{op}}\le\mu^{-1}$.

The exact identity
\[
A-A_0
=\bG^\top(\bm D-\bm D_0)\bG_0
 +E^\top\bm D_0\bG_0
 +\bG^\top\bm D E
\]
together with
$A^{-1}-A_0^{-1}=-A^{-1}(A-A_0)A_0^{-1}$ yields
\begin{align*}
\bu
&=\bu_0-P(\bm D-\bm D_0)\bu_0
 +(I_m-P\bm D)E\bv_0\\
&\quad-\bG A^{-1}E^\top\bm D_0\bu_0.
\end{align*}
Because $\bm D-\bm D_0$ is diagonal,
\[
\|P(\bm D-\bm D_0)\bu_0\|_\infty
\le\frac1\mu\|(\bm D-\bm D_0)\bu_0\|_2
\le\frac\gamma\mu\|\bu_0\|_\infty.
\]
The remaining terms satisfy
\[
\|(I_m-P\bm D)E\bv_0\|_\infty
\le\left(1+\frac B\mu\right)\rho a\|\bg\|_2
\]
and
\begin{align*}
\|\bG A^{-1}E^\top\bm D_0\bu_0\|_\infty
&\le\sqrt{\frac a\mu}\,\rho B
       \sqrt{\frac a\mu}\,\|\bg\|_2=\frac B\mu\,a\rho\|\bg\|_2.
\end{align*}
Combining these inequalities proves
\eqref{eq:weighted_gram_infty_transfer}.
\end{proof}

\begin{lemma}
    \label{eq::norm-pres}
    Let \(S \in \mathbb{R}^{k \times d}\) be such that \(\forall i \in [k],j\in [d]\), \(S_{ij} \overset{i.i.d.}{\sim} N(0,1)\). Let \(A\) in \(\mathbb{R}^{d \times d}\) be a symmetric positive definite matrix with condition number \(\left ( \frac{\lambda_{\max}(A)}{\lambda_{\min}(A)}\right )\) upper bounded by \(\kappa > 0\). Here \(\lambda_{\max}(A)\) and \(\lambda_{\min}(A)\) are the extremal eigenvalues of \(A\). Let \(u\in \mathbb{R}^d\). Then with probability greater than \(1 - C\exp(-c\epsilon^2k)\), where \(C, c > 0\) are absolute constants,
    \begin{equation}
    \frac{1}{\kappa}(1-\epsilon)\frac{k}{d} \le \frac{u^TS^\top(SAS^\top)^{-1}Su}{u^TA^{-1}u} \le \min(\kappa(1+\epsilon)\frac{k}{d},1)
    \end{equation}
\end{lemma}

\begin{proof}[Proof of Lemma~\ref{eq::norm-pres}]
    
    Define, 
    \begin{align*}
        P &\coloneqq A^{1/2}S^\top(SAS^\top)^{-1}SA^{1/2}\\
        \hat{A}^{-1} &\coloneqq S^\top(SAS^\top)^{-1}S\\
        w &\coloneqq A^{-1/2}u
    \end{align*}
See that \(P\) is an orthogonal projection matrix, i.e. for any \(l \in \mathbb{R}^d\), 
\begin{align}
    \|l\|^2 = \|Pl\|^2 + \|(I-P)l\|^2 \label{eq::orthproj}
\end{align}
Also since \(P^2 = P\), \(\|Pw\|^2 = w^TPw\). See also now, 
\begin{align*}
    u^TS^\top(SAS^\top)^{-1}Su &= u^\top\hat{A}^{-1}u = w^TPw = \|Pw\|^2\\
    u^TA^{-1}u &= \|w\|^2
\end{align*}
Thus, in~\Cref{eq::norm-pres}, we are now trying to derive a concentration for \(\frac{\|Pw\|^2}{\|w\|^2}\).
Consider the eigenvalue decomposition of \(A\), \(A = Q\Lambda Q^\top\), \(\Lambda = \text{diag}(\lambda_1, \lambda_2, \dots, \lambda_d)\), where \(\lambda_1 >  \lambda_2 > \dots > \lambda_d > 0\).
Define \(V \coloneqq Q^TS^\top\). Since \(Q\) is orthogonal, we again have \(\forall i \in [d], j \in [k], V_{ij} \overset{i.i.d.}{\sim} N(0,1)\), and \(V^TV\) is almost surely positive definite. See, 
\begin{align*}
    P &= A^{1/2}S^\top(SAS^\top)^{-1}SA^{1/2}\\
    &= Q\Lambda^{1/2}V(V^\top\Lambda V)^{-1}V^\top\Lambda^{1/2}Q^\top = QP_VQ^\top,
\end{align*}
where \(P_V \coloneqq \Lambda^{1/2}V(V^\top\Lambda V)^{-1}V^\top\Lambda^{1/2}\). It is easy to see that \(P_V\) is an orthogonal projection matrix satisfying~\Cref{eq::orthproj} if we replace \(P\) by \(P_V\). Define,
\begin{align*}
    \Pi \coloneqq V(V^TV)^{-1}V^\top; \qquad h \coloneqq  \Lambda^{1/2}w
\end{align*}
We will first show that, almost surely, 
\begin{align}
    \frac{1}{\kappa}\frac{\|\Pi h\|^2}{\|h\|^2} \le \frac{\|P_Vw\|^2}{\|w\|^2} \le \kappa\frac{\|\Pi h\|^2}{\|h\|^2}
    \label{eq::proj-ratio-bd}
\end{align}
By definition of \(\Lambda\), we have, 
\begin{align*}
    &\lambda_d I \preceq \Lambda \preceq \lambda_1 I\\
    &\overset{(i)}{\implies} \lambda_d V^TV \preceq V^\top\Lambda V \preceq \lambda_1 V^TV\\
    &\implies \frac{1}{\lambda_1}(V^TV)^{-1} \preceq (V^\top\Lambda V)^{-1} \preceq \frac{1}{\lambda_d} (V^TV)^{-1},
\end{align*}
where \((i)\) is since for \(B, C, D \in \mathbb{R}^{d \times d}\) and any \(j \in \mathbb{R}^d\), \(B \preceq C \implies j^\top(C - B)j \ge 0 \implies j^TD^\top(C-B)Dj \ge 0\). Thus we have almost surely,
\begin{align}
    \frac{1}{\lambda_1}\|\Pi h\|^2 \le \|P_Vw\|^2 \le \frac{1}{\lambda_d}\|\Pi h\|^2
    \label{eq::proj-norm-bd}
\end{align}
Also it is clear that we have the following bounds for \(\|h\|^2 = w^\top\Lambda w\), \(\lambda_d\|w\|^2 \le \|h\|^2 \le \lambda_1\|w\|^2\). Combining with~\Cref{eq::proj-norm-bd}, we have shown~\Cref{eq::proj-ratio-bd}
\end{proof}
Using the fact that the quadratic forms of a Gaussian vector projected onto an independent random subspace follow a Beta distribution, Chapter 7~\cite{anderson1958introduction}, we have \(\frac{\|\Pi h\|^2}{\|h\|^2} \sim \text{Beta}\left (k/2, (d-k)/2 \right )\). By concentration of Beta distribution (using standard concentration inequalities from~\cite{vershynin2009high}), we have, 
\begin{align}
    P\left (\left \vert \frac{\|\Pi h\|^2}{\|h\|^2} - \frac{k}{d}\right \vert > \epsilon\right ) \le C\exp(-ck\epsilon^2)
    \label{eq::prob-beta}
\end{align}
By~\Cref{eq::proj-ratio-bd}, we have almost surely \(\frac{1}{\kappa}\frac{\|\Pi Q^Th\|^2}{\|Q^Th\|^2} \le \frac{\|P_VQ^Tw\|^2}{\|Q^Tw\|^2} \le \kappa\frac{\|\Pi Q^Th\|^2}{\|Q^Th\|^2}\). Thus since \(P = QP_VQ^\top\) and by~\Cref{eq::prob-beta}, we  have shown~\Cref{eq::norm-pres}.

\subsection{Discussion of more general correlated cases}\label{ssec:general_corr}

Throughout this paper, we evaluate the accuracy of TRAK in two distinct settings:
\begin{itemize}
\item \textbf{Correlated case:} The test data point is identical to one of the training samples.

\item \textbf{Independent case:} The test data point is independent of all training samples.
\end{itemize}

These represent two extreme scenarios that may arise in practice, which naturally raises the question of how TRAK behaves in intermediate regimes. The aim of this section is to address this question. For brevity, we present only the main results along with brief proof sketches. The formal presentation of the results and full proofs will be presented in a longer version of the paper.

Suppose we wish to evaluate the influence of $\bx_i$ on a test point $\bx_{\rm new}$, where both vectors are zero-mean and satisfy
\[
\mathbb{E}(\bx_i^\top \bx_{\rm new}) = \alpha_n \mathbb{E}(\bx_i^\top \bx_i).
\]

To simplify the analysis, assume that $\bx_i, \bx_{\rm new} \sim \mathcal{N}(0, \bSigma)$. Then $\bx_{\rm new}$ can be represented as
\begin{equation}
\label{eq:gen_corr_data}
   \bx_{\rm new} = \alpha_n \bx_i + \sqrt{1-\alpha_n^2}\,\bx_{\rm new}', 
\end{equation}
where $\bx_{\rm new}'$ is an independent copy of $\bx_i$.

We begin by analyzing the magnitude of the TRAK influence function $\mathcal{I}^{\mathrm{Linear}}(\bz_i, \bz_{\rm new})$. Define
\begin{align}
\bg_{i,*} &= \nabla f(\bx_i; \bbeta_*), \nonumber \\
\tilde{\bg}_{i,*} &= \nabla f\big(\alpha_n \bx_i + \sqrt{1-\alpha_n^2}\,\bx_{\rm new}'; \bbeta_*\big).
\end{align}
Following calculations similar to \eqref{eq:upper_influence} and \eqref{eq:influencesize}, we have
\[
\mathcal{I}^{\mathrm{Linear}}(\bz_i,\bz_{\rm new})
=
\Theta_p\left(
\tilde{\bg}_{i,*}^\top \bg_{i,*}
\left(\frac{n\,\mathrm{poly}(\log n)}{\|\bbeta_*\|^2}\right)^{-1}
\right).
\]

Thus, the analysis reduces to understanding the effect of $\alpha_n$ on the inner product $\tilde{\bg}_{i,*}^\top \bg_{i,*}$. Using Lemma~\ref{lem:general_correlation}, we obtain
\[
|\mathbb{E}\tilde{\bg}_{i,*}^\top \bg_{i,*} | \leq \alpha_n |\mathbb{E}{\bg}_{i,*}^\top \bg_{i,*} |
\]
Moreover, this upper bound is sharp for a broad class of functions $f(\cdot;\cdot)$. In particular, if the projection of $f$ onto the first Hermite polynomial has non-negligible energy, then
\[
|\mathbb{E}\tilde{\bg}_{i,*}^\top \bg_{i,*} | = \Omega_p ( \alpha_n |\mathbb{E}{\bg}_{i,*}^\top \bg_{i,*} |)
\]
Hence, we can conclude that under reasonably general conditions we have 
\begin{equation}
\label{eq:alpha_rela}
    \left|
\mathcal{I}^{\mathrm{Linear}}(\bz_i,\bz_{\rm new})
\right|
=
\Theta_p\left(
\alpha_n \,
\mathcal{I}^{\mathrm{Linear}}(\bz_i,\bz_i)
\right).
\end{equation}

Finally, the other results in this paper concerning $\mathcal{I}^{\mathrm{Linear}}(\bz_i,\bz_i)$ can be extended in a similar manner to the setting where the test data has correlation $\alpha_n$ with $\bx_i$. Overall, the magnitude of the influence scales by a factor of $\alpha_n$, and the corresponding error terms scale proportionally as well.

\begin{lemma}\label{lem:general_correlation}
Let $\bu,\bv \sim \mathcal{N}(0,\bSigma)$ be independent, and define
\[
\bm{w}_\alpha = \alpha \bm{u} + \sqrt{1-\alpha^2}\, \bm{v}, \qquad 0 \le \alpha \le 1.
\]
Let $h:\mathbb{R}^n \to \mathbb{R}$ satisfy $\mathbb{E}[h(\bm{u})^2] < \infty$ and $|\mathbb{E} h(\bm{u})|=0$. Then
\[
\mathbb{E}\big[h(\bm{w}_\alpha)h(\bu)\big]
\le
\alpha\, \mathbb{E}[h(\bu)^2].
\]
\end{lemma}

\begin{proof}
Write
\[
\bm{u} = \bSigma^{1/2} \tilde{\bm{u}}, \qquad v = \bSigma^{1/2} \tilde{\bm{v}},
\]
where $\tilde{\bm{u}},\tilde{\bm{v}} \sim \mathcal{N}(0,I_n)$ are independent. Define
\[
\phi(\tilde{\bm{u}}) = h(\bSigma^{1/2} \tilde{\bm{u}}).
\]
Then
\[
h(\bm{w}_\alpha)
=
\phi\big(\alpha \tilde{\bm{u}} + \sqrt{1-\alpha^2} \tilde{\bm{v}} \big).
\]

Expand $\phi$ in multivariate Hermite polynomials:
\[
\phi = \phi_0 + \phi_1 + \phi_2 + \cdots,
\]
where $\phi_m$ is the degree-$m$ Hermite component. Then
\[
\mathbb{E}\big[\phi(\tilde{\bm{u}})\phi(\alpha \tilde{\bm{u}} + \sqrt{1-\alpha^2} \tilde{\bm{v}})\big]
=
\sum_{m \ge 0} \alpha^m \|\phi_m\|_{L^2}^2.
\]

The degree-zero term is
\[
\|\phi_0\|_{L^2}^2 = (\mathbb{E} h(\bu))^2=0.
\]
Hence
\[
\mathbb{E}\big[h(\bm{w}_\alpha)h(\bm{u})\big]
=
\sum_{m \ge 1} \alpha^m \|\phi_m\|_{L^2}^2.
\]
The proof follows by noting that $0 \le \alpha \le 1$,
\[
\alpha^m \le \alpha \quad \text{for all } m \ge 1.
\]
Therefore,
\[
\mathbb{E}\big[h(\bm{w}_\alpha)h(\bm{u})\big]
\le
\alpha \sum_{m \ge 1} \|\phi_m\|_{L^2}^2.
\]

But
\[
\sum_{m \ge 1} \|\phi_m\|_{L^2}^2
=
\mathrm{Var}(h(\bm{u})).
\]
Thus,
\[
\mathbb{E}\big[h(\bm{w}_\alpha)h(\bm{u})\big]
\le
\alpha\, \mathrm{Var}(h(\bm{u})).
\]
\end{proof}

\subsection{Confirming our assumptions on the loss function}
\label{sec:corder_check}

In this section, we examine Assumption~\ref{as5} in the context of logistic regression, linear regression, and Poisson regression. Our goal is to justify that the orders are reasonable.

\subsubsection{Logistic Regression}
Since
\[
\dot{\ell}_i\!\left(\bm \beta\right) = -y_i + \frac{e^{\bg_i^\top \bm \beta}}{1 + e^{\bg_i^\top \bm \beta}}, 
\quad 
\ddot{\ell}_i\!\left(\bm \beta\right) = \frac{e^{\bg_i^\top \bm \beta}}{\left(1 + e^{\bg_i^\top \bm \beta}\right)^2}, 
\quad 
\dddot{\ell}_i\!\left(\bm \beta\right) = \frac{e^{\bg_i^\top \bm \beta}\left(1 - e^{\bg_i^\top \bm \beta}\right)}{\left(1 + e^{\bg_i^\top \bm \beta}\right)^3},
\]
using simple algebra, it is straightforward to show that for any $\bm \beta$, we have
\[
\|\dot{\bm \ell}\!\left(\bm \beta\right)\|_\infty \leq 1, 
\quad 
\|\ddot{\bm \ell}\!\left(\bm \beta\right)\|_\infty \leq \tfrac{1}{4}, 
\quad 
\|\dddot{\bm \ell}\!\left(\bm \beta\right)\|_\infty \leq \tfrac{1}{10}.
\]
Hence, the first  inequality of Assumption~\ref{as5} holds.
Note that we can assume $\bg_i^\top \bbeta = O_p(1)$, so the second inequality of Assumption~\ref{as5} also holds. 
oreover,
\[
\|\bm\ldd_{ / i}\!\left(\bm \beta + \delta\right) - \bm \ldd_{ / i}\!\left(\bm \beta\right)\|_2
\leq \|\ddot{\bm \ell}\!\left(\bm \beta + \delta\right) - \ddot{\bm \ell}\!\left(\bm \beta\right)\|_2
= \sqrt{\sum_i \left(\ddot{\ell}_i\!\left(\bm \beta + \delta\right) - \ddot{\ell}_i\!\left(\bm \beta\right)\right)^2}.
\]
Using the mean-value theorem, where $\epsilon_i \in [0,\delta_i]$, we have
\[
= \sqrt{\sum_i \dddot{\ell}_i\!\left(\bm \beta + \epsilon_i\right)^2 \left(\bg_i^\top \delta\right)^2}
\leq \sqrt{\delta^\top \bG^\top \bG \delta}
\leq \sqrt{\sigma_{\max}\!\left(\bG^\top \bG\right)} \, \|\delta\|_2.
\]
Based on the inequality above, we obtain 
\[
\sup_{t \in [0,1]}
\frac{\|\bm{\ldd}_{/i}\!\left((1-t)\hat{\bm \beta}_{/i} + t\hat{\bm \beta}\right) - \bm{\ldd}_{/i}\!\left(\hat{\bm \beta}\right)\|_2}
     {\|\hat{\bm \beta}_{/i} - \hat{\bm \beta}\|_2}
\;\leq\;
\sup_{t \in [0,1]}
\frac{(1-t)\|\hat{\bm \beta}_{/i} - \hat{\bm \beta}\|_2 \sqrt{\sigma_{\max}\!\left(\bG^\top \bG\right)}}
     {\|\hat{\bm \beta}_{/i} - \hat{\bm \beta}\|_2}
\;\leq\;
\sqrt{\sigma_{\max}\!\left(\bG^\top \bG\right)}.
\]
By Assumption~\ref{as3}, we know that $\sigma_{\max}\!\left(\bG^\top \bG\right) = O_p \!\left({n / \|\bbeta^*\|^2}\right)$, thus the third inequality in Assumption~\ref{as5} holds.

\subsubsection{Linear Regression}
Considering $\ell\!\left(y \mid \bm x^\top \bm\beta\right) = \tfrac{1}{2}\!\left(y - \bm x^\top \bm\beta\right)^2$, we have $\dot{\ell}_i\!\left(\bm{\beta}\right) = -\!\left(y_i - \bm \bg_i^\top \bm\beta\right)$ and $\ddot{\ell}_i\!\left(\bm{\beta}\right) = 1$. Therefore, the second and third inequalities in Assumption~\ref{as5} hold directly.
Moreover, since $y \sim \mathcal{N}\!\left(\bm x^\top \bm\beta^*, \sigma_\epsilon^2\right)$, by properties of sub-Gaussian random variables we obtain
\[
\max_{i \in [n]} \big|\dot{\ell}_i\!\left(\bm{\beta}\right)\big| \preceq \sqrt{\log n}.
\]
Hence, the first inequality in Assumption~\ref{as5} holds.

\subsubsection{Poisson Regression}
Consider
\[
\ell\!\left(y \mid \bm x^\top \bm\beta\right) =  h\!\left(\bm x^\top \bm\beta\right) - y \log h\!\left(\bm x^\top \bm\beta\right) + \log y!,
\]
with the conditional mean $\mathbb{E}\!\left[y \mid \bm x, \bm\beta\right] = h\!\left(\bm x^\top \bm\beta\right)$ where 
$h(s) = \log\!\left(1+e^s\right)$. Then the first- and second-order derivatives of the loss with respect to $s:=\bx^\top \bbeta$ are
\[
\dot{\ell}(y,s) 
= h'(s)\Bigl(1-\frac{y}{h(s)}\Bigr)
= \sigma(s)\Bigl(1-\frac{y}{h(s)}\Bigr),
\]
and
\[
\ddot{\ell}(y,s)
= \sigma(s)\bigl(1-\sigma(s)\bigr)\Bigl(1-\frac{y}{h(s)}\Bigr)
+ y\frac{\sigma(s)^2}{h(s)^2}.
\]
% \[\dddot{\ell}(y,s) = 
% \sigma(s)(1-\sigma(s))(1-2\sigma(s))\Bigl(1-\frac{y}{h(s)}\Bigr)
% + 3y\frac{\sigma(s)^2(1-\sigma(s))}{h(s)^2}
% - 2y\frac{\sigma(s)^3}{h(s)^3}\]

By Lemma 8 of \citet{rad2018scalable},
we have
\[
\sup_{t \in [0,1]} 
\frac{\|\bm \ldd_{/i}\!\left((1-t)\bli + t\bl\right) - \bm \ldd_{/i}\!\left(\bl\right)\|_2}
{\|\bli - \bl\|_2}
\;\;\leq\;\;
\left(1 + 6\|y\|_\infty\right)\sqrt{\sigma_{\max}\!\left(\bm \bG^\top \bm \bG\right)},
\]
and
\[
\|\bm\ld\!\left(\bm\beta\right)\|_\infty \;\;\leq\;\; 1 + \|y\|_\infty.
\]
% By Lemma 8 of \citet{rad2018scalable},
% we have
% \[
% \sup_{t \in [0,1]} 
% \frac{\|\bm \ldd_{/i}\!\left((1-t)\bli + t\bl\right) - \bm \ldd_{/i}\!\left(\bl\right)\|_2}
% {\|\bli - \bl\|_2}
% \;\;\leq\;\;
% \left(1 + 6\|y\|_\infty\right)\sqrt{\sigma_{\max}\!\left(\bm \bG^\top \bm \bG\right)},
% \]
% and
% \[
% \|\bm\ld\!\left(\bm\beta\right)\|_\infty \;\;\leq\;\; 1 + \|y\|_\infty.
% \]

where $\sigma$ is sigmoid function. Since $\bm{x}_i^\top \bm{\beta} \sim \mathcal{N}\!\left(0, \bm{\beta}^\top \bm{\Sigma} \bm{\beta}\right)$, we have 
$\max_{i \in [n]} \left|\bm{x}_i^\top \bm{\beta}\right| = O\!\left(\log n\right)$ and 
$\|y\|_\infty = O\!\left(\mathrm{poly}\log n\right)$, both holding with high probability. 
Therefore, using Assumption~\ref{as3}, the three inequalities in Assumption~\ref{as5} hold. 

\subsection{Confirming our assumptions on the nonlinear function $f(\cdot, \cdot)$}
\label{sec:output}
\subsubsection{Verification of Example~\ref{Ex:lin}.}\label{ssec:proof:example:glm}
Consider the GLM where $f(\bx, \bbeta) = \bx^\top \bbeta$. We verify Assumptions~\ref{as2}, \ref{as3}, \ref{ass:nsg_grad} and \ref{as4}.

First, note that $p = d$ and $\nabla f(\bx, \bbeta) = \bx$. Since we assume $|\beta_i^*| = O(1)$, we have $p \asymp \|\bbeta^*\|^2$. Combined with Assumption~\ref{as1}, this implies that Assumptions~\ref{as2} and \ref{ass:nsg_grad} hold, due to standard properties of sub-Gaussian vectors. Next, because $\nabla^2 f = 0$, Assumption~\ref{as4} holds automatically.

Furthermore, since
\[
\bG(\bbeta)^\top \bG(\bbeta) = \bX^\top \bX,
\]
and $\sigma_{\max}(\bSigma) \asymp \sigma_{\min}(\bSigma) \asymp O_p(1/p)$ with $p \asymp \|\bbeta^*\|^2$, and the vectors $\bx_i,\ i\in[n]$, are independent sub-Gaussian, we obtain
\[
\sigma_{\max}(\bX^\top \bX) \asymp \sigma_{\min}(\bX^\top \bX) \asymp 
O_p\!\left(\frac{1}{\|\bbeta^*\|^2}\right),
\]
which verifies Assumption~\ref{as3}.

\subsubsection{Verification of Example~\ref{Ex:NN}.}\label{ssec:proof:example:nn}
\textbf{Setting and Assumptions.}
We verify Assumptions~\ref{as2}, \ref{as3}, \ref{ass:nsg_grad} and \ref{as4}, in a linear-layer setting.
Let the parameter vector be 
\[
\bbeta = ( \vecop(\bW)^\top, \bv^\top )^\top,
\]
where $\bW \in \RR^{h\times p}$ satisfies 
\[
\sigma_{\min}(\bW^\top\bW)\asymp \sigma_{\max}(\bW^\top\bW).
\]
We further assume $|v_l| = O(1)$; that is, there exist constants $v, V>0$ such that 
\[
v \le |v_l| \le V, \qquad \forall l\in[h].
\]
We also assume the initialization satisfies 
\[
p\|\bv\|^2 \asymp \|\bW\|_F^2,
\]
which holds for common isotropic initializations.

Let $\sigma$ be an activation function with
\[
0\le \sigma'(x)\le E.
\]
For $\bx\in\RR^p$, define $\sigma(\bx)$ element-wise.
Define the function
\[
f(\bx,\bW,\bv)=\bv^\top \sigma(\bW\bx).
\]
The gradients are given by
\[
\nabla_{\bv} f(\bx,\bW,\bv)=\sigma(\bW\bx)\in\RR^h,
\]
\[
\nabla_{\bW} f(\bx,\bW,\bv)
= \diag\big(\sigma'(\bW\bx)\big)\,\bv\,\bx^\top \in\RR^{h\times p}.
\]

For notational convenience, for $l\in[h]$ and $i\in B\subset[n]$, define
\[
A^{B}_{li}=\sigma'(\bW_l^\top\bx_i), \qquad 
D^{B}_{li}=v_l \sigma'(\bW_l^\top\bx_i). % Corrected notation here
\]
Then the relationship between the Gram matrices is
\[
D^B (D^B)^\top
 = \diag(\bv)\, A^B (A^B)^\top \diag(\bv).
\]

Assume that for any $B\subset[n]$ with $|B|\gg h$, with high probability we can find constants $\underline{\alpha},\bar{\alpha}>0$ such that 
\[
\bar{\alpha}|B|
 \;\ge\;
\lambda_{\max}(A^BA^{B,\top})
 \;\ge\;
\lambda_{\min}(A^BA^{B,\top})
 \;\ge\;
\underline{\alpha}|B|.
\]
Intuitively, if for all $l\in[h]$, 
\[
\sum_{i\in B}\sigma'^2(\bW_l^\top\bx_i)\asymp |B|,
\]
then
\[
\tr(A^BA^{B,\top})
 =\sum_{l=1}^h \sum_{i\in B}\sigma'^2(\bW_l^\top\bx_i)
 \asymp h|B|.
\]
If $A^BA^{B,\top}$ is well-conditioned, this implies the eigenvalue bounds above.
Consequently, utilizing the bounds on $v_l$,
\[
V^2\,\bar{\alpha}\,|B|
 \;\ge\;
\lambda_{\max}(D^BD^{B,\top})
 \;\ge\;
\lambda_{\min}(D^BD^{B,\top})
 \;\ge\;
v^2\,\underline{\alpha}\,|B|.
\]
We denote $A^{[n]} =A$ and $D^{[n]} =D$ for simplicity.

We further assume $n> d$ and $n\gg h^2$. 
Intuitively, if $n<d$ where $d$ is the parameter dimension, the empirical Fisher 
\[
\bG^\top \bG=\sum_{i=1}^n \bg_i\bg_i^\top
\]
cannot be invertible. 
Similarly, in very wide networks, the sample size must dominate the width to obtain stable approximations.

\paragraph{Verification of Assumption~\ref{as2}.}
First, we verify Assumption~\ref{as2}. Note that
\[
\|\nabla_\bv f(\bx, \bW, \bv)\| = \|\sigma(\bW \bx)\| \leq E\|\bW\bx\| = O_{\HP}(\sqrt{\tr(\bW\bSigma \bW^\top )}) = O_{\HP}(\sqrt{\|\bW\|^2_F/\|\bbeta^*\|^2}) = O_{\HP}(1).
\]
\[
\|\nabla_\bW f(\bx, \bW, \bv)\|_F^2 \leq E^2\|\bv\|^2\|\bx\|^2 = O_{\HP}\left({\frac{p\|\bv\|^2}{\|\bbeta^*\|^2}}\right) = O_{\HP}(1),
\]
using the scaling assumption $p\|\bv\|^2 \asymp \|\bW\|_F^2$.

\paragraph{Verification of Assumption~\ref{as3}.}
Next, we verify Assumption~\ref{as3}. For simplicity, we consider 
\[
\bg_i = \vecop(\nabla_\bW f(\bx_i, \bW, \bv)) = c_i\otimes\bx_i,
\]
where $c_i= D_{:,i}$ is the $i$-th column of $D$. Denoting $\bG = (\bg_1,\dots,\bg_n)$, we know that
\[
M:=\frac{1}{n} \bG^\top \bG = \frac{1}{n}\sum_{i=1}^n \bg_i \bg_i^\top = \frac{1}{n} \sum_{i=1}^n c_ic_i^\top \otimes \bx_i\bx_i^\top = \frac{1}{n} \sum_{i=1}^n(c_i\otimes \bx_i)(c_i\otimes \bx_i)^\top.
\]
To analyze the eigenvalues of $M$, consider any $q\in \RR^{ph}$ such that $q = \vecop(Q)$ and $\|q\|=\|Q\|_F = 1$. We only need to bound 
\[
q^\top Mq = q^\top \left(\frac{1}{n} \sum_{i=1}^n(c_i\otimes \bx_i)(c_i\otimes \bx_i)^\top\right) q = \frac{1}{n}\sum_{i=1}^n (\bx_i^\top Qc_i)^2.
\]

Note that $Q\in \RR^{p\times h}$. We consider two cases separately: $\text{rank}(Q) = 1$ and $\text{rank}(Q)>1$.

First, consider the case when $\text{rank}(Q) = 1$. There exist $a\in \RR^p$ and $b \in \RR^h$ such that $Q = a b^\top$. WLOG, assume $\|a\|^2 = \|b\|^2= 1$. Then we have:
\begin{align}
q^\top M q
&= a^\top \Big[\tfrac{1}{n}(\bx_1,\ldots,\bx_n)
    \diag\big((b^\top c_1)^2,\ldots,(b^\top c_n)^2\big)
    (\bx_1,\ldots,\bx_n)^\top\Big] a \notag \\
\label{eq:goal_ex2}
&\ge \|a\|^2\, 
   \lambda_{\min}\!\left(\tfrac{1}{n}(\bx_1,\ldots,\bx_n)
   \diag\big((b^\top c_1)^2,\ldots,(b^\top c_n)^2\big)
   (\bx_1,\ldots,\bx_n)^\top\right).
\end{align}

For any $b \in \RR^h$ with $\|b\| = 1$, denote $S_b = \{i:\ (b^\top c_i)^2 \geq \frac{\underline{\alpha} v^2}{2}\}$ and let $S_b^c = [n]\setminus S_b$. We know that 
\begin{align*}
    \|b\|^2\underline{\alpha} n v^2 &\leq b^\top DD^\top b = \sum_{i=1}^n(b^\top c_i)^2 = \sum_{i\in S_b} (b^\top c_i)^2+\sum_{i\notin S_b} (b^\top c_i)^2 \\
    &\leq \|b\|^2 E^2V^2 h |S_b| + (n-|S_b|)\frac{\underline{\alpha} v^2}{2}\\
    &= \left(E^2V^2 h - \frac{\underline{\alpha} v^2}{2}\right)|S_b|+\frac{\underline{\alpha} v^2}{2} n,
\end{align*}
where in the second inequality we used $|c_i|^2 \leq E^2V^2 h$.
Thus, we get 
\[
|S_b| \geq \frac{\underline{\alpha} v^2}{2E^2V^2h - \underline{\alpha} v^2}n \gg h,
\]
using $n\gg h^2$. We can further refine the inequality:
\begin{align*}
    \|b\|^2\underline{\alpha} n v^2 &\leq b^\top DD^\top b = \sum_{i=1}^n(b^\top c_i)^2 = \sum_{i\in S_b} (b^\top c_i)^2+\sum_{i\notin S_b} (b^\top c_i)^2 \\
    &\leq \|b\|^2 \bar{\alpha}V^2 |S_b| + (n-|S_b|)\frac{\underline{\alpha} v^2}{2}\\
    &= \frac{\underline{\alpha} v^2}{2} n + |S_b|\left(\bar{\alpha}V^2 - \frac{\underline{\alpha} v^2}{2}\right),
\end{align*}
where we used $b^\top D^{S_b}D^{S_b,\top} b \leq \|b\|^2|S_b| \bar{\alpha}V^2$ for $|S_b| \gg h$. Then we obtain
\[
|S_b| \geq \frac{\underline{\alpha} v^2}{2\bar{\alpha} V^2 - \underline{\alpha} v^2} n := C(\underline{\alpha},\bar{\alpha}, v,V)n.
\]
Returning to \eqref{eq:goal_ex2}, we have
\begin{align*}
    q^\top M q &\geq \lambda_{\min}\!\left(\tfrac{1}{n}(\bx_1,\ldots,\bx_n)
   \diag\big((b^\top c_1)^2,\ldots,(b^\top c_n)^2\big)
   (\bx_1,\ldots,\bx_n)^\top\right) \\
   &\geq \lambda_{\min}\left(\frac{1}{n}\sum_{i\in S_b} \bx_i\bx_i^\top (b^\top c_i)^2\right)\geq \lambda_{\min}\left(\frac{1}{n}\sum_{i\in S_b} \bx_i\bx_i^\top \frac{\underline{\alpha }v^2}{2}\right) \\
   &= \frac{\underline{\alpha }v^2}{2} \frac{|S_b|}{n} \lambda_{\min}\left(\frac{1}{|S_b|}\sum_{i\in S_b} \bx_i\bx_i^\top\right).
\end{align*}
By the sub-Gaussianity of i.i.d $\bx_i$, we have
\[
\lambda_{\min}\left(\frac{1}{|S_b|}\sum_{i\in S_b} \bx_i\bx_i^\top\right) - \lambda_{\min}(\bSigma) \left(1-\sqrt{\frac{p}{|S_b|}}\right)^2 \to 0, \quad \as,
\]
where $\lambda_{\min}(\bSigma) (1-\sqrt{\frac{p}{|S_b|}})^2\asymp \|\bbeta^*\|^{-2}$. 
Finally, we get 
\[
q^\top M q \succeq \|\bbeta^*\|^{-2},\quad \as.
\]

Next, we consider the upper bound in the rank-1 case. Similar to \eqref{eq:goal_ex2}, we get
\begin{align}
\noindent
q^\top M q
&= a^\top \Big[\tfrac{1}{n}(\bx_1,\ldots,\bx_n)
    \diag\big((b^\top c_1)^2,\ldots,(b^\top c_n)^2\big)
    (\bx_1,\ldots,\bx_n)^\top\Big] a \notag \\
\label{eq:goal_ex2_2}
&\le \|a\|^2\, 
   \lambda_{\max}\!\left(\tfrac{1}{n}(\bx_1,\ldots,\bx_n)
   \diag\big((b^\top c_1)^2,\ldots,(b^\top c_n)^2\big)
   (\bx_1,\ldots,\bx_n)^\top\right).
\end{align}
Define $T_b = \{i: (b^\top c_i)^2 \geq ph\}$. We know that 
\[
    \|b\|^2\bar{\alpha} n V^2 \geq b^\top DD^\top b = \sum_{i=1}^n(b^\top c_i)^2 = \sum_{i\in T_b} (b^\top c_i)^2+\sum_{i\notin T_b} (b^\top c_i)^2 \geq |T_b| ph.
\]
So we get $|T_b|\leq \frac{\bar{\alpha} V^2}{ph}n$. Note that
\begin{align*}
   &\lambda_{\max}\!\left(\tfrac{1}{n}(\bx_1,\ldots,\bx_n)
   \diag\big((b^\top c_1)^2,\ldots,(b^\top c_n)^2\big)
   (\bx_1,\ldots,\bx_n)^\top\right) \\
   &\leq \lambda_{\max}\left(\tfrac{1}{n}\sum_{i\in T_b}\bx_i\bx_i^\top (b^\top c_i)^2\right)+ \lambda_{\max}\left(\tfrac{1}{n}\sum_{i\notin T_b}\bx_i\bx_i^\top (b^\top c_i)^2\right).
\end{align*}
We consider these two parts. For the first part:
\[
    \lambda_{\max}\left(\tfrac{1}{n}\sum_{i\in T_b}\bx_i\bx_i^\top (b^\top c_i)^2\right) \preceq \frac{|T_b|}{n} \frac{p}{\|\bbeta^*\|^2} E^2V^2 h\|b\|^2 \leq \frac{\bar{\alpha} V^2}{ph} \frac{ph E^2V^2}{\|\bbeta\|^2} \asymp \|\bbeta^*\|^{-2},
\]
with high probability. For the second part:
\begin{align*}
    \lambda_{\max}\left(\tfrac{1}{n}\sum_{i\notin T_b}\bx_i\bx_i^\top (b^\top c_i)^2\right) &\leq \frac{n-|T_b|}{n}\lambda_{\max}\left(\frac{1}{n-|T_b|} \sum_{i\notin T_b}\bx_i\bx_i^\top (b^\top c_i)^2\right) \\
     &\leq \lambda_{\max} (\bSigma)\left(1+\sqrt{\frac{p}{n-|T_b|}}\right)^2 \asymp \|\bbeta^*\|^{-2},
\end{align*}
with high probability, where we used that $\|\bx_i\bx_i^\top\|_2 = \bx_i^\top \bx_i = O_{\HP} (\frac{p}{\|\bbeta^*\|^2})$ and $|T_b|\leq \frac{\bar{\alpha} V^2}{ph}n$.
Thus, 
\[
q^\top M q \preceq \|\bbeta^*\|^{-2}, \quad \text{with high probability}.
\]
This concludes the rank-1 case.

Next, we consider the case when $\text{rank}(Q)=R>1$. We can write $Q=\sum_{r=1}^R a_r b_r^\top$, where $\{a_r\}_{r\in [R]}$ and $\{b_r\}_{r\in [R]}$ are orthogonal sets respectively. WLOG, assume $\|b_r\| = 1$ for $r=1,\dots,R$. Since $\|Q\|_F^2 = 1$, we have $\sum_{r=1}^R \|a_r\|^2=1$. Note that 
\[
    q^\top M q \approx \sum_{r=1}^R a_r^\top \left[\tfrac{1}{n}(\bx_1,\ldots,\bx_n)
   \diag\big((b_r^\top c_1)^2,\ldots,(b_r^\top c_n)^2\big)
   (\bx_1,\ldots,\bx_n)^\top\right]a_r.
\]
Using results from Case 1, we already know that for each $r$:
\begin{align*}
   &\lambda_{\min}(\tfrac{1}{n}(\bx_1,\ldots,\bx_n)
   \diag\big((b_r^\top c_1)^2,\ldots,(b_r^\top c_n)^2\big)
   (\bx_1,\ldots,\bx_n)^\top) \\\asymp &\lambda_{\max}\left(\tfrac{1}{n}(\bx_1,\ldots,\bx_n)
   \diag\big((b_r^\top c_1)^2,\ldots,(b_r^\top c_n)^2\big)
   (\bx_1,\ldots,\bx_n)^\top\right) \asymp \|\bbeta^*\|^{-2}.
\end{align*}
We then obtain
\[
q^\top M q \asymp \sum_{r=1}^R \|a_r\|^2 \|\bbeta^*\|^{-2} \asymp \|\bbeta^*\|^{-2}.
\]

\paragraph{Verification of Assumption~\ref{ass:nsg_grad}.}
For simplicity, we consider $\sigma(x)=x$.
We have
\[
\bg_{i,*}
=
\begin{pmatrix}
\bW\\[3pt]
\bv\otimes\II_p
\end{pmatrix}
\bx_i,
\]
and
\[
\left\|
\begin{pmatrix}
\bW\\[3pt]
\bv\otimes\II_p
\end{pmatrix}
\right\|
=
\sqrt{
  \lambda_{\max}(\bW^\top\bW)+\|\bv\|^2
}
\leq\sqrt{\frac{2\|\bW\|^2_F}{p}} \asymp \frac{\|\bbeta^*\|}{\sqrt{p}},
\]
where for the first inequality we use the assumption that $\lambda_{\min}(\bW^\top\bW)\asymp \lambda_{\max}(\bW^\top\bW)$.

Using $\bx_i\sim \text{zero-mean sub-Gaussian with covariance } \bSigma$, where $\|\bSigma\| \asymp \|\bbeta^*\|^{-2}$, we know that $\frac{\|\bbeta^*\|}{\sqrt{p}}\bx_i$ is zero-mean sub-Gaussian with covariance $\bSigma'$ such that
$\|\bSigma'\|\asymp \frac{1}{p}$. Thus $\frac{\|\bbeta^*\|}{\sqrt{p}}\bx_i \sim \text{nSG}(\sigma)$ for some $\sigma>0$. 
Then
\begin{align*}
\mathbb{P}\!\left(
\left\|
\begin{pmatrix}
\bW\\[3pt]
\bv\otimes\II_p
\end{pmatrix}\bx_i
- 
\mathbb{E}
\begin{pmatrix}
\bW\\[3pt]
\bv\otimes\II_p
\end{pmatrix}\bx_i
\right\|
\ge t
\right)
&\le \mathbb{P}\left(\left\|\begin{pmatrix}
\bW\\[3pt]
\bv\otimes\II_p
\end{pmatrix}\right\|\|\bx_i-\mathbb{E}\bx_i\|\geq t\right) \\
&\asymp \mathbb{P}\left(\frac{\|\bbeta^*\|}{\sqrt{p}} \|\bx_i-\mathbb{E}\bx_i\| \geq t\right) \\
&\leq 2\exp\!\left(-\frac{t^2}{2\sigma^2}\right).
\end{align*}
Therefore,
\[
\bg_{i,*} \overset{i.i.d.}{\sim} \text{nSG}(\sigma).
\]

\section{Simulations and empirical studies}
\label{sec:simresults}
\subsection{Correlation-based experiments}
\label{sec:corr-ind}

\subsubsection{Vanilla linear models}
\label{sec:sim_linear_main}

We first consider the linear case, where \( f(\bx,\bbeta^*) = \bx^\top \bbeta^* \), and focus on binary classification and Poisson regression as representative examples. In this setting, the linearization error is exactly zero; therefore, only the ALO step and the projection step need to be evaluated.

We generate the rows \( \bx_1^\top,\dots,\bx_n^\top \) of the design matrix \( \bX \) independently from a mean-zero multivariate normal distribution with covariance matrix \( \bSigma \). Here, \( \bSigma \) is a Toeplitz matrix satisfying
\[
\mathrm{corr}(X_{ij}, X_{ij'}) = 0.1^{|j-j'|}.
\]
We rescale \( \bbeta^* \) and \( \bSigma \) such that \( \|\bbeta^*\|^2 = p \) and \( \bbeta^{*\top} \bSigma \bbeta^* = 1 \), ensuring that our theoretical assumptions are satisfied.

\paragraph{Binary classification.}
For binary classification, we randomly select \(100\) training points to remove. To reduce variability, we independently sample \(10\) new test points \( \bz_{\new} \) from the same distribution. For each test point, we compute its influence with respect to each removed training point. This procedure yields \(1{,}000\) realizations of \( \mathcal{I}^{\rm True} \), \( \mathcal{I}^{\rm ALO} \), and \( \mathcal{I}^{\rm TRAK}(k) \), where \(k\) denotes the projection dimension.

\textbf{ALO Step.}
Figure~\ref{fig:binary_alo} validates Theorem~\ref{thm:g_base_ALO}. As shown, \( \mathcal{I}^{\rm True} \) and \( \mathcal{I}^{\rm ALO} \) are nearly identical, with fitted slopes close to one. This confirms that the ALO step introduces only negligible magnitude error, consistent with our theoretical guarantee.

\begin{figure}[ht]
    \centering
    \includegraphics[width=0.45\linewidth]{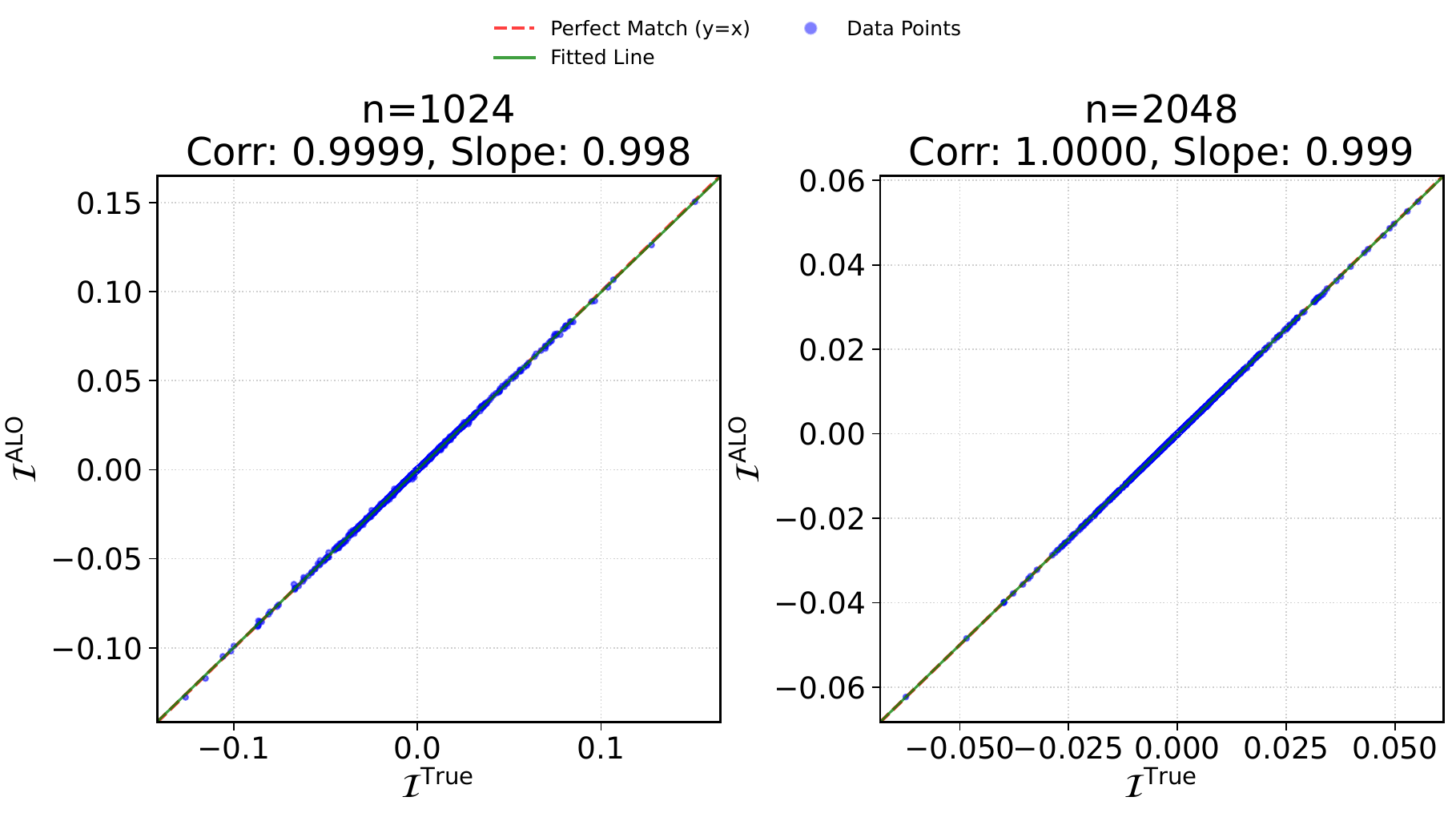}
    \includegraphics[width=0.45\linewidth]{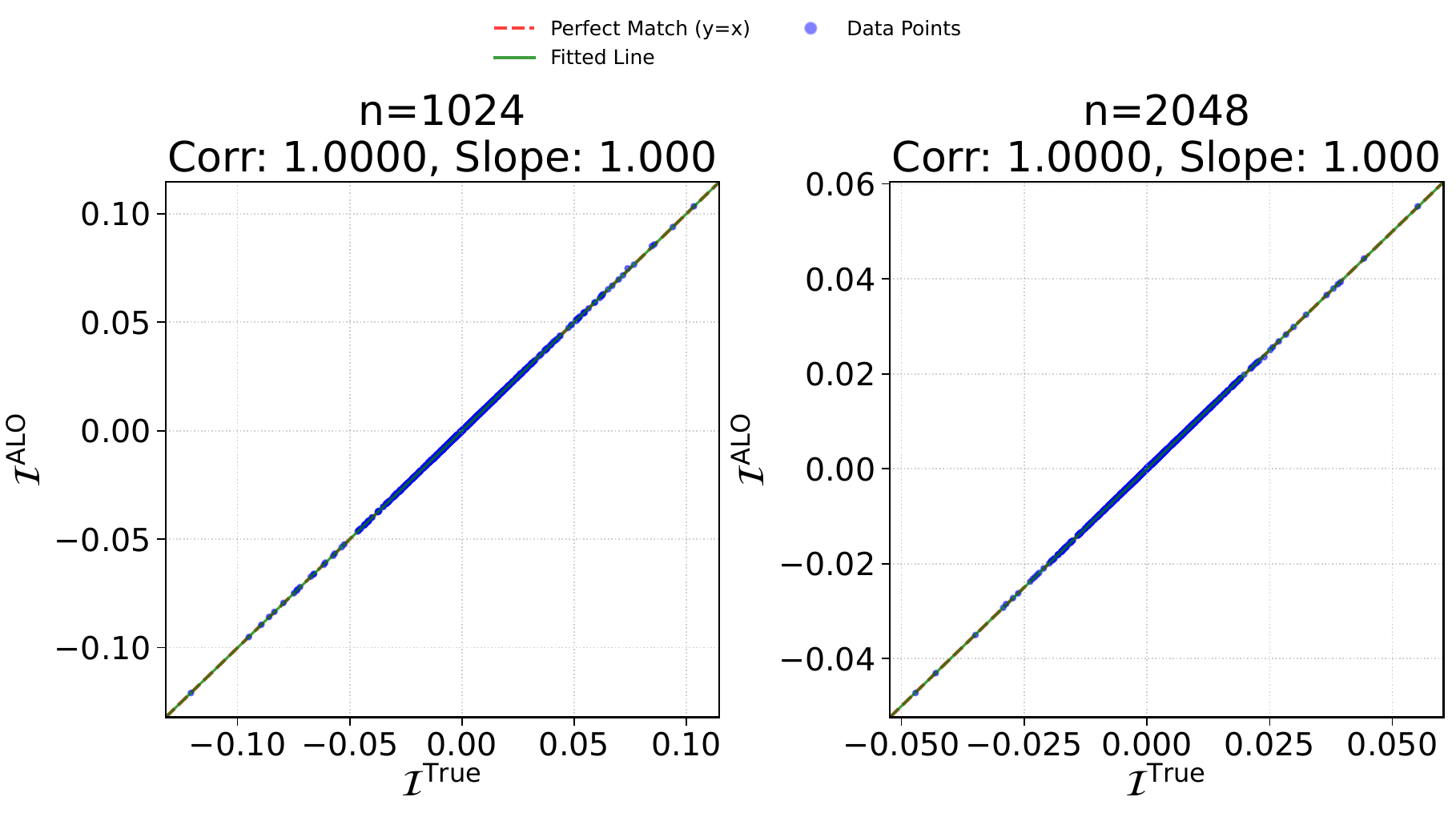}
    \caption{
        \textbf{Left two panels:} Binary logistic regression with \(p=100\), \(n=1024\), and \(n=2048\). The x-axis shows \( \mathcal{I}^{\rm True} \), and the y-axis shows \( \mathcal{I}^{\rm ALO} \).
        \textbf{Right two panels:} Poisson regression under the same \(p\) and \(n\) settings.
    }
    \label{fig:binary_alo}
\end{figure}

\textbf{Projection Step.}
Figure~\ref{fig:binary_proj} compares \( \mathcal{I}^{\rm True} \) and \( \mathcal{I}^{\rm TRAK} \) for \(k=0.75d\), \(0.50d\), and \(0.25d\), where \(d=p=100\). The slopes decrease as \(k\) decreases, validating Theorem~\ref{theo:magnitude_projection}. This shows that the projection step systematically reduces the magnitude of the TRAK influence. Nevertheless, for sufficiently large \(k\), the projected influence still preserves the qualitative separation between influential and non-influential data points.

\begin{figure}[ht]
    \centering
    \includegraphics[width=0.3\linewidth]{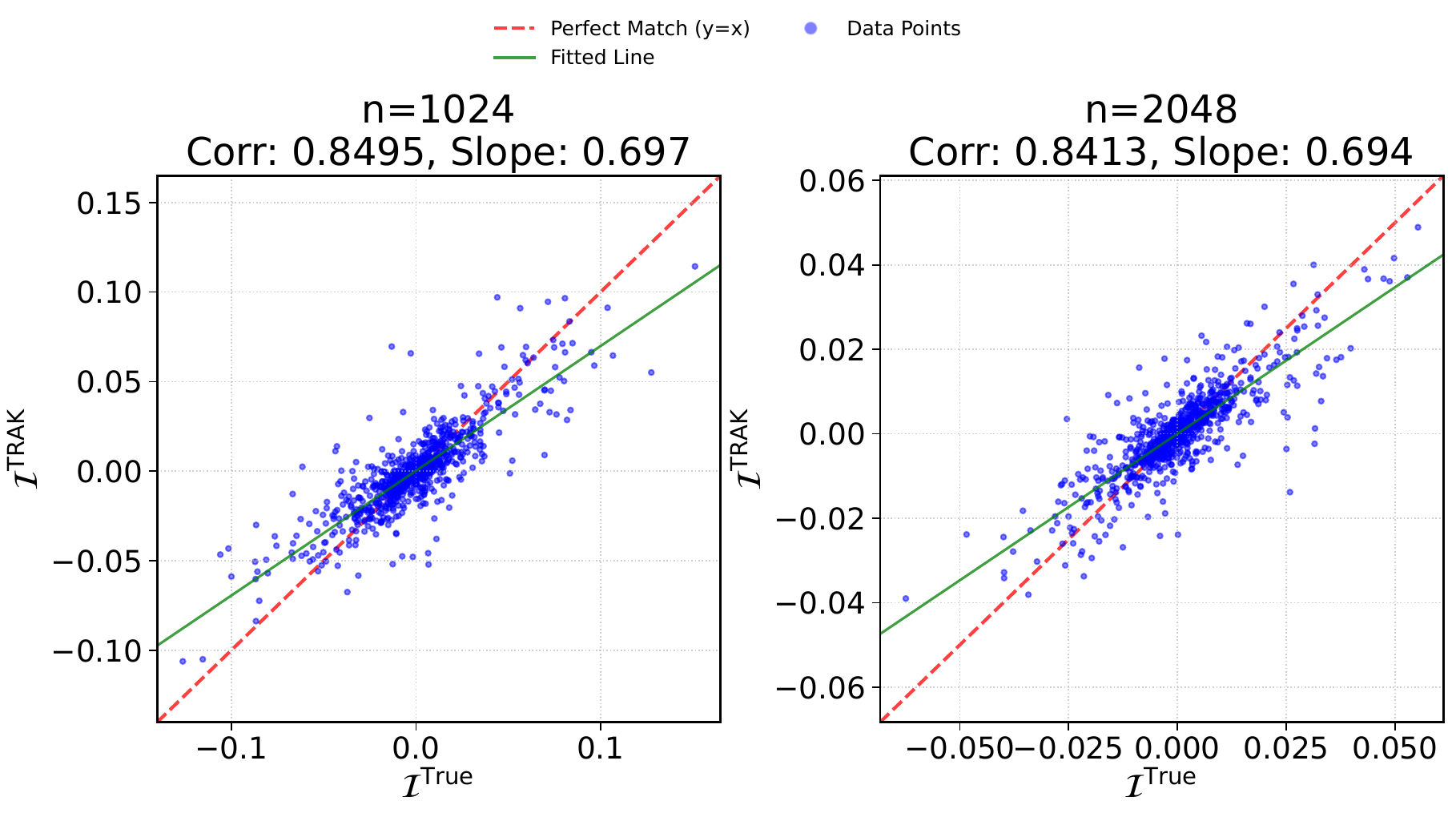}
    \includegraphics[width=0.3\linewidth]{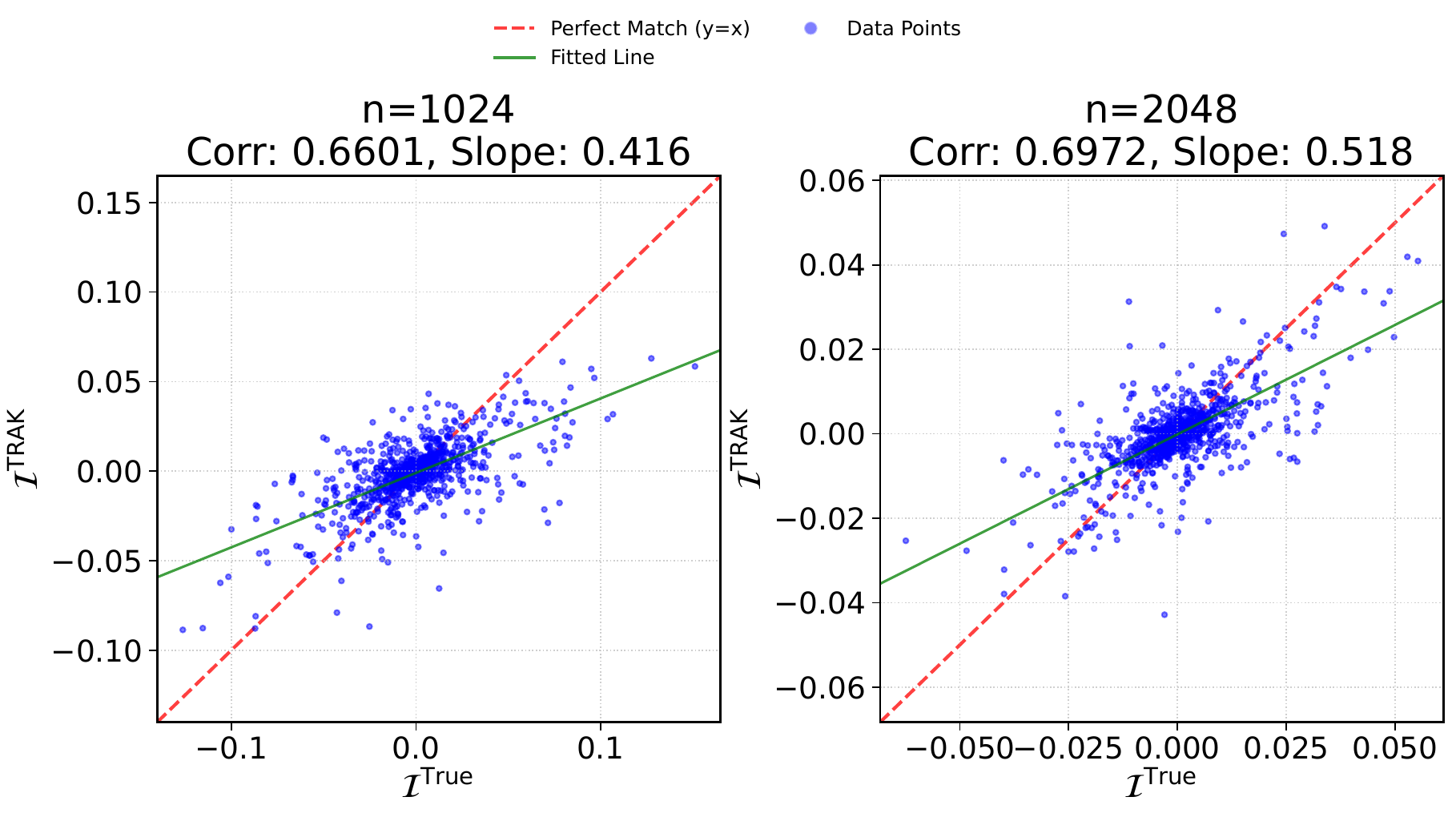}
    \includegraphics[width=0.3\linewidth]{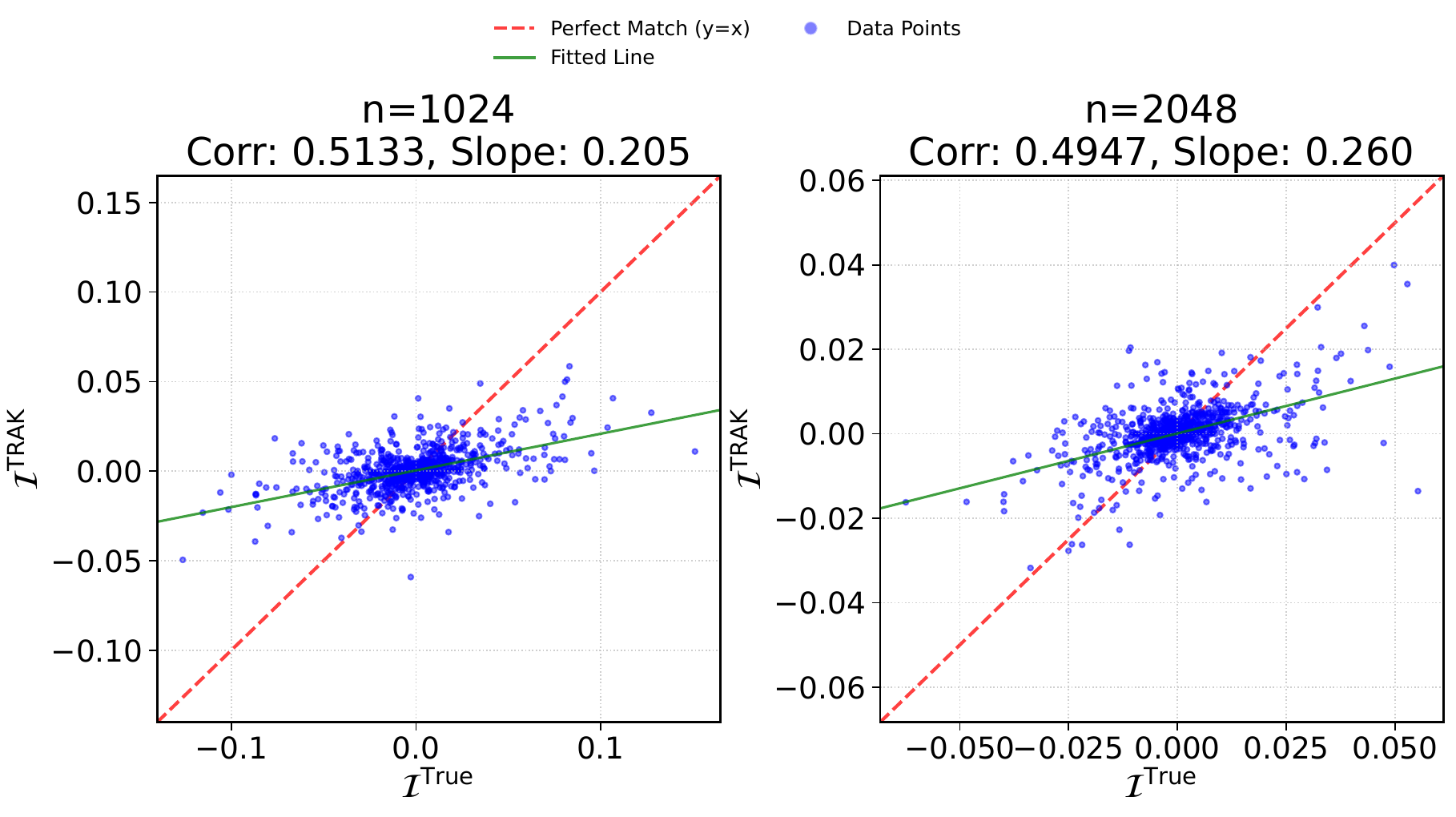}
    \caption{
        Binary logistic regression with projection (\(p=100\), \(n=1024\), and \(n=2048\)). The x-axis shows \( \mathcal{I}^{\rm True} \), and the y-axis shows \( \mathcal{I}^{\rm TRAK} \). From left to right, the projection dimensions are \(k=75\), \(50\), and \(25\).
    }
    \label{fig:binary_proj}
\end{figure}

\begin{figure}[ht]
    \centering
    \includegraphics[width=0.3\linewidth]{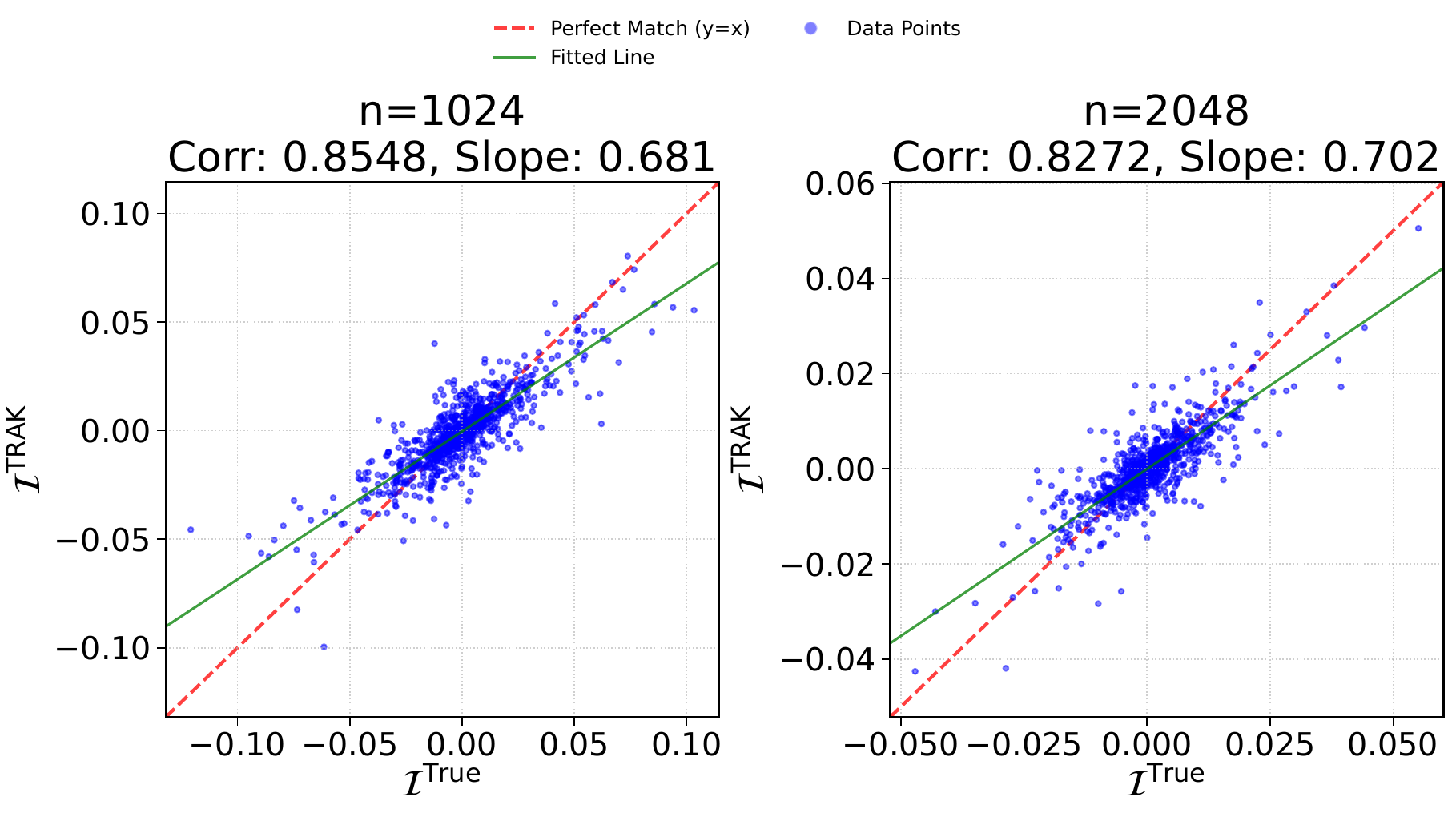}
    \includegraphics[width=0.3\linewidth]{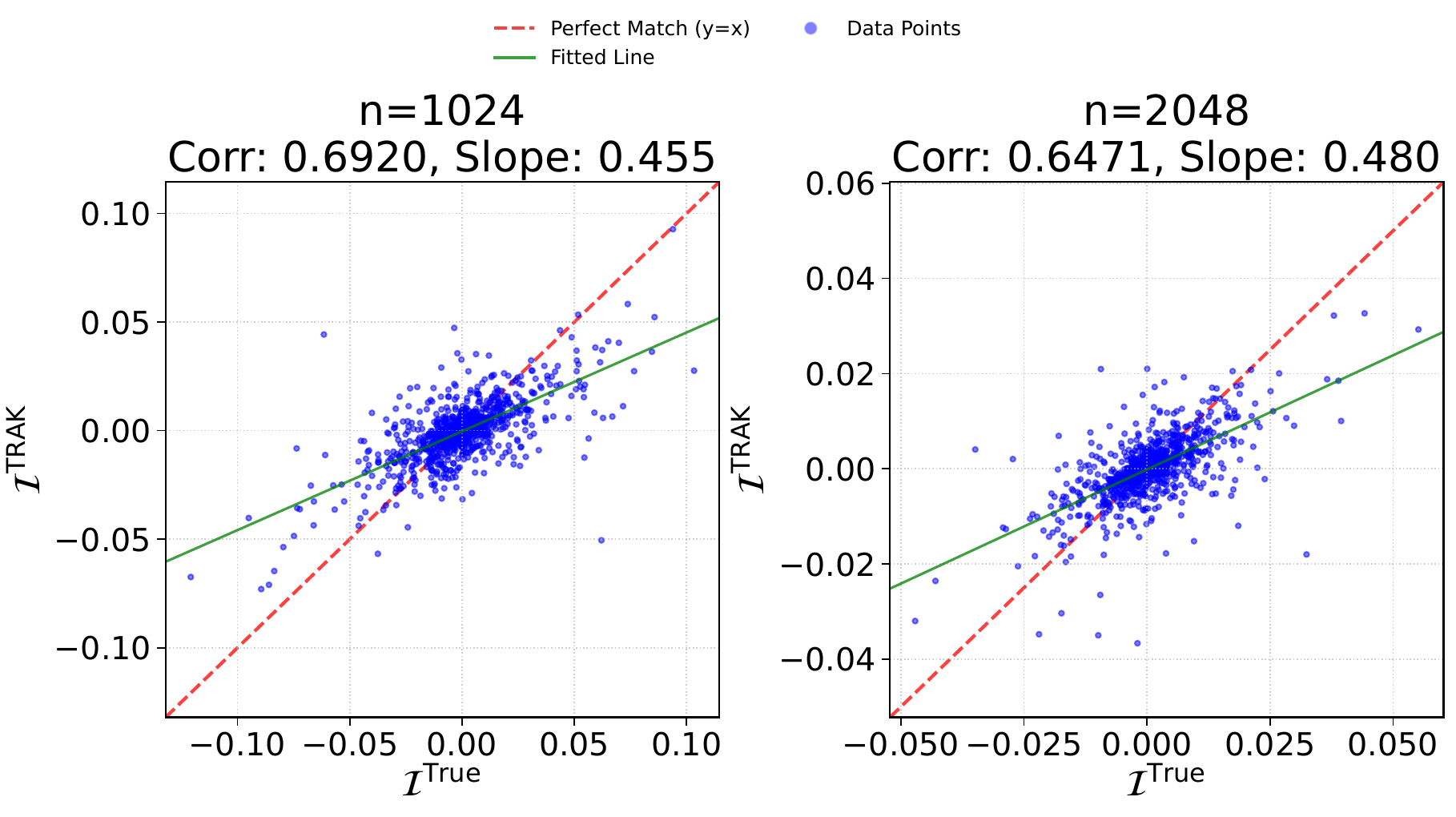}
    \includegraphics[width=0.3\linewidth]{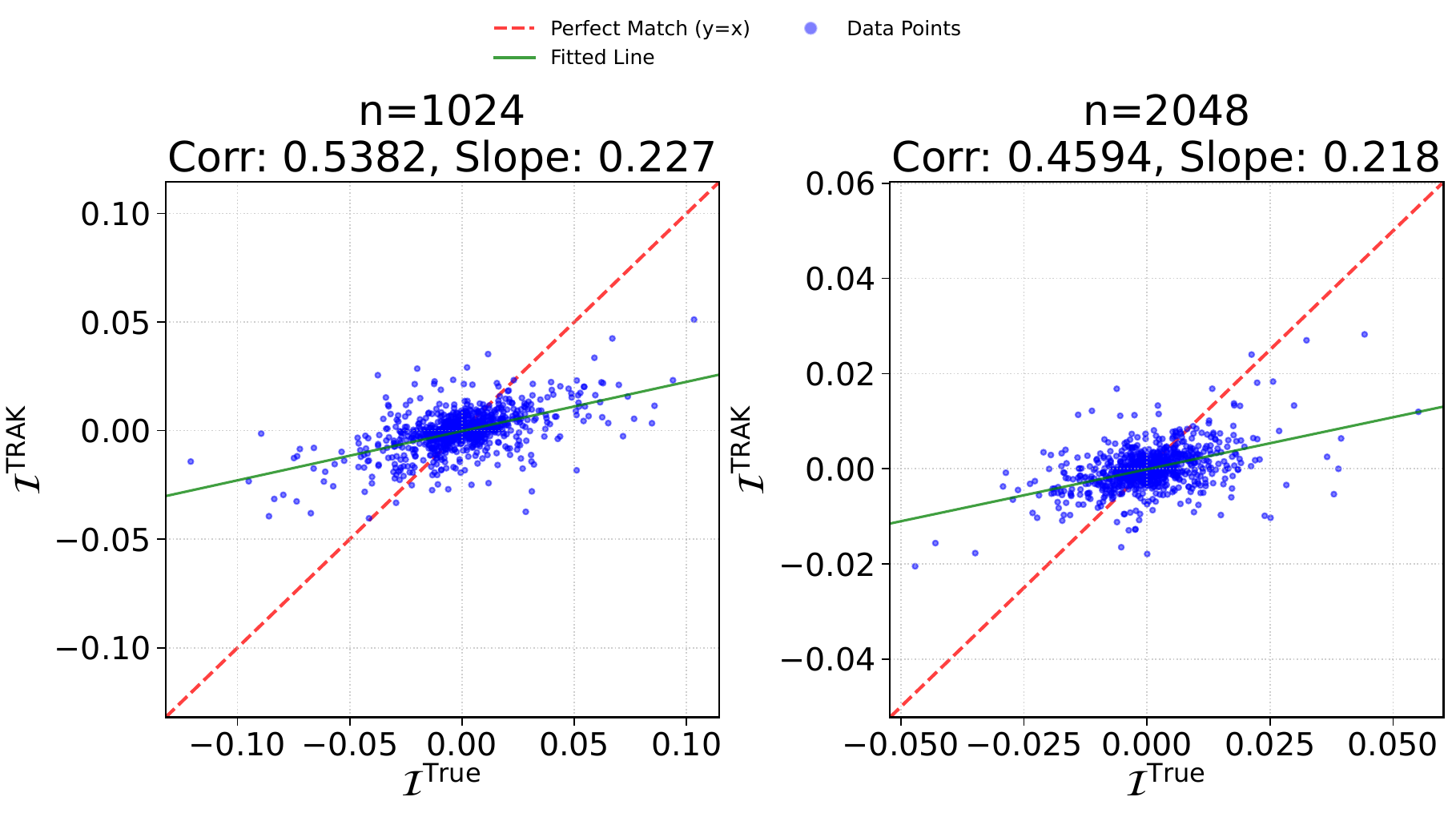}
    \caption{
        Poisson regression with projection. The settings are identical to those in Figure~\ref{fig:binary_proj}.
    }
    \label{fig:poisson_proj}
\end{figure}

\paragraph{Poisson regression.}
We next consider Poisson regression with a softplus mean function. As shown in Figures~\ref{fig:binary_alo} and~\ref{fig:poisson_proj}, we observe trends similar to those in binary classification, further validating Theorems~\ref{thm:g_base_ALO} and~\ref{theo:magnitude_projection}.

\subsubsection{Nonlinear case}
\label{sec:K5_corr}

\paragraph{Simulation setting.}
We next consider nonlinear multi-class classification. The feature vectors $\bx_1^\top,\ldots,\bx_n^\top$ are sampled independently from
$\mathcal{N}(\mathbf{0},\bSigma)$, where $\bSigma$ is a Toeplitz covariance matrix satisfying
\[
\mathrm{corr}(X_{ij},X_{ij'})=0.1^{|j-j'|}.
\]
We consider multi-class classification with $K=3$ in Figure~\ref{fig:multi_alo}, results for $K=5$ are shown in Appendix Figure~\ref{fig:multi_alo_K_5}. The simulation results are qualitatively consistent across both settings.

We rescale $\bbeta^*$ and $\bSigma$ such that $\|\bbeta^*\|^2=p$ and
$\|\bSigma\|=\|\bbeta^*\|^{-2}$.
For each trial $t$, we generate $(\bX^{(t)},y^{(t)})$ with
\[
y^{(t)}\sim\mathcal{M}(\mathrm{softmax}(\bX^{(t)}\bW^{*(t)\top})),
\]
where $\mathcal{M}$ denotes the multinomial distribution and $\bW^{*(t)} \in \RR^{K\times p}$. To ensure identifiability, we set $\bW_K^{*(t)}=\mathbf{0}_p$. We write the free model parameters as
\[
\bbeta^* \triangleq \mathrm{vec}(\bW^{*(t)}_{[1:(K-1),:]}) \in \RR^{d}, \quad d = (K-1)\times p.
\]

For the uncorrelated case, we randomly remove $100$ training points and independently sample $10$ new test points $\bz_{\new}$, yielding $1{,}000$ realizations of the influence quantities. For the correlated case, we set $\bz_i=\bz_{\new}$, yielding $100$ realizations. These two settings allow us to test whether the approximations preserve the separation between dependent and independent data points, and the rankings of $\mathcal{I}^{\rm True}$.

\begin{figure}[ht]
    \centering
    \includegraphics[width=0.48\linewidth]{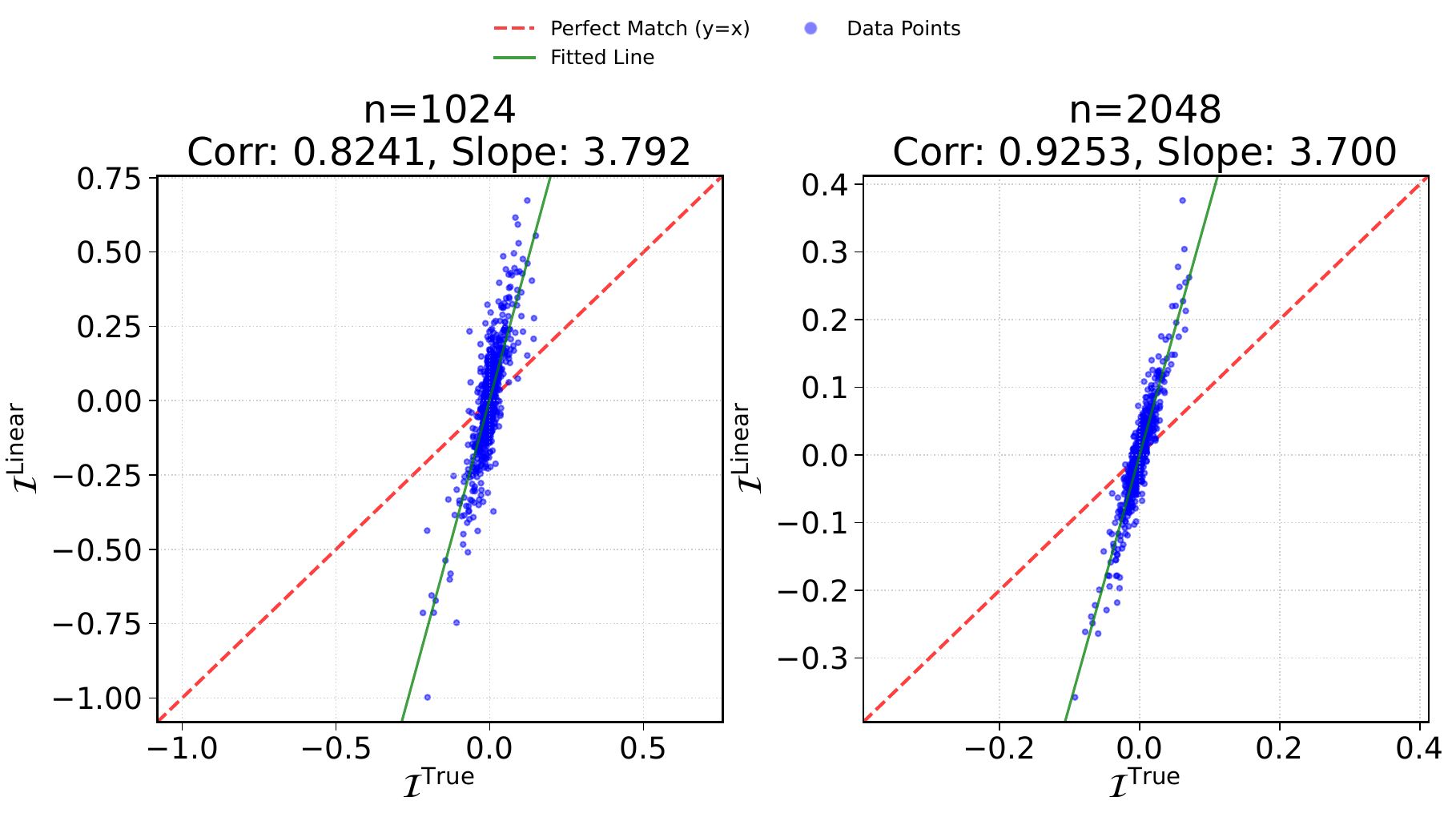}
    \includegraphics[width=0.48\linewidth]{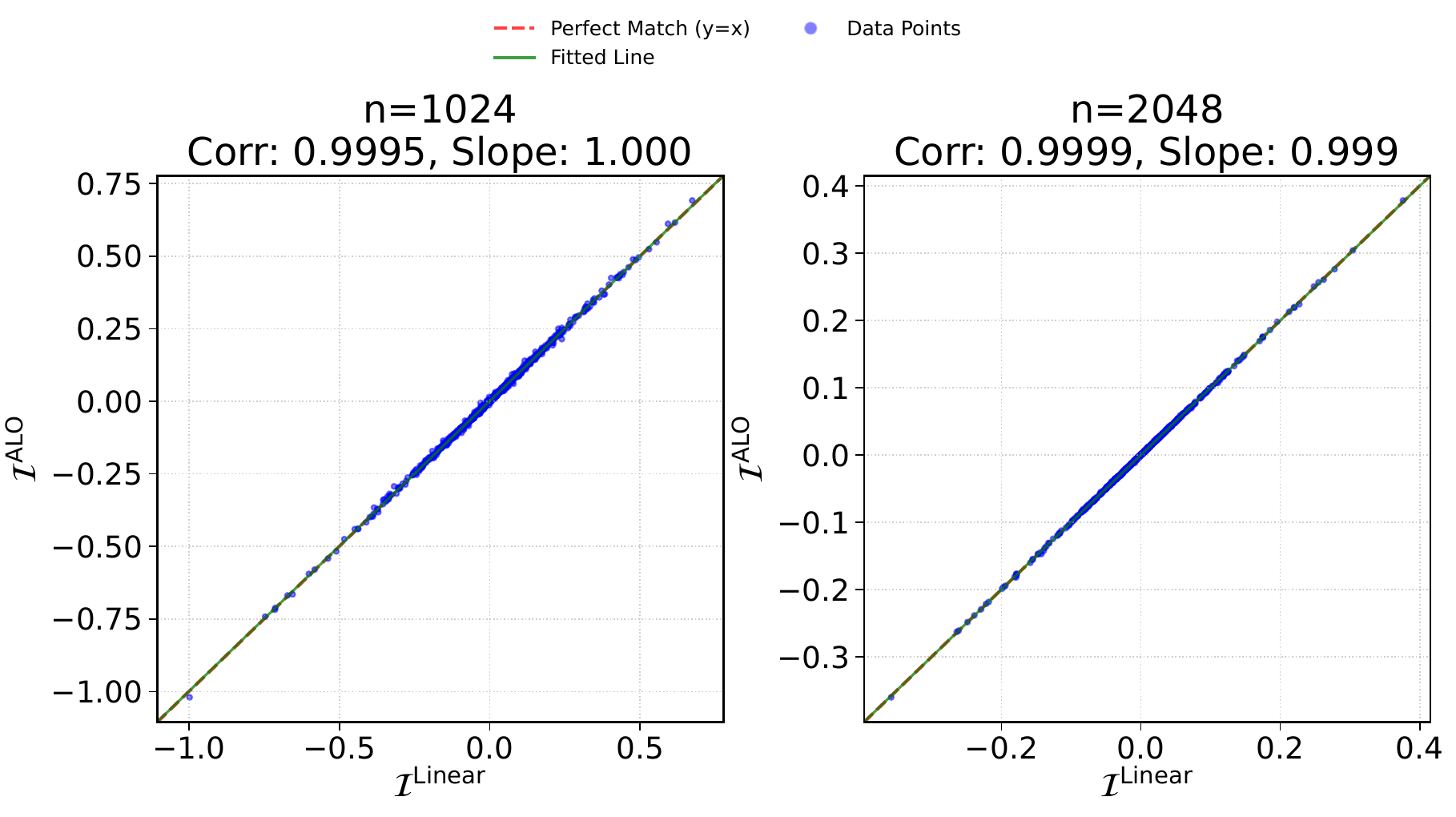}
    \includegraphics[width=0.48\linewidth]{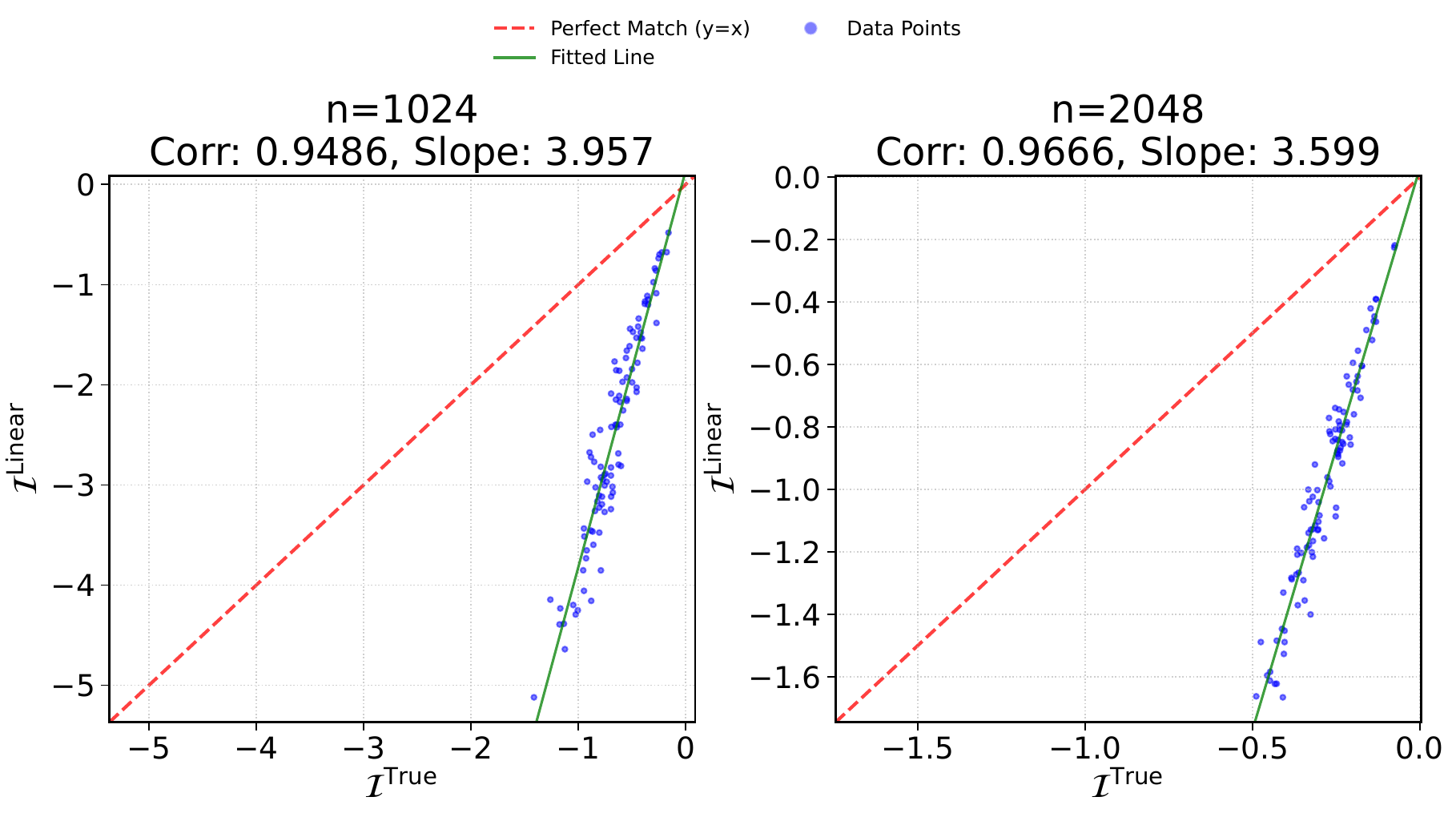}
    \includegraphics[width=0.48\linewidth]{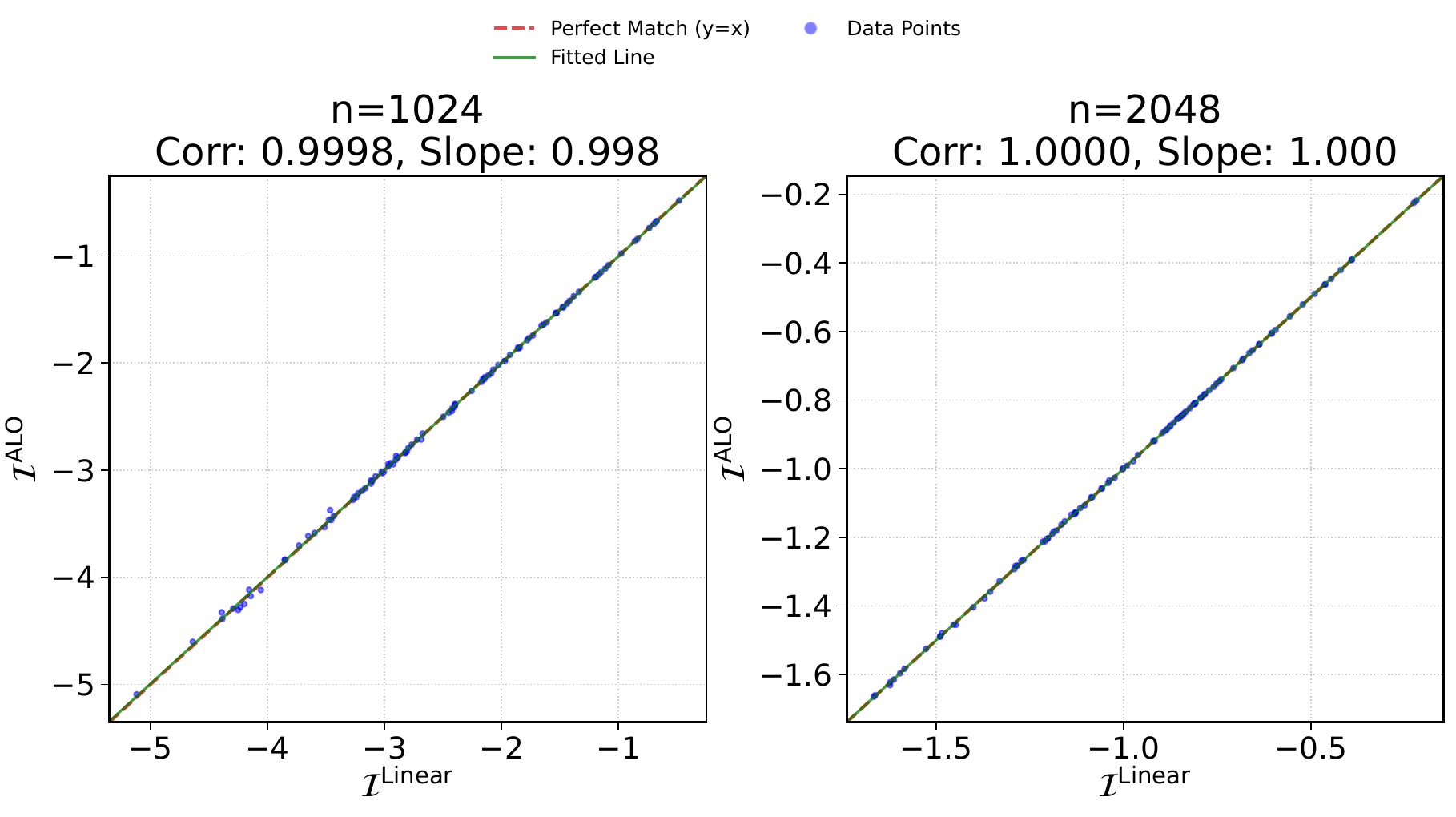}
    
    \includegraphics[width=0.32\linewidth]{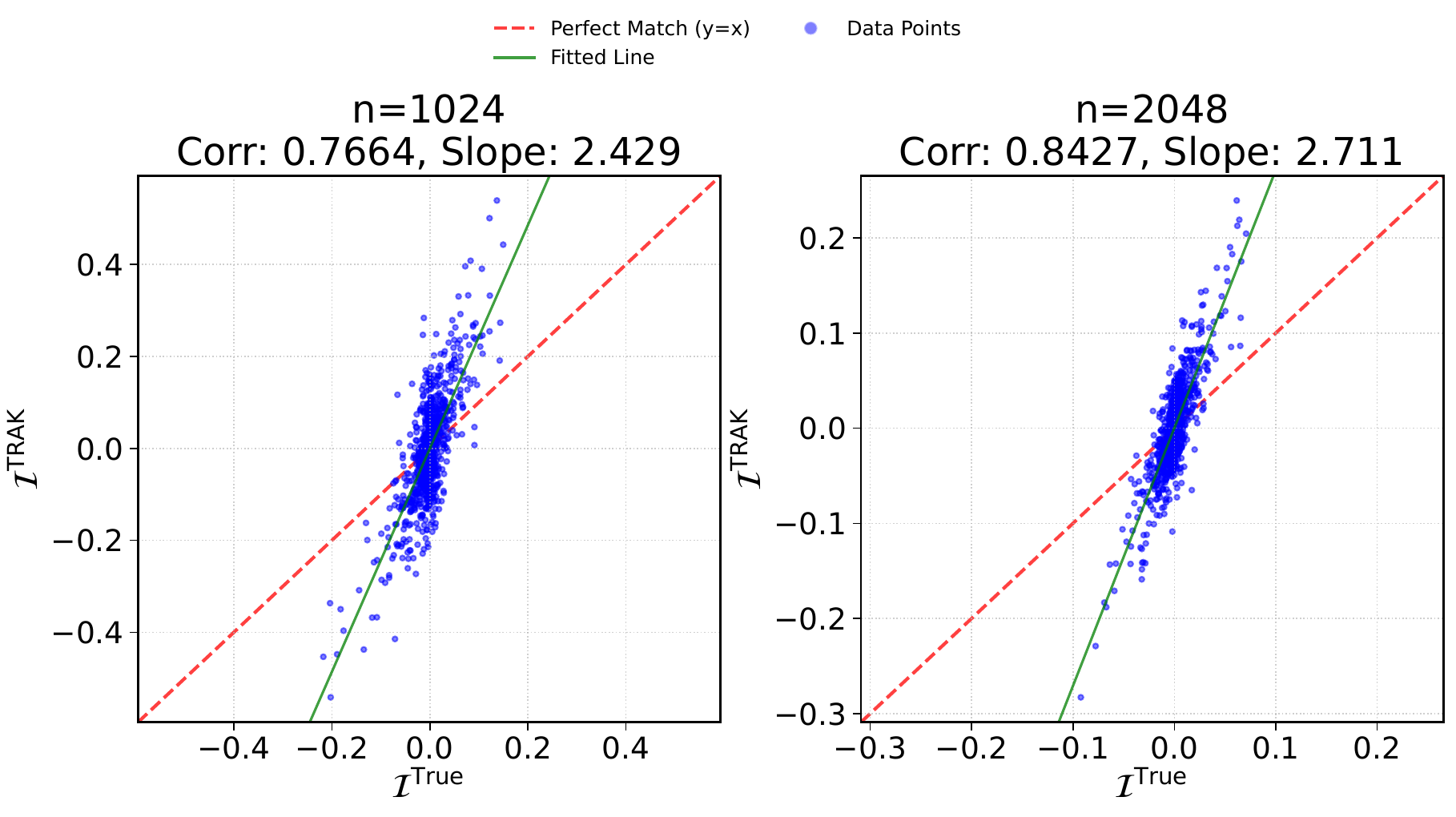}
    \includegraphics[width=0.32\linewidth]{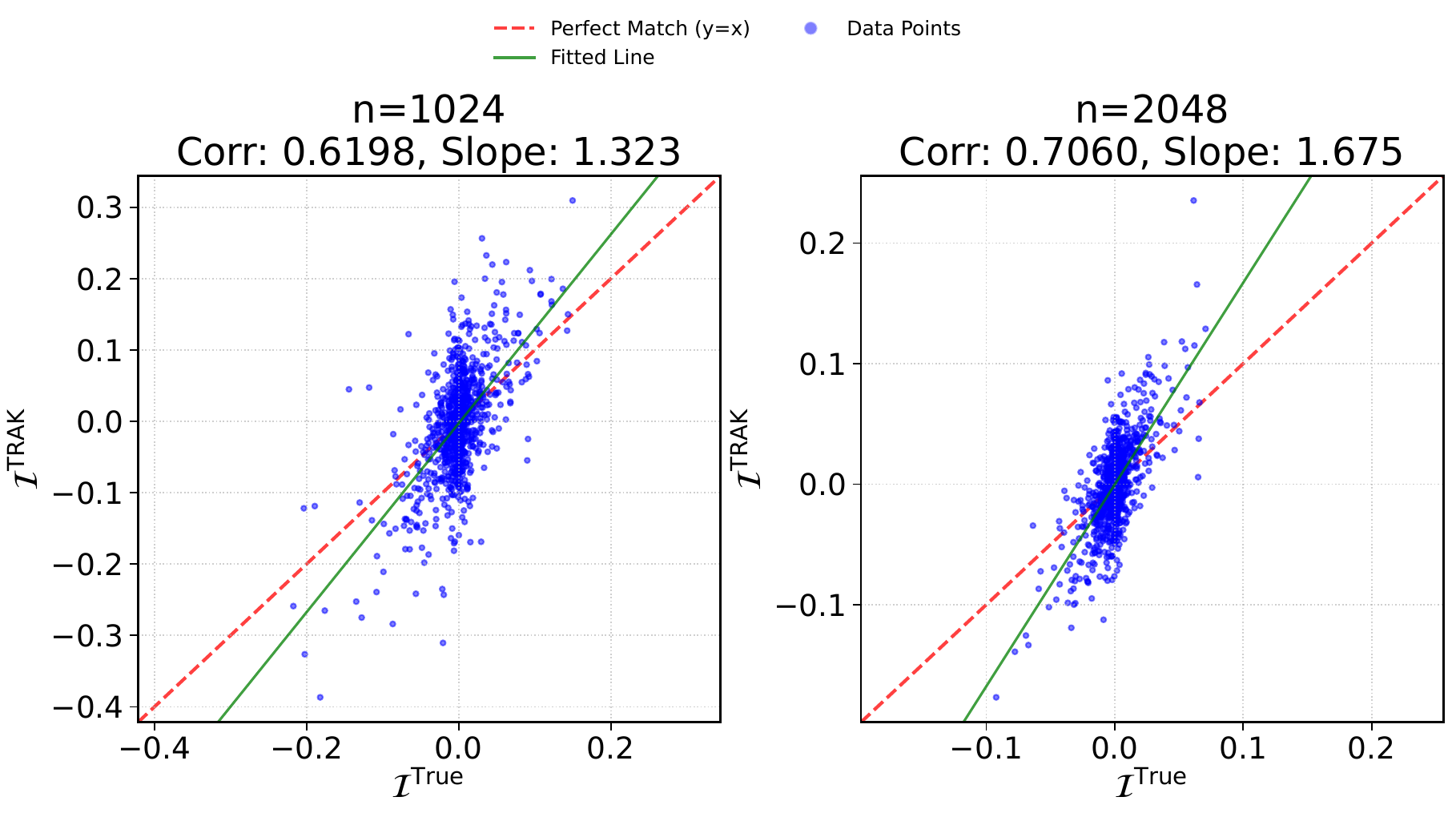}
    \includegraphics[width=0.32\linewidth]{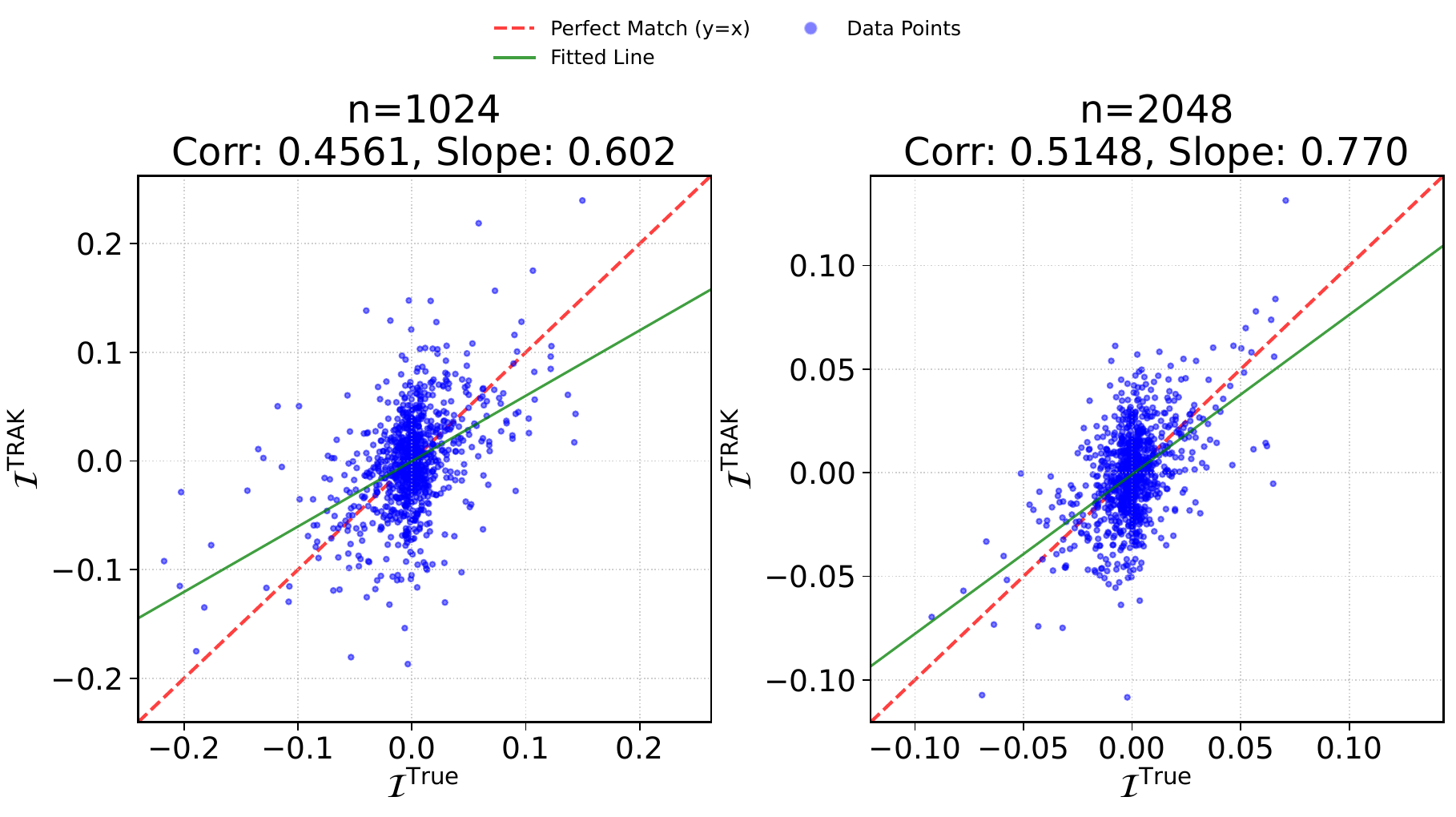}
    
    \includegraphics[width=0.32\linewidth]{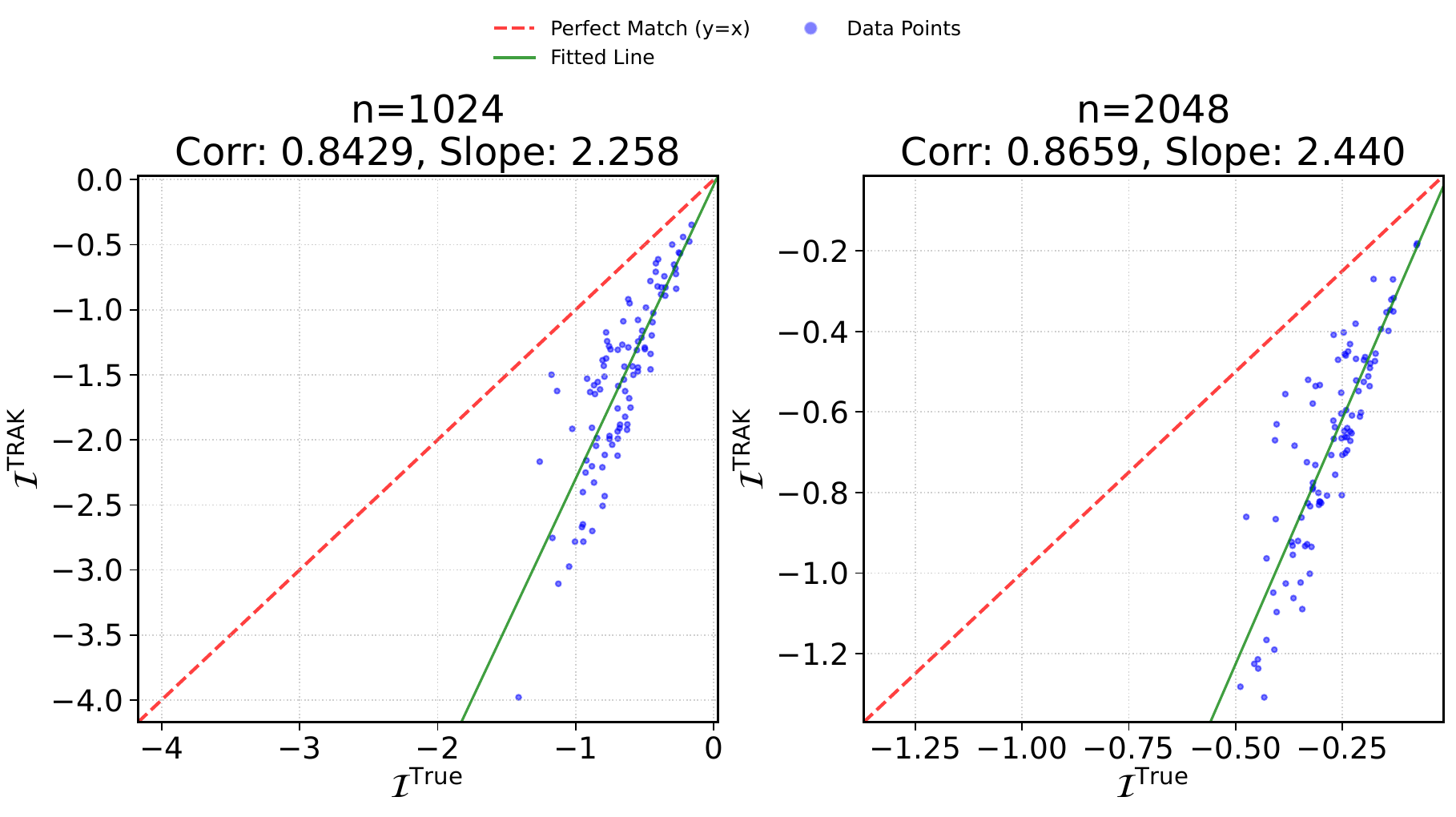}
    \includegraphics[width=0.32\linewidth]{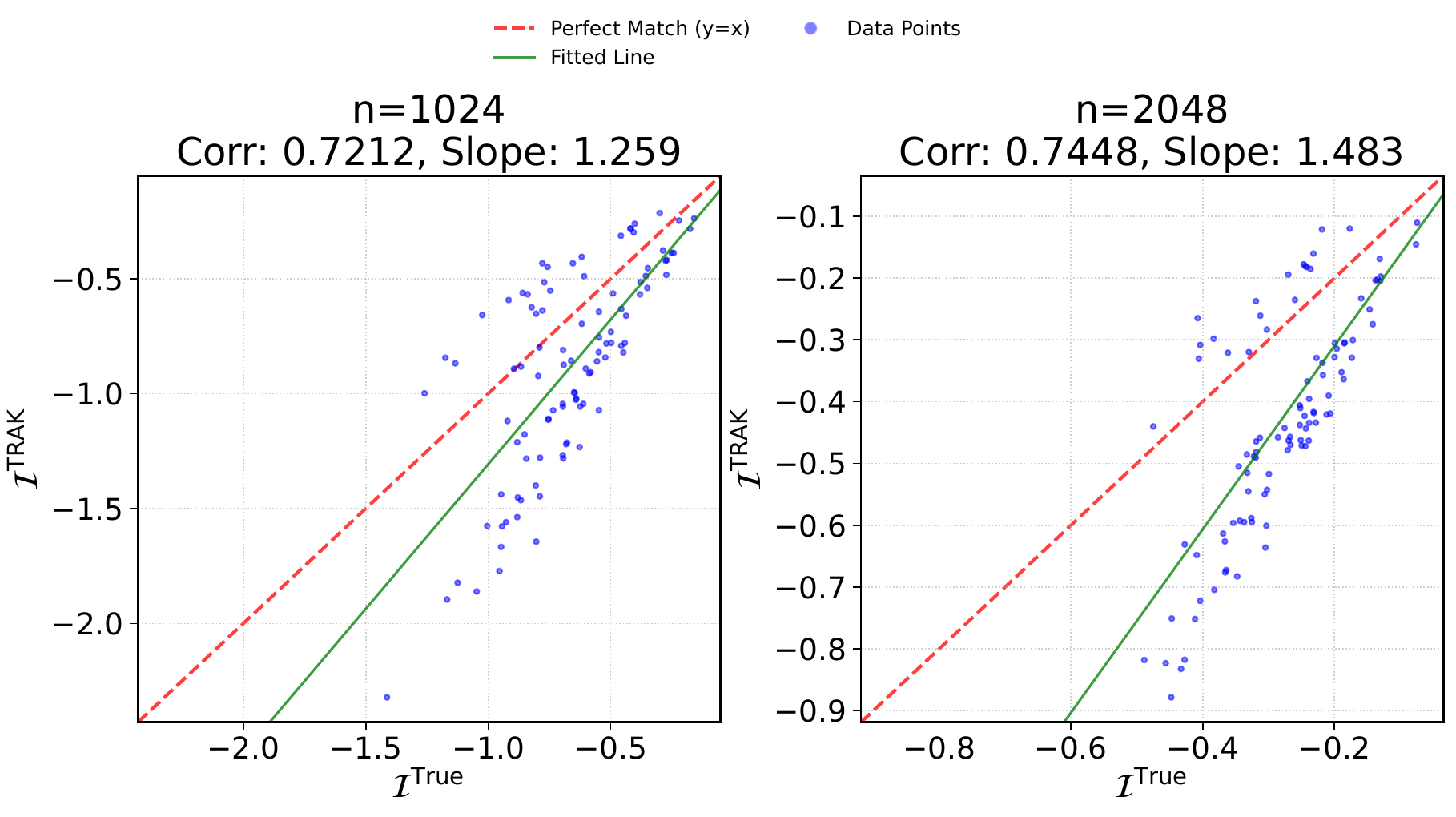}
    \includegraphics[width=0.32\linewidth]{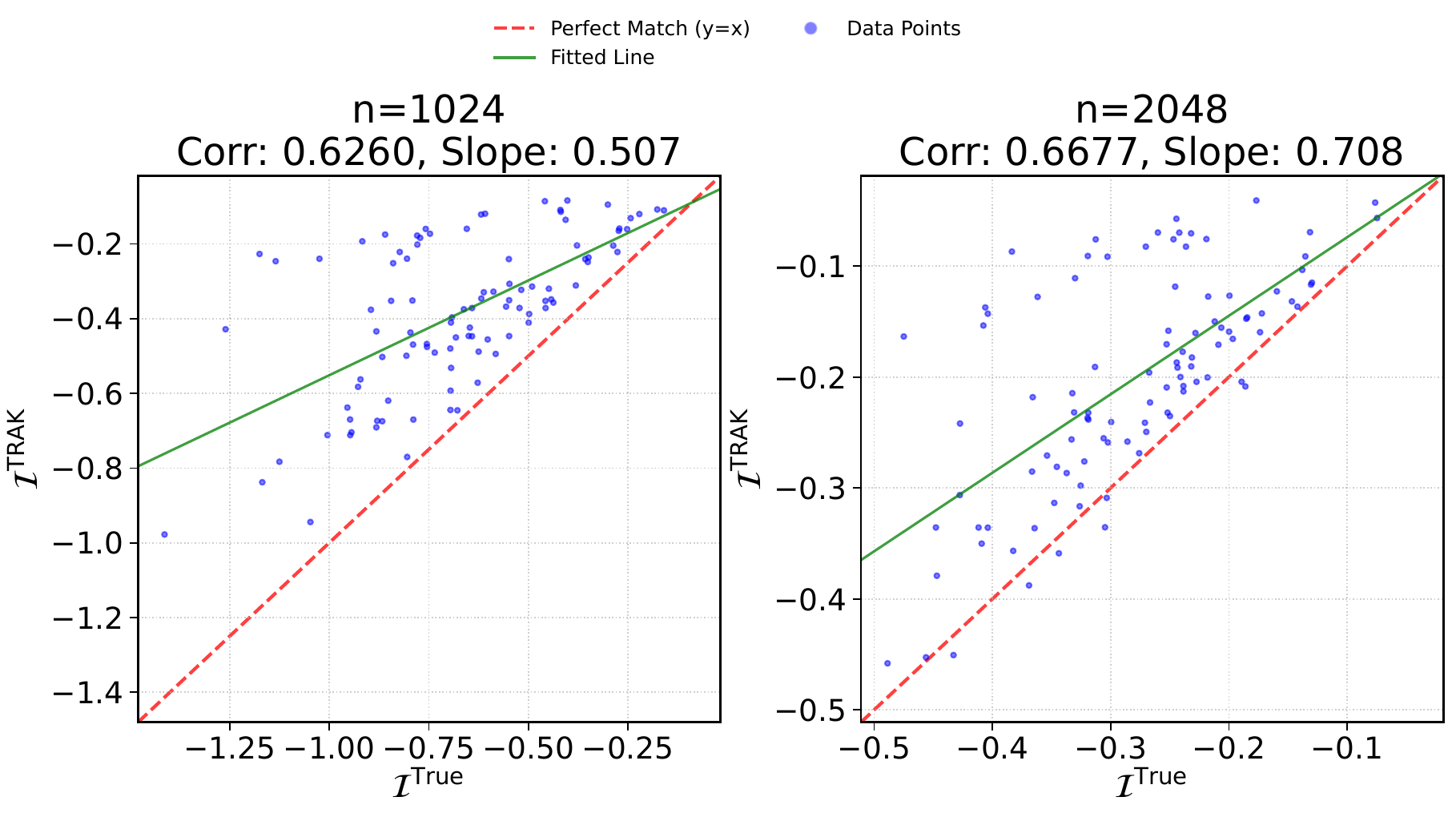}
    \caption{
Experimental results for $5$-class classification with $p=100$ and $d=400$.
\textbf{First and third rows:} Results for the uncorrelated case, where $\bz_i$ and $\bz_{\new}$ are independent.
\textbf{Second and fourth rows:} Results for the correlated case, where $\bz_i=\bz_{\new}$.
\textbf{First and second rows:} $\mathcal{I}^{\rm Linear}$ versus $\mathcal{I}^{\rm True}$ and $\mathcal{I}^{\rm ALO}$ versus $\mathcal{I}^{\rm Linear}$.
\textbf{Third and fourth rows:} $\mathcal{I}^{\rm TRAK}(k)$ versus $\mathcal{I}^{\rm True}$ with $k=75\%d,50\%d,25\%d$, corresponding to $k=300,200,100$.
}
    \label{fig:multi_alo_K_5}
\end{figure}

\subsection{Magnitude-based simulations}
\label{sec:magnitude_check}

In this section, we examine how the magnitudes of the influence quantities and approximation errors vary with the sample size $n$, the feature dimension $p$, and the projection dimension $k$. We focus on the nonlinear setting of $3$-class classification as a representative example. The data-generating mechanism and other experimental details are the same as in Appendix~\ref{sec:corr-ind}. We consider both the correlated case, where $\bz_i=\bz_{\new}$, and the uncorrelated case, where $\bz_i$ and $\bz_{\new}$ are independent.

\begin{figure}[t]
    \centering
    \includegraphics[width=0.32\linewidth]{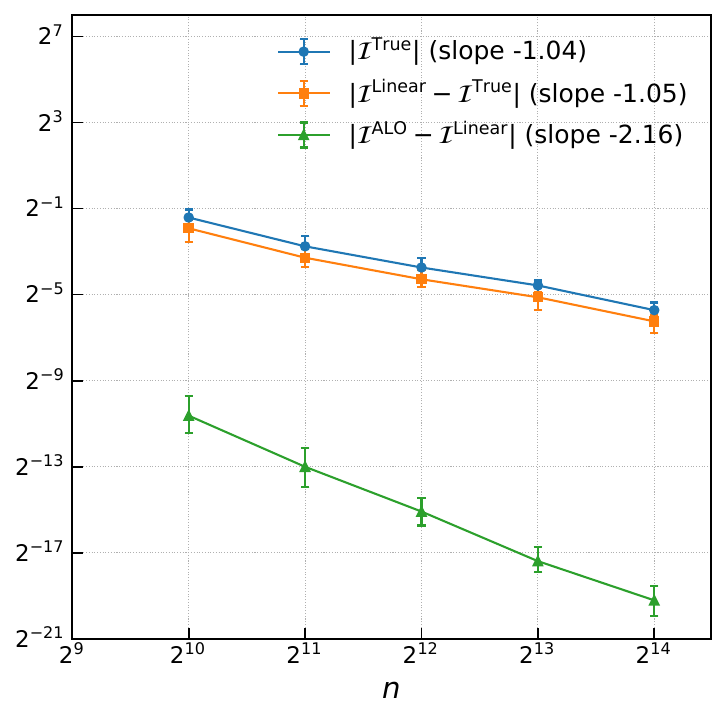}
    \includegraphics[width=0.32\linewidth]{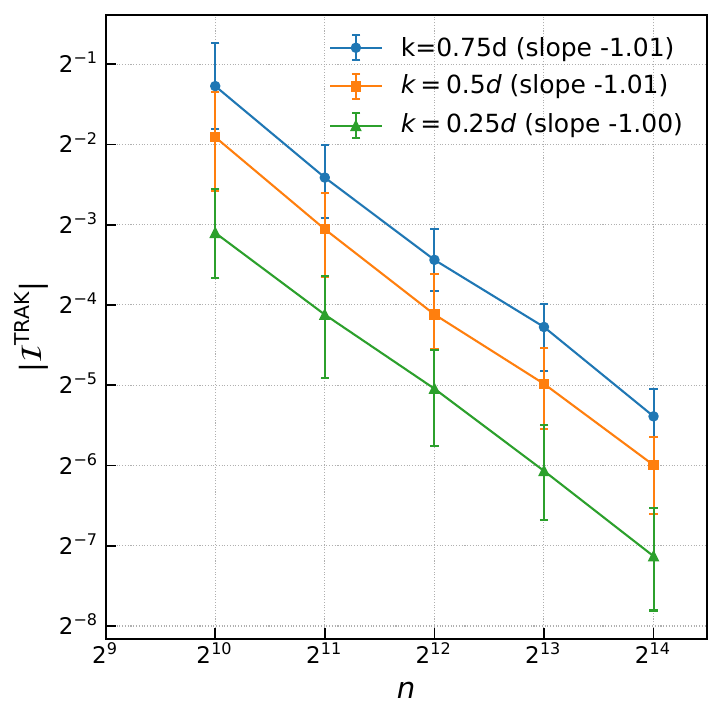}
    \includegraphics[width=0.32\linewidth]{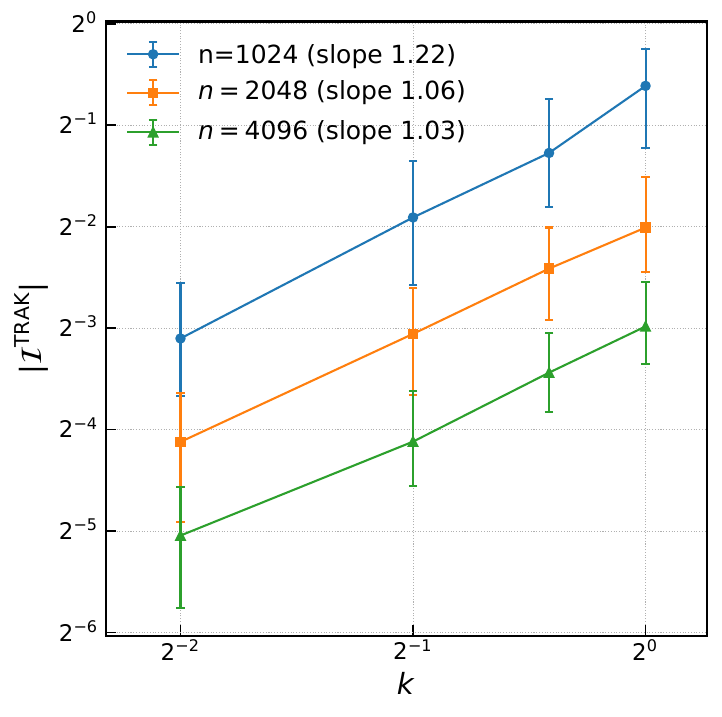}
    \includegraphics[width=0.32\linewidth]{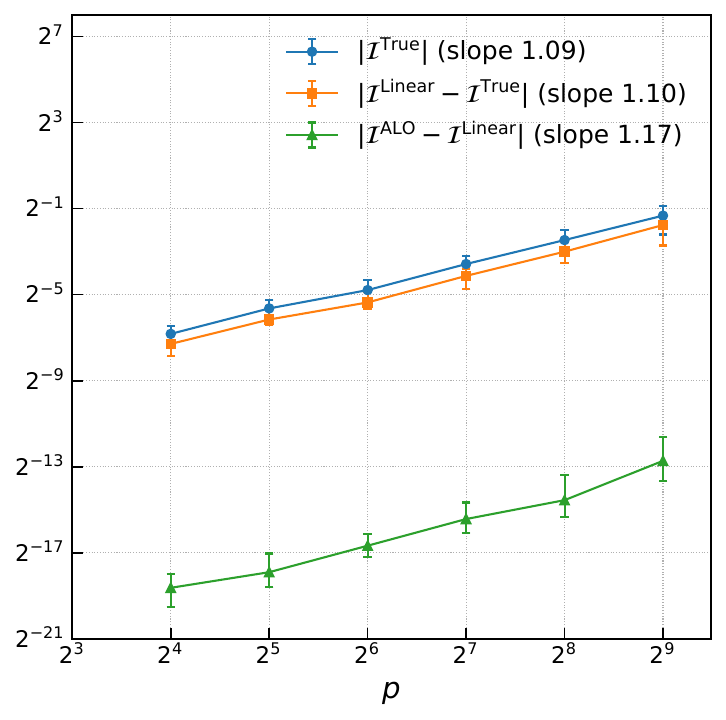}
    \includegraphics[width=0.32\linewidth]{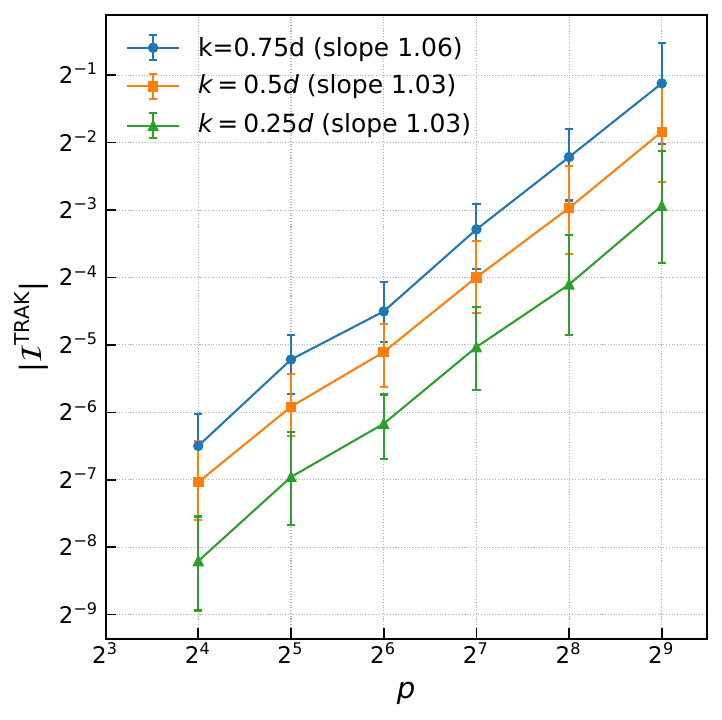}
    \includegraphics[width=0.32\linewidth]{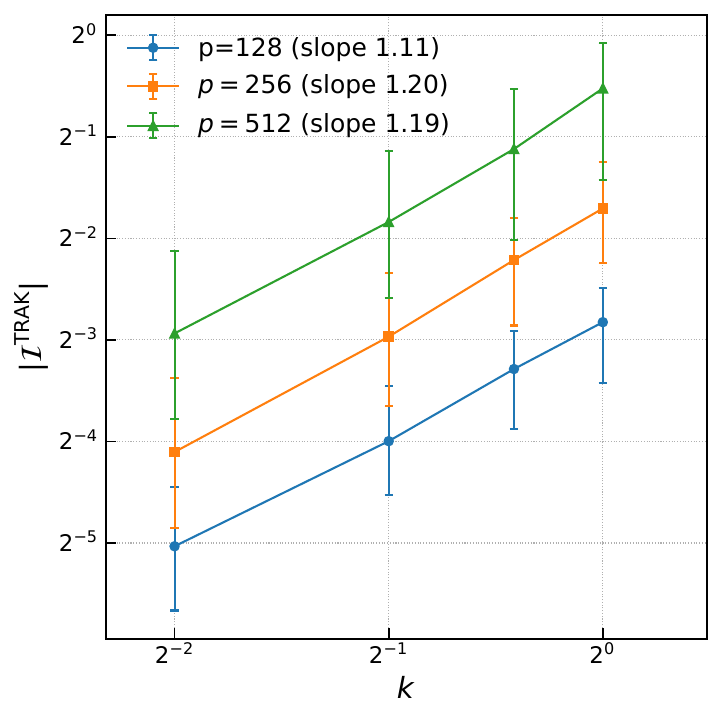}
    \includegraphics[width=0.32\linewidth]{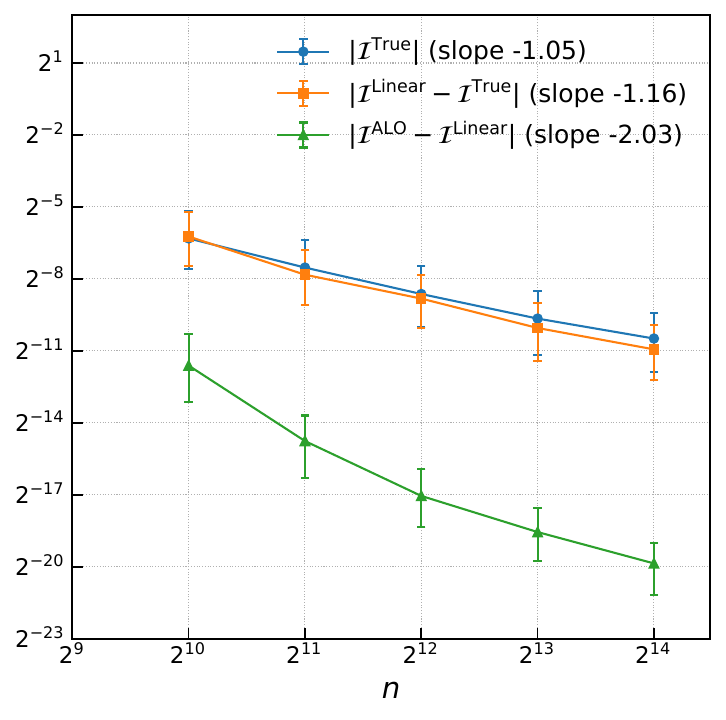}
    \includegraphics[width=0.32\linewidth]{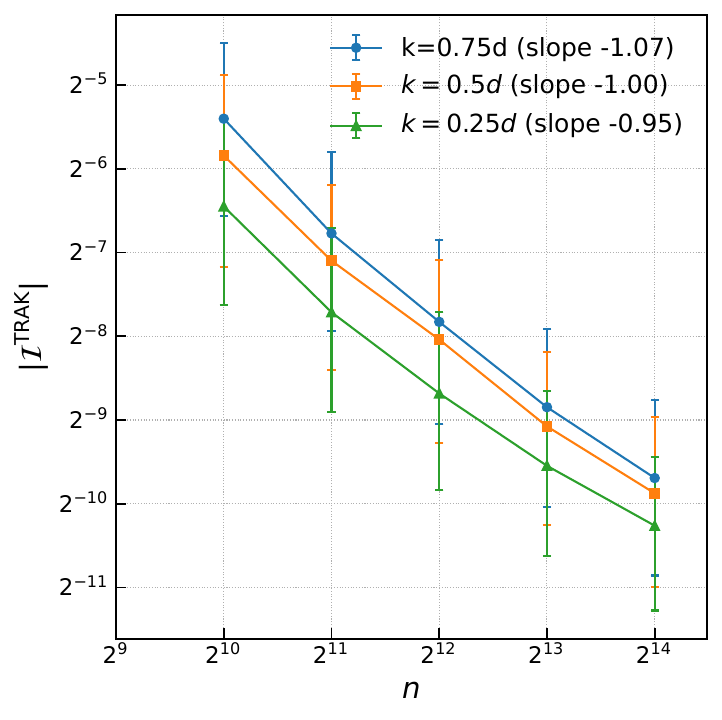}
    \includegraphics[width=0.32\linewidth]{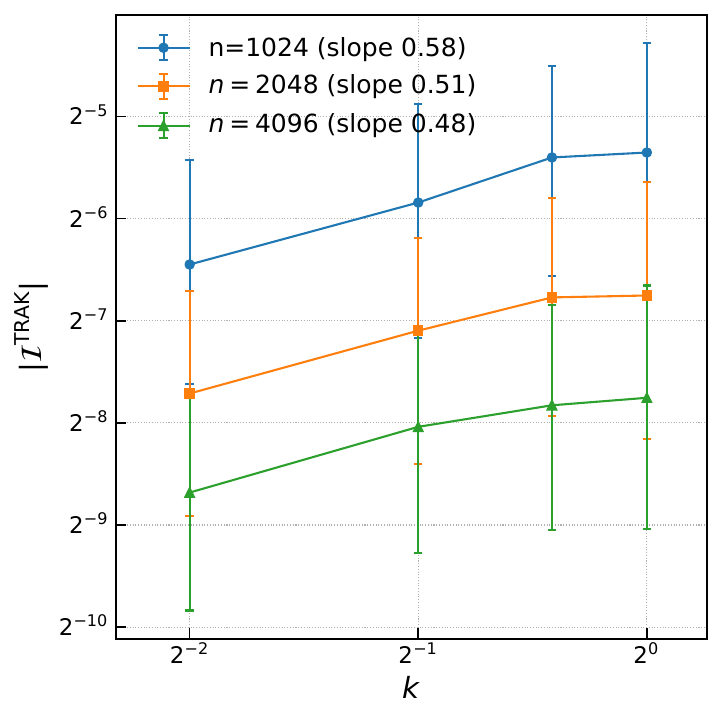}
    \includegraphics[width=0.32\linewidth]{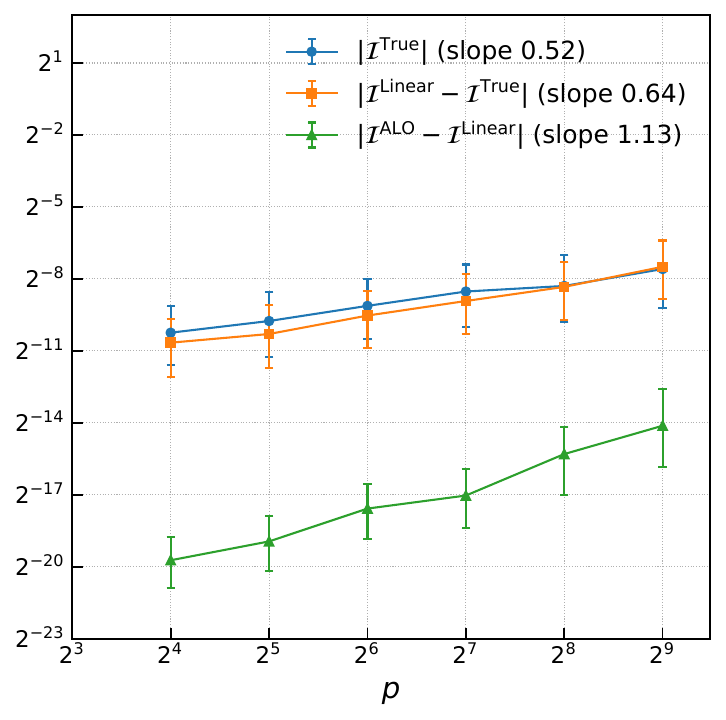}
    \includegraphics[width=0.32\linewidth]{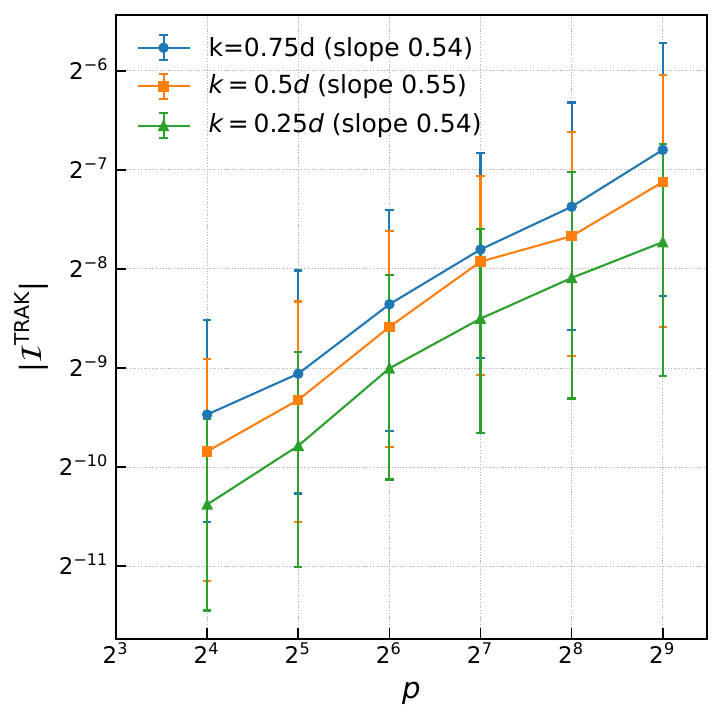}
    \includegraphics[width=0.32\linewidth]{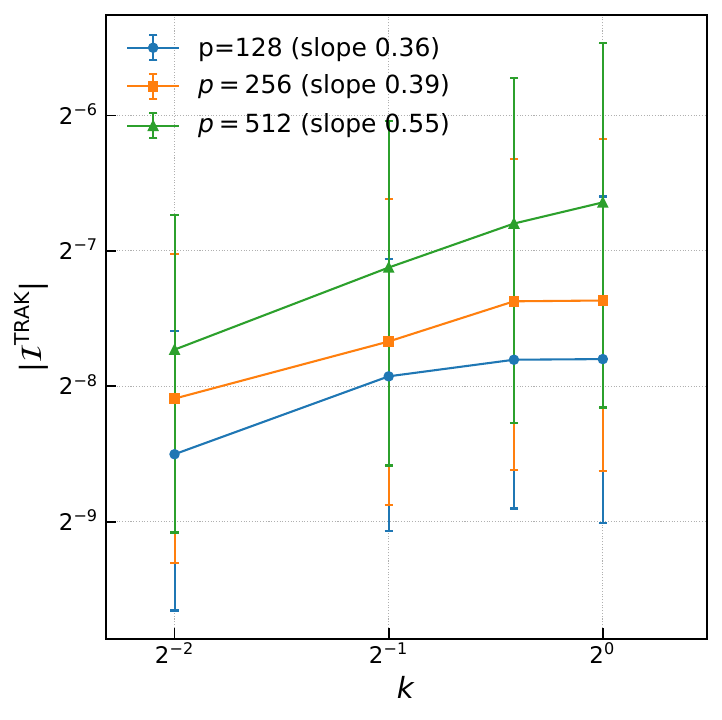}    
    \caption{
    $3$-class classification: magnitude-order checks.
    \textbf{First and second rows:} Correlated case, where $\bz_i=\bz_{\new}$.
    \textbf{Third and fourth rows:} Uncorrelated case, where $\bz_i$ and $\bz_{\new}$ are independent.
    \textbf{First and third rows:} $p$ is fixed at $100$ and $n$ varies.
    \textbf{Second and fourth rows:} $n$ is fixed at $5{,}000$ and $p$ varies.
    In each row, the first panel shows the magnitudes of $\mathcal{I}^{\rm True}$, the linearization error, and the ALO error; the second panel shows $\mathcal{I}^{\rm TRAK}(k)$ for different projection dimensions; and the third panel shows how $\mathcal{I}^{\rm TRAK}(k)$ varies with $k/d$.
    }
    \label{fig:magnitude}
\end{figure}

Figure~\ref{fig:magnitude} confirms the scaling behavior predicted by our theory. We first consider the magnitude of the true influence. In both the correlated and uncorrelated cases, $|\mathcal{I}^{\rm True}|$ decreases at rate $1/n$ as the sample size increases. The dependence on $p$, however, differs across the two regimes. In the correlated case, $|\mathcal{I}^{\rm True}(\bz_i,\bz_i)|$ grows approximately as $p/n$, consistent with Proposition~\ref{theo:I_true_corr}. In the uncorrelated case, $|\mathcal{I}^{\rm True}(\bz_i,\bz_{\new})|$ grows approximately as $\sqrt{p}/n$, consistent with Proposition~\ref{theo:I_true}. This difference is precisely the magnitude separation between dependent and independent data points predicted by our theory.

We next examine the linearization error. In the correlated case, the empirical order of
\[
|\mathcal{I}^{\rm Linear}(\bz_i,\bz_i)-\mathcal{I}^{\rm True}(\bz_i,\bz_i)|
\]
is approximately $p/n$, matching the order of the true influence itself and validating Theorem~\ref{theo:step1_corr}. In the uncorrelated case, the corresponding error scales approximately as $\sqrt{p}/n$, consistent with Theorem~\ref{theo:step1}. These results confirm that the linearization step can introduce a non-negligible absolute error, while its order still reflects the separation between the correlated and uncorrelated regimes.

The ALO step behaves differently. Across both regimes, the empirical magnitude of
\[
|\mathcal{I}^{\rm ALO}-\mathcal{I}^{\rm Linear}|
\]
is substantially smaller than the corresponding influence magnitude. In both the correlated and uncorrelated cases, the observed order is approximately $p/n^2$, which is smaller than our theoretical upper bounds: $p/n^{3/2}$ for the correlated case in Theorem~\ref{thm:g_base_ALO_ii}, and $\sqrt{p}/n^{3/2}$ for the uncorrelated case in Theorem~\ref{thm:g_base_ALO}. Although these bounds may not be tight, the conclusion is unchanged: the ALO approximation is benign in magnitude and does not destroy the separation between dependent and independent data points.

Finally, we examine the projection step. In the uncorrelated case, the magnitude of $\mathcal{I}^{\rm TRAK}(k)$ scales as $\sqrt{p}/n$ and grows approximately as $\sqrt{k/d}$ as the relative projection dimension increases. In contrast, in the correlated case, $\mathcal{I}^{\rm TRAK}(k)$ scales as $p/n$ and grows approximately linearly in $k/d$. These observations are consistent with Theorem~\ref{theo:magnitude_projection}. They also illustrate the main limitation of the projection step: when $k/d$ is too small, the projection can substantially reduce the separation between the correlated and uncorrelated influence magnitudes.
\subsection{Empirical validation on CIFAR-10}
\label{sec:empirical}

In this section, we conduct experiments on CIFAR-10 to assess whether our conclusions remain robust on real-world data that may not satisfy the data-generation assumptions used in our theory. The goal is not to maximize predictive accuracy, but to examine whether the qualitative behavior observed in the simulations persists in a more realistic multi-class classification problem.

\paragraph{Experimental setup.}
We use CIFAR-10, which contains $10$ classes and $50{,}000$ training images. Following the multi-class formulation in Appendix~\ref{sec:multi-class}, we train a multinomial logistic regression model using the standard $K-1$ parameterization. For computational feasibility, each image is first flattened after standard normalization, and we then apply a fixed Gaussian random projection to obtain a $p=200$-dimensional feature vector. The projected features are standardized using the empirical mean and standard deviation computed from the selected training set.

Since $K=10$ and $p=200$, the parameter dimension is
\[
d=p(K-1)=200\times 9=1800.
\]
The trained model achieves a training accuracy of $39.78\%$ and a test accuracy of $38.68\%$. This accuracy is lower than that of modern deep learning baselines, but it is typical for a simple linear classifier built on low-dimensional random-projection features.

We sample $100$ training points and $100$ test points, giving $10{,}000$ uncorrelated pairs. We also evaluate the corresponding correlated case, where $\bz_i=\bz_{\new}$. For each pair, we compute $\mathcal{I}^{\rm True}$, $\mathcal{I}^{\rm Linear}$, $\mathcal{I}^{\rm ALO}$, and $\mathcal{I}^{\rm TRAK}(k)$ for several projection dimensions.
\begin{figure}[H]
    \centering
    \includegraphics[width=0.19\linewidth]{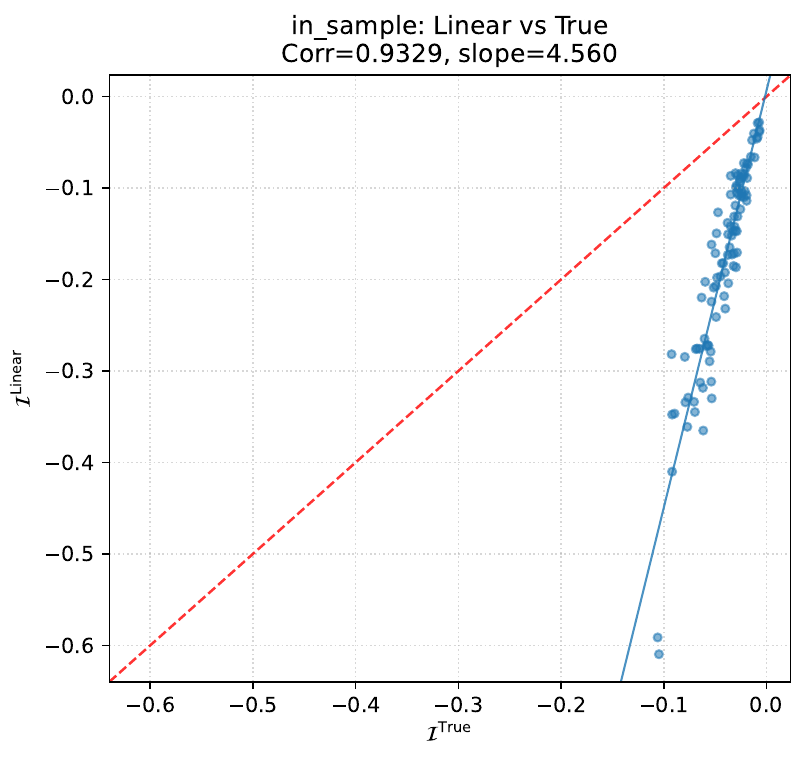}
    \includegraphics[width=0.19\linewidth]{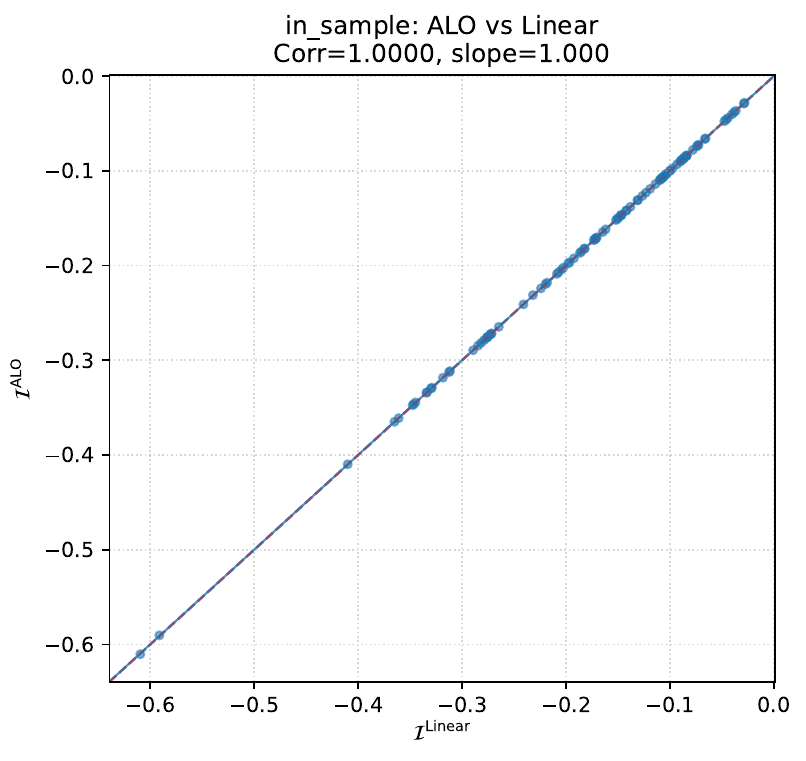}
    \includegraphics[width=0.19\linewidth]{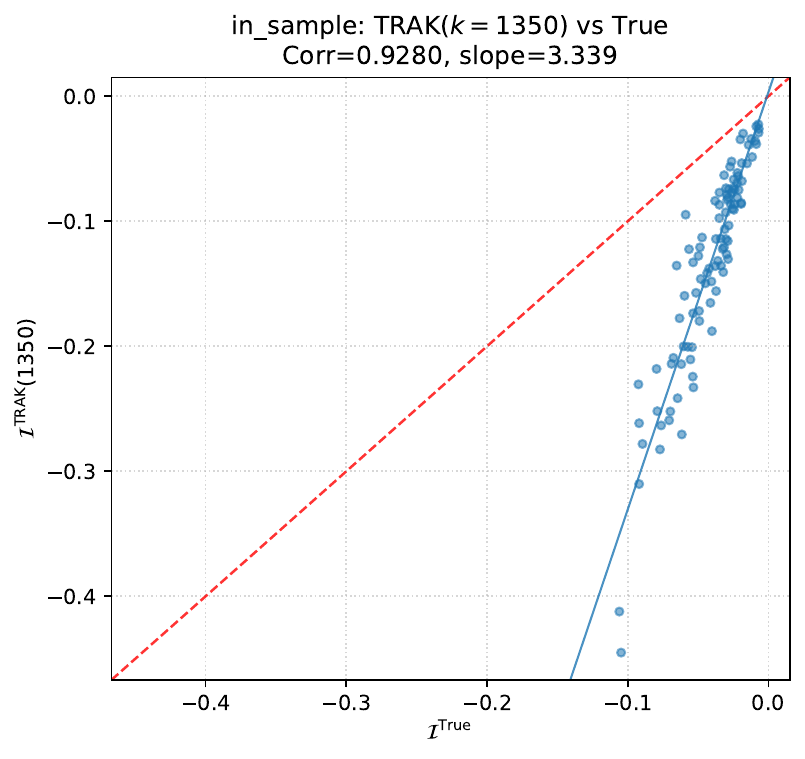}
    \includegraphics[width=0.19\linewidth]{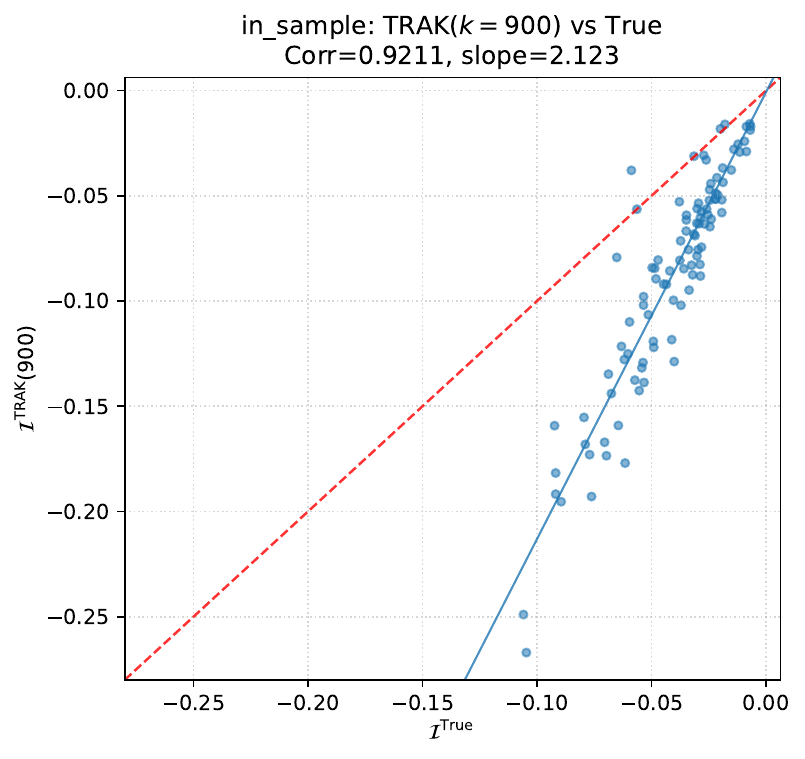}
    \includegraphics[width=0.19\linewidth]{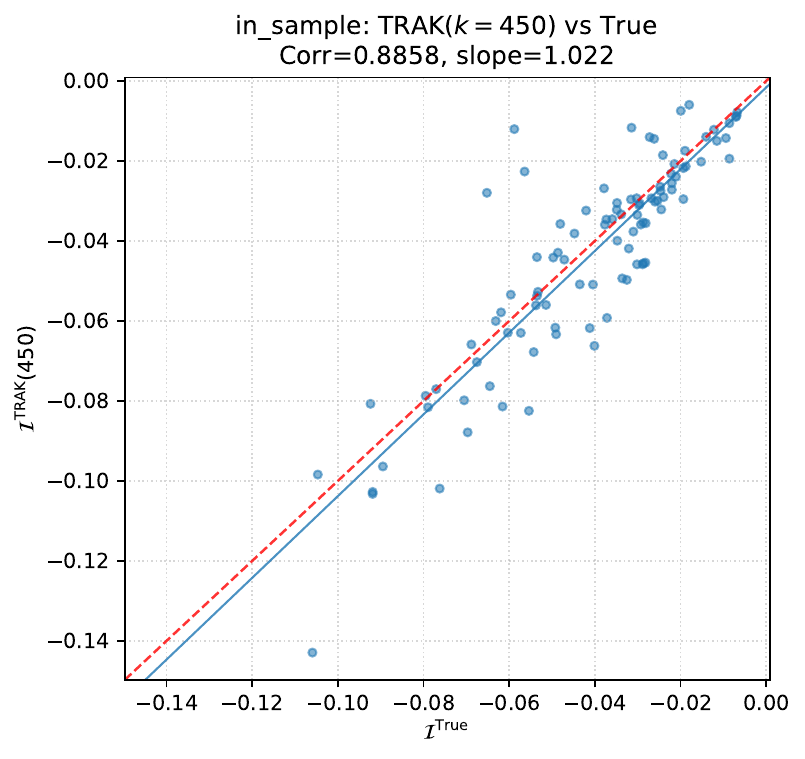}

    \includegraphics[width=0.19\linewidth]{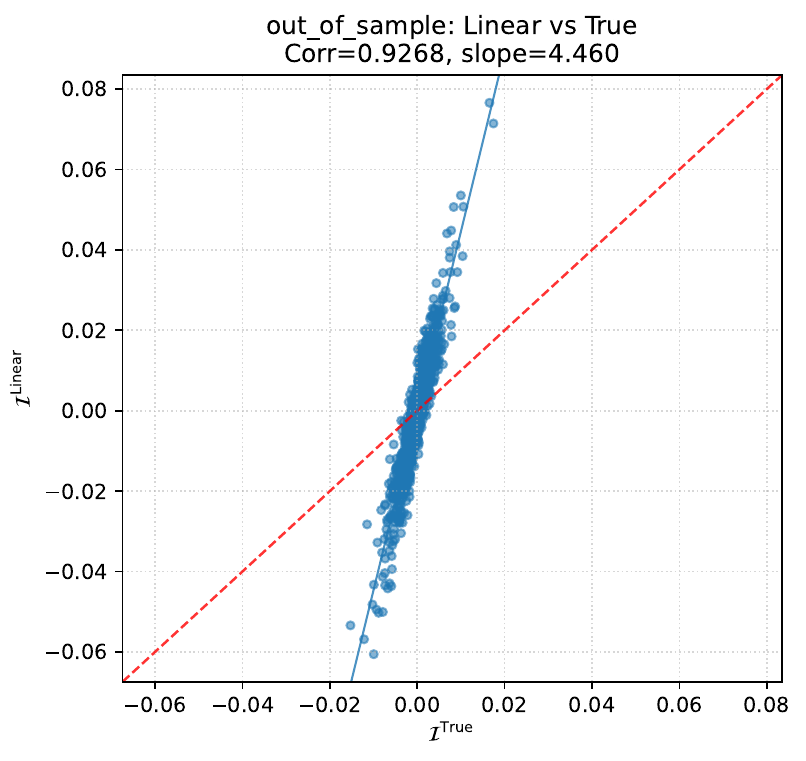}
    \includegraphics[width=0.19\linewidth]{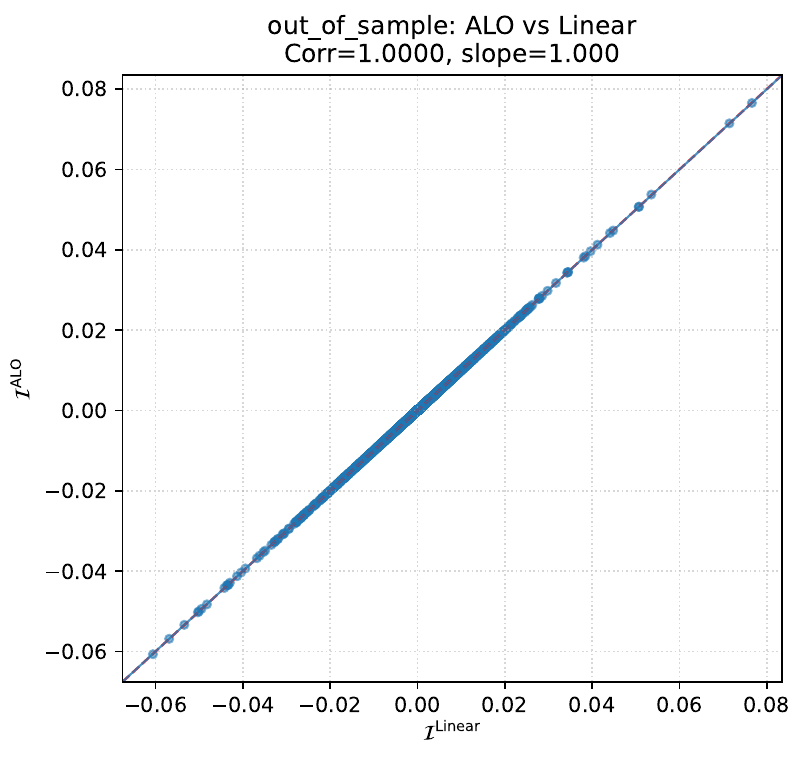}
    \includegraphics[width=0.19\linewidth]{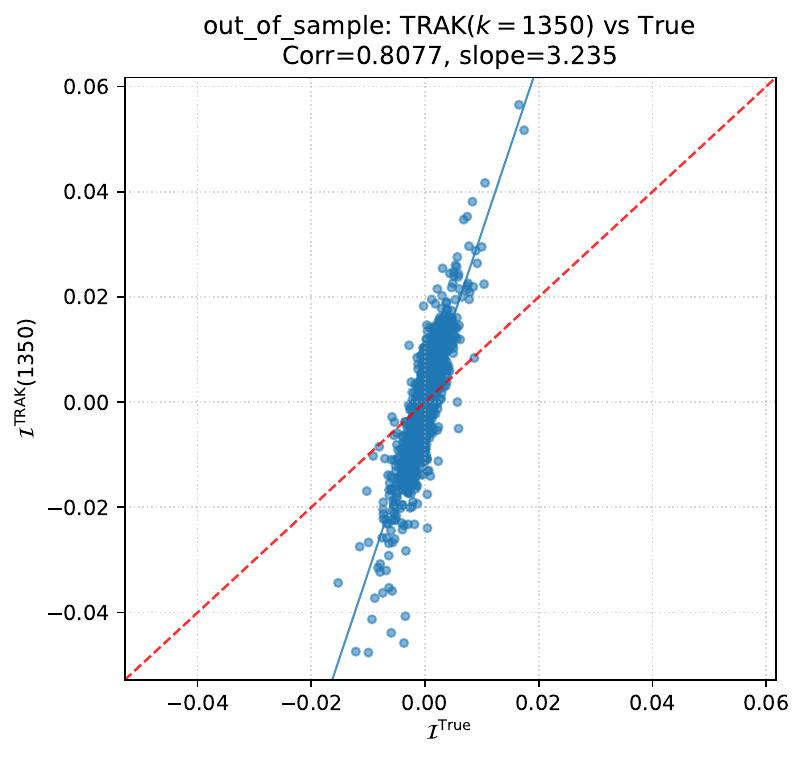}
    \includegraphics[width=0.19\linewidth]{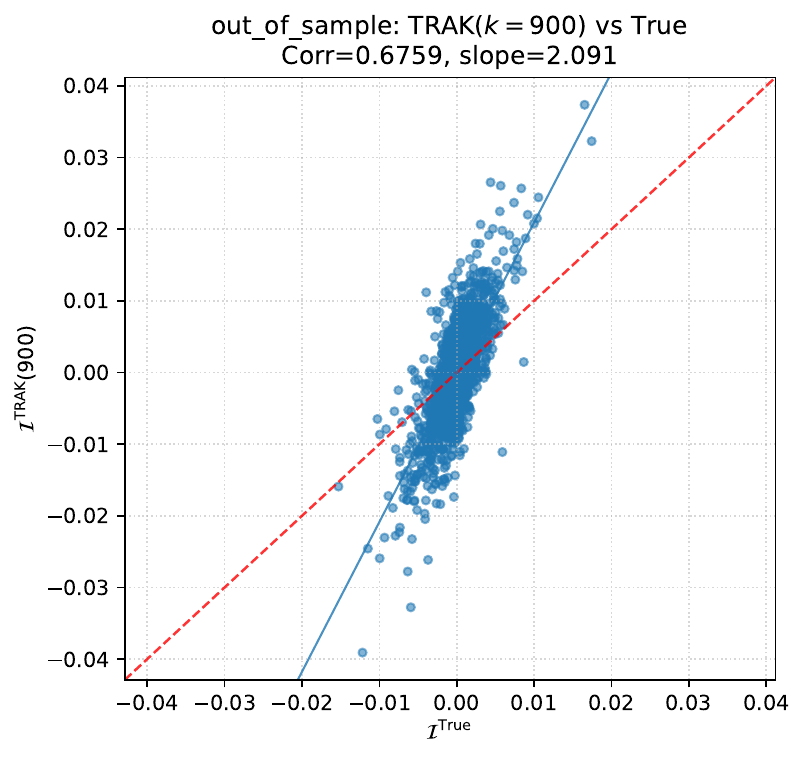}
    \includegraphics[width=0.19\linewidth]{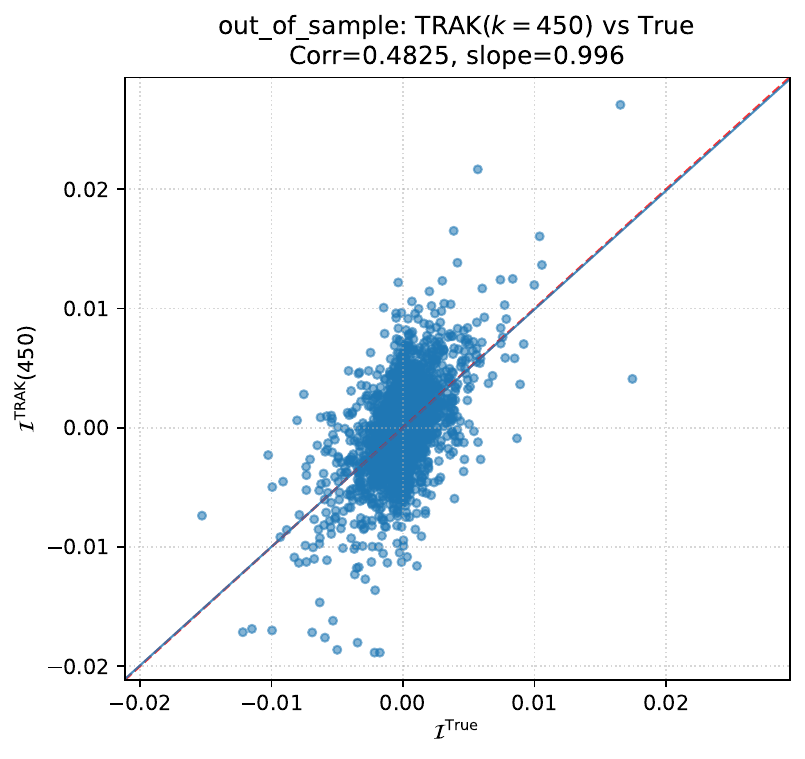}

    \caption{
    CIFAR-10 experimental results with $p=200$ and $d=1800$.
    \textbf{Top row:} Correlated case, where $\bz_i=\bz_{\new}$.
    \textbf{Bottom row:} Uncorrelated case, where $\bz_i$ and $\bz_{\new}$ are independent.
    \textbf{First column:} $\mathcal{I}^{\rm Linear}$ versus $\mathcal{I}^{\rm True}$.
    \textbf{Second column:} $\mathcal{I}^{\rm ALO}$ versus $\mathcal{I}^{\rm Linear}$.
    \textbf{Right three columns:} $\mathcal{I}^{\rm TRAK}(k)$ versus $\mathcal{I}^{\rm True}$ with $k=75\%d,50\%d,25\%d$, corresponding to $k=1350,900,450$.
    }
    \label{fig:cifar10_summary}
\end{figure}

\paragraph{Approximation behavior.}
Figure~\ref{fig:cifar10_summary} visualizes how the influence estimates behave after the three TRAK approximation steps in both the correlated and uncorrelated cases. The conclusions are qualitatively similar to those in Figure~\ref{fig:multi_alo}. The linearization step introduces visible magnitude distortion, but the approximate scores remain aligned with the true influence scores. The ALO step is nearly indistinguishable from the linearized influence, suggesting that ALO introduces little additional error. The projection step reduces the magnitude of the scores, and this reduction becomes more pronounced as the projection dimension decreases.

These observations are consistent with the main simulation results. In particular, the projection step is the main source of degradation when $k$ is small: smaller projection dimensions compress the influence scores more aggressively and can weaken the separation between the correlated and uncorrelated regimes.

\paragraph{Rank alignment.}
We next examine whether the approximate influence scores preserve the rankings induced by $\mathcal{I}^{\rm True}$. For each test point, we compare the Top-$m$ and Bottom-$m$ training examples selected by each approximation with those selected by $\mathcal{I}^{\rm True}$. We evaluate the Top-$m$ most positive influence scores, and the Bottom-$m$ most negative influence scores.

We report two ranking metrics. The first is the \textbf{Exact Match Count}, defined as the number of test points, out of $100$, for which the retrieved ranked list exactly matches the ranked list obtained from $\mathcal{I}^{\rm True}$, including the ordering. The second is the \textbf{Overlap Ratio}, defined as the fraction of overlap between the retrieved set and the exact set, averaged across test points.
\begin{table}[H]
    \centering
    \caption{
    Rank alignment between exact and approximate influence rankings on CIFAR-10 in the uncorrelated case. 
    Metrics are reported for both Top-$m$ positive influence scores, and Bottom-$m$ negative influence scores. 
    Exact Matches count the number of test points, out of $100$, for which the retrieved ranked list exactly matches the list obtained from $\mathcal{I}^{\rm True}$. 
    Overlap Ratio is the average set overlap with the corresponding exact Top-$m$ or Bottom-$m$ set.
    }
    \label{tab:cifar10-rank-out}
    \resizebox{\textwidth}{!}{%
    \begin{tabular}{llcccccc|cccccc}
    \toprule
    \textbf{Method} & \textbf{Metric} 
    & \multicolumn{6}{c|}{\textbf{Top-$m$ Positive}} 
    & \multicolumn{6}{c}{\textbf{Bottom-$m$ Negative}} \\
    & & \textbf{1} & \textbf{3} & \textbf{5} & \textbf{10} & \textbf{20} & \textbf{50} 
      & \textbf{1} & \textbf{3} & \textbf{5} & \textbf{10} & \textbf{20} & \textbf{50} \\
    \midrule
    $\mathcal{I}^{\rm Linear}$ 
    & Exact Matches 
    & 74 & 24 & 6 & 0 & 0 & 0 
    & 73 & 20 & 1 & 0 & 0 & 0 \\
    & Overlap Ratio 
    & 0.740 & 0.767 & 0.790 & 0.772 & 0.760 & 0.805 
    & 0.730 & 0.777 & 0.750 & 0.762 & 0.781 & 0.805 \\
    \midrule
    $\mathcal{I}^{\rm ALO}$ 
    & Exact Matches 
    & 75 & 24 & 6 & 0 & 0 & 0 
    & 73 & 20 & 1 & 0 & 0 & 0 \\
    & Overlap Ratio 
    & 0.750 & 0.767 & 0.788 & 0.772 & 0.759 & 0.805 
    & 0.730 & 0.777 & 0.750 & 0.763 & 0.781 & 0.805 \\
    \midrule
    $\mathcal{I}^{\rm TRAK}(k=1350)$ 
    & Exact Matches 
    & 68 & 13 & 0 & 0 & 0 & 0 
    & 66 & 15 & 1 & 0 & 0 & 0 \\
    & Overlap Ratio 
    & 0.680 & 0.627 & 0.648 & 0.576 & 0.575 & 0.674 
    & 0.660 & 0.670 & 0.620 & 0.607 & 0.581 & 0.674 \\
    \midrule
    $\mathcal{I}^{\rm TRAK}(k=900)$ 
    & Exact Matches 
    & 44 & 8 & 0 & 0 & 0 & 0 
    & 54 & 2 & 0 & 0 & 0 & 0 \\
    & Overlap Ratio 
    & 0.440 & 0.503 & 0.508 & 0.473 & 0.474 & 0.628 
    & 0.540 & 0.537 & 0.476 & 0.474 & 0.480 & 0.628 \\
    \midrule
    $\mathcal{I}^{\rm TRAK}(k=450)$ 
    & Exact Matches 
    & 26 & 2 & 0 & 0 & 0 & 0 
    & 30 & 0 & 0 & 0 & 0 & 0 \\
    & Overlap Ratio 
    & 0.260 & 0.323 & 0.352 & 0.343 & 0.390 & 0.586 
    & 0.300 & 0.383 & 0.354 & 0.376 & 0.409 & 0.586 \\
    \bottomrule
    \end{tabular}
    }%
\end{table}

Table~\ref{tab:cifar10-rank-out} shows that the linearization and ALO approximations preserve substantial ranking information. Although exact list matches become rare for larger values of $k$, the overlap ratios for $\mathcal{I}^{\rm Linear}$ and $\mathcal{I}^{\rm ALO}$ remain around $0.75$--$0.80$ for both Top-$m$ and Bottom-$m$. This indicates that these two approximations often identify similar influential training points, even if the exact order differs.

In contrast, the TRAK projection step leads to a clear degradation as the projection dimension decreases. For example, for Top-$1$ Positive, the exact match count decreases from $68\%$ at $k=1350$ to $44\%$ at $k=900$ and $26\%$ at $k=450$. The overlap ratios also decrease substantially across most Top-$m$ and Bottom-$m$ settings. This supports the same conclusion as the synthetic simulations: random projection can preserve useful ranking information when $k$ is sufficiently large, but aggressive projection can weaken the separation structure and harm rank agreement.

\subsection{Results of more general correlated cases}
\label{sec:result_general_corr}

In this section, we show some simulations as defined in Section~\ref{ssec:general_corr}. We still consider $3$-class classification with $p=100$ and $d=200$, using the same data-generating mechanism as in Figure~\ref{fig:multi_alo} to generate our training data. But for the new test data, we generate them using \eqref{eq:gen_corr_data}. We randomly select $100$ data points to remove and generate $10$ new data points for each removed data point, using $\alpha_n$ uniformly distributed from $-1$ to $1$. We then compare the influence between $\bz_\new$ and $\bz_i$ with the influence of $\bz_i$ on $\bz_i$. 
\begin{figure}[ht]
    \centering
    \includegraphics[width=1\linewidth]{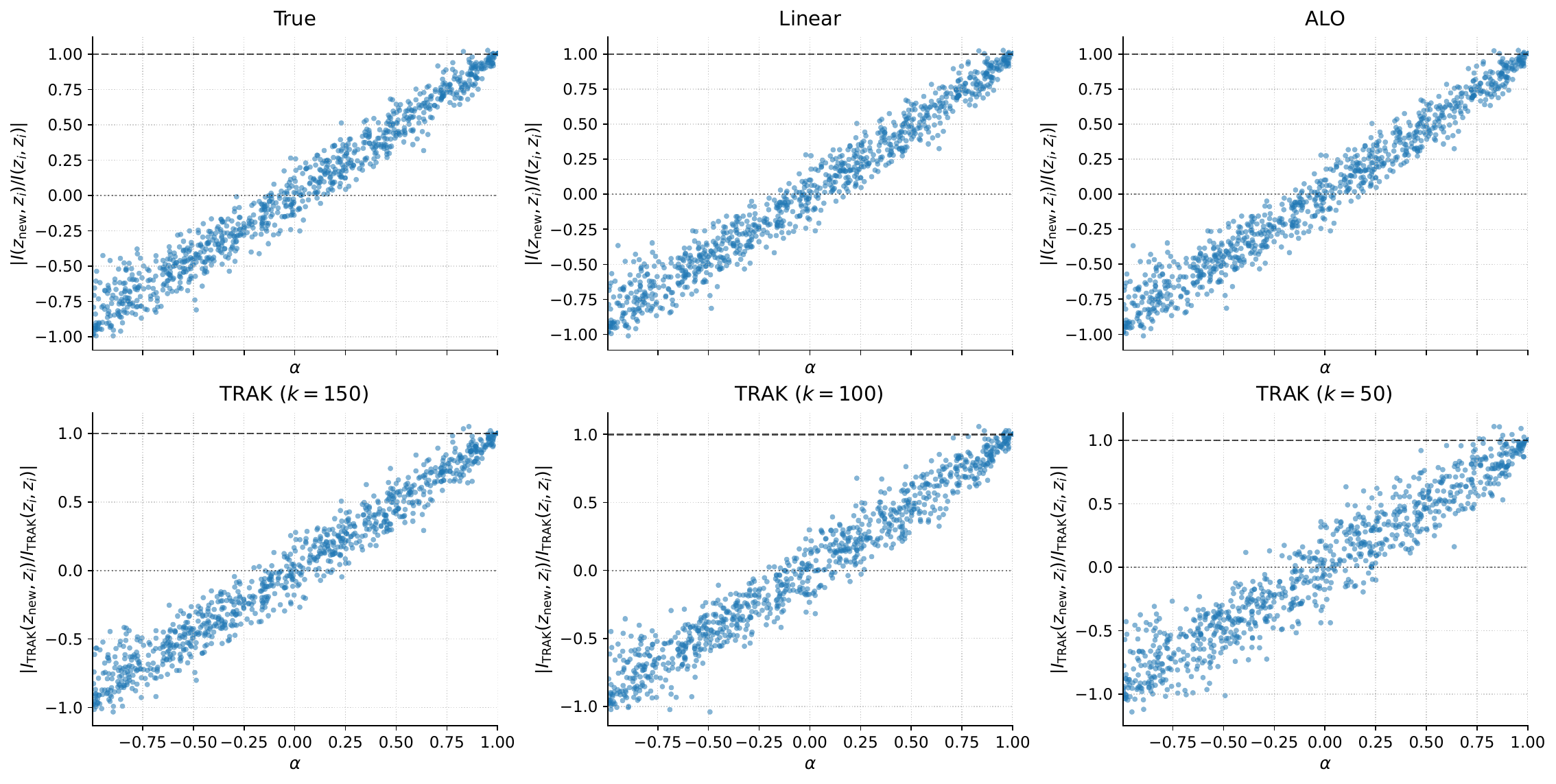}
    \caption{The setting is the same as in Figure~\ref{fig:multi_alo}. We show the ratio of the influence between $\bz_\new$ and $\bz_i$ to the influence of $\bz_i$ on $\bz_i$. The first row shows the ratios for $\mathcal{I}^{\rm True}$, $\mathcal{I}^{\rm Linear}$, and $\mathcal{I}^{\rm ALO}$; the second row shows the ratios for $\mathcal{I}^{\rm TRAK}(k)$, with $k=75\%, 50\%, 25\%$ of $d$, corresponding to $k=150,100,50$.}
    \label{fig:ratio}
\end{figure}

We plot their ratio using different influence definitions. As we can see from Figure~\ref{fig:ratio}, the x-axis shows the different $\alpha$ values of each new data point, and the y-axis shows the influence ratio. The result is quite intuitive and is indicated in Section~\ref{ssec:general_corr}: although the exact relationship cannot be preserved due to the randomness of the independent part of $\bx_i$, the ratio is still of order $\Theta(1)$, meaning that the prediction in \eqref{eq:alpha_rela} holds.

Also, some interesting and consistent results can be seen from Figure~\ref{fig:ratio}. As we can see, the error introduced in the linearization and ALO steps is mild. However, as the projection dimension $k$ becomes smaller, for example from $150$ to $75$, the width of the band becomes obviously wider, which again shows that too aggressive projection may deteriorate the performance of TRAK, consistent with our previous message.
\clearpage
\section*{NeurIPS Paper Checklist}

% IMPORTANT, please:
% \begin{itemize}
%     \item {\bf Delete this instruction block, but keep the section heading ``NeurIPS Paper Checklist"},
%     \item  {\bf Keep the checklist subsection headings, questions/answers and guidelines below.}
%     \item {\bf Do not modify the questions and only use the provided macros for your answers}.
% \end{itemize} 

\begin{enumerate}

\item {\bf Claims}
    \item[] Question: Do the main claims made in the abstract and introduction accurately reflect the paper's contributions and scope?
    \item[] Answer: \answerYes{} % Replace by \answerYes{}, \answerNo{}, or \answerNA{}.
    \item[] Justification: The abstract and introduction state the paper's main theoretical contribution: characterizing the approximation errors in TRAK and showing that, under the stated assumptions, the approximations can preserve meaningful separation between influential and weakly influential data points. The empirical claims are limited to simulations and CIFAR-10 validation reported in the main text and appendix.


\item {\bf Limitations}
    \item[] Question: Does the paper discuss the limitations of the work performed by the authors?
    \item[] Answer: \answerYes{} % Replace by \answerYes{}, \answerNo{}, or \answerNA{}.
    \item[] Justification: We add a paragraph in conclusion explaining the limitations of our work.

\item {\bf Theory assumptions and proofs}
    \item[] Question: For each theoretical result, does the paper provide the full set of assumptions and a complete (and correct) proof?
    \item[] Answer: \answerYes{} % Replace by \answerYes{}, \answerNo{}, or \answerNA{}.
    \item[] Justification: We state the full set of assumptions in the appendix and verify that these assumptions are satisfied by multiple model classes, including GLMs and MLPs.

    \item {\bf Experimental result reproducibility}
    \item[] Question: Does the paper fully disclose all the information needed to reproduce the main experimental results of the paper to the extent that it affects the main claims and/or conclusions of the paper (regardless of whether the code and data are provided or not)?
    \item[] Answer: \answerYes{} % Replace by \answerYes{}, \answerNo{}, or \answerNA{}.
    \item[] Justification: We clearly state the experimental settings so that the experiments can be reproduced.

\item {\bf Open access to data and code}
    \item[] Question: Does the paper provide open access to the data and code, with sufficient instructions to faithfully reproduce the main experimental results, as described in supplemental material?
    \item[] Answer: \answerNo{}.
    \item[] Justification: We do not provide code as part of the anonymous submission. The main contribution of the paper is theoretical, and the simulation and empirical studies are described in detail in the main text and appendix, including the data-generating mechanisms, model classes, sample sizes, projection dimensions, and evaluation metrics. The empirical study uses CIFAR-10, which is publicly available. We plan to release code after publication.

\item {\bf Experimental setting/details}
    \item[] Question: Does the paper specify all the training and test details (e.g., data splits, hyperparameters, how they were chosen, type of optimizer) necessary to understand the results?
    \item[] Answer: \answerYes{} % Replace by \answerYes{}, \answerNo{}, or \answerNA{}.
    \item[] Justification: We specify the experimental settings in the main text and appendix, including the data-generating mechanisms, model classes, sample sizes, feature dimensions, projection dimensions, and the in-sample and out-of-sample evaluation regimes. For the CIFAR-10 experiments, we also describe the preprocessing steps, random-projection feature construction, model parameterization, and evaluation metrics.

\item {\bf Experiment statistical significance}
    \item[] Question: Does the paper report error bars suitably and correctly defined or other appropriate information about the statistical significance of the experiments?
\item[] Answer: \answerYes{}.
\item[] Justification: For the magnitude-based results, such as Appendix Figure~6, we report error bars and define how they are computed. For other results, such as the ranking table and scatter plots, we do not report error bars or statistical significance tests. The main claims of the paper are theoretical, and the experiments are designed to illustrate and verify the predicted approximation behavior, scaling patterns, and separation effects rather than to compare competing methods through repeated randomized trials.

\item {\bf Experiments compute resources}
    \item[] Question: For each experiment, does the paper provide sufficient information on the computer resources (type of compute workers, memory, time of execution) needed to reproduce the experiments?
    \item[] Answer: \answerNo{}.
    \item[] Justification: We do not provide detailed compute-resource information. The experiments are moderate-scale simulations and multinomial logistic-regression experiments, and each reported experiment can be run within about one hour on standard computing resources.
    
\item {\bf Code of ethics}
    \item[] Question: Does the research conducted in the paper conform, in every respect, with the NeurIPS Code of Ethics \url{https://neurips.cc/public/EthicsGuidelines}?
    \item[] Answer: \answerYes{} % Replace by \answerYes{}, \answerNo{}, or \answerNA{}.   
    \item[] Justification: The research is theoretical and empirical work on data attribution and influence approximation. It uses synthetic data and the public CIFAR-10 dataset only.

\item {\bf Broader impacts}
    \item[] Question: Does the paper discuss both potential positive societal impacts and negative societal impacts of the work performed?
    \item[] Answer: \answerYes{}.
\item[] Justification: The paper studies theoretical properties of data attribution methods. Potential positive impacts include improving transparency, debugging, and data valuation. Potential negative impacts include possible misuse of attribution scores for overconfident data ownership or compensation decisions when the assumptions of the theory are violated. We therefore emphasize that our results concern separation under stated assumptions, not exact ranking or exact monetary valuation.
    
\item {\bf Safeguards}
    \item[] Question: Does the paper describe safeguards that have been put in place for responsible release of data or models that have a high risk for misuse (e.g., pre-trained language models, image generators, or scraped datasets)?
\item[] Answer: \answerNA{}.
\item[] Justification: The paper does not release high-risk models, scraped datasets, generative models, or datasets with sensitive content. The empirical study uses CIFAR-10, a standard public benchmark dataset, and the work poses no such misuse risks requiring safeguards.

\item {\bf Licenses for existing assets}
    \item[] Question: Are the creators or original owners of assets (e.g., code, data, models), used in the paper, properly credited and are the license and terms of use explicitly mentioned and properly respected?
    \item[] Answer: \answerYes{}.
    \item[] Justification: The paper uses CIFAR-10 as an existing public benchmark dataset and properly cites the original dataset source. We do not redistribute CIFAR-10 or introduce a derived dataset as a new asset.

\item {\bf New assets}
    \item[] Question: Are new assets introduced in the paper well documented and is the documentation provided alongside the assets?
    \item[] Answer: \answerNA{}.
\item[] Justification: The paper does not release any new datasets, model checkpoints, or software packages as new assets.

\item {\bf Crowdsourcing and research with human subjects}
    \item[] Question: For crowdsourcing experiments and research with human subjects, does the paper include the full text of instructions given to participants and screenshots, if applicable, as well as details about compensation (if any)? 
   \item[] Answer: \answerNA{}.
\item[] Justification: The paper does not involve crowdsourcing, user studies, or research with human subjects.

\item {\bf Institutional review board (IRB) approvals or equivalent for research with human subjects}
    \item[] Question: Does the paper describe potential risks incurred by study participants, whether such risks were disclosed to the subjects, and whether Institutional Review Board (IRB) approvals (or an equivalent approval/review based on the requirements of your country or institution) were obtained?
    \item[] Answer: \answerNA{}.
\item[] Justification: The paper does not involve crowdsourcing, user studies, or research with human subjects, so IRB approval or equivalent review is not applicable.

\item {\bf Declaration of LLM usage}
    \item[] Question: Does the paper describe the usage of LLMs if it is an important, original, or non-standard component of the core methods in this research? Note that if the LLM is used only for writing, editing, or formatting purposes and does \emph{not} impact the core methodology, scientific rigor, or originality of the research, declaration is not required.
    %this research? 
   \item[] Answer: \answerNA{}.
\item[] Justification: The core method development, theoretical analysis, simulations, and empirical evaluation do not involve LLMs as an important, original, or non-standard component. Any use of LLMs for writing, editing, or formatting does not affect the methodology, scientific rigor, or originality of the research.

\end{enumerate}
\end{document}